\documentclass[twoside,11pt]{article}

\usepackage{graphicx} % more modern the issue of spectral 
\graphicspath{{figures/}}
\usepackage{subfigure} 
\usepackage{dsfont}
\usepackage{url}
\usepackage{dcolumn}
\usepackage{enumitem}
\usepackage{theorem}
\usepackage{array}
\usepackage{jmlr2e}
\usepackage{tabmac}
\usepackage{alex}
\usepackage[]{natbib}
\usepackage{algorithm}
\usepackage{algorithmic}
\usepackage{amssymb}
\usepackage{amsmath}
\usepackage{mathtools}

\newcommand{\symm}{\mathfrak{S}} 
\newcommand{\hatm}{\hat{M}}

\newcommand{\ib}{{\bf i}}
\newcommand{\jb}{{\bf j}}
\usepackage{tikz}
\usepackage{pgfplots}

\usetikzlibrary{calc,trees,positioning,arrows}
\usetikzlibrary{automata,positioning}
\usetikzlibrary{trees}
\usetikzlibrary{decorations.pathmorphing} % noisy shapes
\usetikzlibrary{fit}         % fitting shapes to coordinates
\usetikzlibrary{backgrounds} % drawing the background after the
\usetikzlibrary{positioning}
\usetikzlibrary{shadows}
\usepgflibrary{shapes}

\tikzstyle{observed}=[circle, thick, minimum size=0.9cm, draw=black!100, fill=black!20]
\tikzstyle{latent}=[circle, thick, minimum size=0.9cm, draw=black!80]
\tikzstyle{plate}=[rectangle, thick, inner sep=0.3cm, draw=black!100]
\tikzstyle{shadeplate}=[rectangle, thick, inner sep=0.4cm, draw=black!100]
\tikzstyle{table}=[circle,fill=blue!20,draw=black!100,inner sep=1pt, minimum size=30pt]
\tikzstyle{client}=[rectangle,fill=blue!20,draw=black!100,inner sep=1pt, minimum size=12pt]

\ShortHeadings{Spectral Methods for Nonparametric Models}{Spectral Methods for Nonparametric Models}
\begin{document}

\title{Spectral Methods for Nonparametric Models}

\author{\name Hsiao-Yu Fish Tung \email htung@cs.cmu.edu \\
       \addr Machine Learning Department\\
       Carnegie Mellon University\\
       5000 Forbes Ave, Pittsburgh, PA 15213   
       \AND
     \name Chao-Yuan Wu \email cywu@cs.utexas.edu  \\
       \addr Department of Computer Science\\
       University of Texas at Austin\\
       Austin, TX 78712, USA
   \AND
   \name Manzil Zaheer \email manzilz@cs.cmu.edu \\
       \addr Machine Learning Department\\
       Carnegie Mellon University\\
       5000 Forbes Ave, Pittsburgh, PA 15213   
       \AND
       \name Alexander J. Smola \email alex@smola.org \\
       \addr Amazon Web Services\\
       \addr Machine Learning Department\\
       Carnegie Mellon University\\
       5000 Forbes Ave, Pittsburgh, PA 15213          
       }

\editor{}

\maketitle

\begin{abstract}%   <- trailing '%' for backward compatibility of .sty file
Nonparametric models are versatile, albeit computationally expensive, tool 
for modeling mixture models. In this paper, we introduce spectral methods 
for the two most popular nonparametric models: the Indian Buffet Process (IBP) 
and the Hierarchical Dirichlet Process (HDP). We show that using spectral 
methods for the inference of nonparametric models are computationally and 
statistically efficient. In particular, we derive the lower-order moments of 
the IBP and the HDP, propose spectral algorithms for both models, and 
provide reconstruction guarantees for the algorithms. For the HDP, we further 
show that applying hierarchical models on dataset with hierarchical structure, 
which can be solved with the generalized spectral HDP, produces better 
solutions to that of flat models regarding likelihood performance.
\end{abstract}


\section{Introduction}

Latent variable models have become ubiquitous in statistical data
analysis, spanning over a diverse set of applications ranging from text
\citep{BleNgJor02}, images \citep{QuaColDar04} to user behavior
\citep{AlyHatJosNar12}. In these works, latent variables are introduced to
represent unobserved properties or hidden causes of the observed
data. In particular, Bayesian Nonparametrics such as the Dirichlet mixture 
models \citep{Neal98b}, the Indian Buffet Process (IBP) \citep{GriGha11} 
and the Hierarchical Dirichlet Process (HDP) \citep{TehJorBeaBle06} allow
for flexible representation and adaptation in terms model complexity.

In recent years spectral methods have become a credible alternative to
sampling \citep{GriSte04} and variational methods
\citep{BleJor05,DemLaiRub77} for the inference of such structures.
In particular, the work of
\cite{AnaGeHsuKakTel12,AnaChaHsuKakSonZha2011,BooGreGeo13,HsuKakZha09,SonBooSidGorSmo10}
demonstrates that it is possible to infer latent variable structure
accurately, despite the problem being nonconvex, thus exhibiting many
local minima. A particularly attractive aspect of spectral methods is
that they allow for efficient means of inferring the model complexity
in the same way as the remaining parameters, simply by thresholding
eigenvalue decomposition appropriately. This makes them suitable for
nonparametric Bayesian approaches. 

While the issue of spectral inference with the Dirichlet Distribution is
largely settled \citep{AnaGeHsuKakTel12, AnaGeHsuKak13}, the domain of
nonparametric tools is much richer and it is therefore desirable to
see whether the methods can be extended to popular nonparametric models 
such as the IBP. As sampling-based methods are computationally 
expensive for models with complicated hierarchical structure, another attractive 
direction is to apply spectral method to nonparametric hierarchical model such 
as the HDP. 
By using countsketch FFT technique for fast tensor decomposition
\citep{WanTunSmoAna15}, spectral method for the Latent Dirichlet Allocation (LDA), which can be viewed as the simplest case in thespectral algorithm for the HDP, already outperform sampling-based algorithms significantly both in terms of perplexity and speed.
%Spectral Latent Dirichlet Allocation (LDA), which can be viewed as 
%the simplest case in the HDP,
%already outperform sampling-based algorithms significantly both in
%terms of perplexity and speed by using count sketch FFT technique for fast tensor decomposition \citep{WanTunSmoAna15}. 
Since the time complexity of the proposed spectral method for the HDP
does not scale with the number of layers, the algorithm enjoys
significant improvement in time over HDP samplers. In a nutshell, this work 
contributes to completing the tool set of spectral methods.
This is an important goal to ensure that entire models can be
translated wholly into spectral algorithms, rather than just parts.

 We provide a full analysis of the tensors arising from the IBP and the HDP. 
 For the IBP, we show how spectral algorithms need to be modified, since a 
 degeneracy in the third order tensor requires fourth order terms, to
 successfully infer all the hidden variables.
 For the HDP, we derive the generalized form in obtaining tensors for any 
 arbitrary hierarchical structure. To recover the
parameters and latent factors, we use Excess Correlation Analysis
(ECA) \citep{AnaFosHsuKakLiu12} to whiten the higher order tensors and
to reduce their dimensionality. Subsequently we employ the power
method to obtain symmetric factorization of the higher-order
terms. The methods provided in this work are simple to implement and have
high efficiency in recovering the latent factors and related
parameters. We demonstrate how this approach can be used in inferring
an IBP structure in the models discussed in \cite{GriGha11} and
\cite{KnoGha07} and the generalized spectral method for the HDP, which can 
be used in modeling problems involving grouped data such that mixture 
components are shared across all the groups.
Moreover, we show that empirically the spectral algorithms
outperform sampling-based algorithms and variational approaches
both in terms of perplexity and speed. Statistical guarantees for recovery and
stability of the estimates conclude the paper. 

{\bf Outline:} 
The key idea of spectral methods is to use the method of moments to solve the 
underlying parameters, which includes the following steps: 
\begin{itemize}
\item Construct equations for obtaining diagonalized tensors using moments of 
the latent variables defined in the probabilistic graphical model.
\item Replace the theoretical moments with the empirical moments and obtain 
an empirical version of the diagonalized tensor.
\item Use tensor decomposition solvers to decompose the empirical diagonalized 
tensor and obtain its eigenvalues/eigenvectors, which corresponds to the desired 
hidden vectors/topics. 
\end{itemize}
In order to use tensor decomposition solver, a decomposable symmetric tensor
must be constructed. 
A tensor is decomposable and symmetric if it can be written as a summation of 
the outer products of its eigenvectors weighted by their correspnding eigenvalues. 
In the two dimensional case (i.e, as a matrix), a rank-k symmetric tensor is 
decomposable and symmetric since it can be decomposed as $M = \sum_{i=1}^k \lambda_i v_i v_i^T,$ where $\lambda_i/v_i$ are the eigenvalue/eigenvector pairs. 
In the first step, we construct a tensor that has such properties using theoretical
 moments so that the tensor can be further estimated using empirical moments 
 and decomposed by tensor decomposition tools.

The paper is structured as follows: In Section \ref{sec:model} we introduce the IBP and the HDP models. 
In Section \ref{sec:moments} we construct equations for obtaining the diagonalized tensors using moments of the IBP and apply them on two applications, the linear Gaussian latent factor model and the infinte sparse factor analysis. We also derive the generalized tensors for the HDP that are applicable on any arbitrary hierarchical structure. 
In Section \ref{sec:alg} the spectral algorithms for thel IBP and the HDP is proposed. 
We also list out several tensor decomposition tools that can be used to solve our problem. 
In Section \ref{sec:concentration} we show the concentration measure of moments 
and tensors for these two models and provide overall guarantees on $L_2$ distance
 between the recovered latent vectors and the ground truth. 
In Section \ref{sec:experiment} we demonstrate the power of the spectral IBP by 
showing that the method is able to produce comparable results to that of 
variational approaches with much lesser time. 
We also applied it on image data and gene expression data to show that the algorithm is able to infer mearningful patterns in real data. 
For the spectral method for the HDP, we show that (1) computational does not increase with number of layer using our method, while obviously the factor will significantly affect Gibbs sampling
and (2) when the number of samples underneath each nodes in a hierarchical structure 
is highly unbalanced, the spectral for the HDP is able to obtain solutions better than that of spectral LDA in terms of perplexity.

\section{Model Settings}
\label{sec:model}
We begin with defining the models of the IBP and the HDP.
\subsection{The Indian Buffet Process}
\label{sec:ibp}

The Indian Buffet Process defines a distribution over equivalence
classes of binary matrices $Z$ with a finite number of rows and a
(potentially) infinite number of columns \citep{GriGha06,GriGha11}. The
idea is that this allows for automatic adjustment of the number of
binary entries, corresponding to the number of independent sources,
underlying causes, etc.\ This is a very useful strategy and it has led
to many applications including structuring Markov transition
matrices \citep{FoxSudJorWil10}, learning hidden causes with a bipartite
graph \citep{WooGriGha06} and finding latent features in link prediction \citep{MilGriJor09}.
Denote by $n \in \mathbb{N}$ the number of rows of $Z$, i.e.\ the number of
customers sampling dishes from the `` Indian Buffet'', let $m_k$ be the
number of customers who have sampled dish $k$, let $K_+$ be the total number
of dishes sampled, and denote by $K_h$ the number of dishes with a
particular selection history $h \in \cbr{0;1}^n$. That is, $K_h > 1$
only if there are two or more dishes that have been selected by
exactly the same set of customers. Then the probability of
generating a particular matrix $Z$ is given by \cite{GriGha11}
\begin{align}
  \label{eq:pz}
  p(Z) = \frac{\alpha^{K_+}}{\prod_h K_h!} \exp\Biggl[{-\alpha
    \sum_{j=1}^n \textstyle \frac{1}{j}}\Biggr] \prod_{k=1}^{K_+} \frac{(n-m_k)! (m_k-1)!}{n!}
\end{align}
Here $\alpha$ is a parameter determining the expected number of
nonzero columns in $Z$.  Due to the conjugacy of the prior an
alternative way of viewing $p(Z)$ is that each column (aka dish)
contains nonzero entries $Z_{ij}$ that are drawn from the binomial
distribution $Z_{ij} \sim \mathrm{Bin}(\pi_i)$. That is, if we
\emph{knew} $K_+$, i.e.\ if we knew how many nonzero features $Z$
contains, and if we knew the probabilities $\pi_i$, we could draw $Z$
efficiently from it. We take this approach in our analysis: determine
$K_+$ and infer the probabilities $\pi_i$ directly from the data. This
is more reminiscent of the model used to derive the IBP --- a
hierarchical Beta-Binomial model, albeit with a variable number of entries:
$$
\begin{tikzpicture}[>=latex,text height=1.5ex,text depth=0.25ex]
  \matrix[row sep=0.8cm,column sep=0.6cm] {
    \node (alpha) [observed]{$\alpha$}; &
    \node (pi) [latent]{$\pi_i$}; &
    \node (z) [latent]{$Z_{ij}$}; 
    \\
  };
  \path[->]
  (alpha) edge[thick] (pi)
  (pi) edge[thick] (z)
  ;
  \begin{pgfonlayer}{background}
    \node (customers) [plate, fit=(z)] {\
      \\[9.5mm] \tiny $j \in \cbr{n}$};
    \node (dishes) [plate, fit=(pi) (customers)] {\
      \\[16.5mm]\tiny $i \in \cbr{K_+}$};
    \end{pgfonlayer}
\end{tikzpicture}
$$
In general, the binary attributes $Z_{ij}$ are \emph{not}
observed. Instead, they capture auxiliary structure pertinent to a
statistical model of interest. To make matters more concrete, consider
the following two models proposed by \cite{GriGha11} and
\cite{KnoGha07}. They also serve to showcase the algorithm design in
our paper. 
\paragraph{Linear Gaussian Latent Feature Model \citep{GriGha11}.}
The assumption is that we observe vectorial data $x$. It is generated
by linear combination of dictionary atoms $\Phi$ and an associated unknown
number of binary causes $z$, all corrupted by some additive noise
$\epsilon$. That is, we assume that 
\begin{align}
  \label{eq:lingalaf}
  x = \Phi z + \epsilon
  \text{ where } \epsilon \sim \Ncal(0, \sigma^2\one)
  \text{ and } z \sim \mathrm{IBP}(\alpha).
\end{align}
The dictionary matrix $\Phi$ is considered to be fixed but unknown. In
this model our goal is to infer both $\Phi$, $\sigma^2$ and the
probabilities $\pi_i$ associated with the IBP model. Given that, a
maximum-likelihood estimate of $Z$ can be obtained efficiently. 

\paragraph{Infinite Sparse Factor Analysis \citep{KnoGha07}.} 

A second model is that of sparse independent component analysis. In a
way, it extends \eq{eq:lingalaf} by replacing binary attributes with
sparse attributes. That is, instead of $z$ we use the entry-wise
product $z \mathrm{.*} y$. This leads to the model
\begin{align}
  \label{eq:linzy}
  x = \Phi (z \mathrm{.*} y) + \epsilon
  \text{ where } \epsilon \sim \Ncal(0, \sigma^2 \one)
  \text{ , } z \sim \mathrm{IBP}(\alpha)
  \text{ and } y_i \sim p(y)
\end{align}
Again, the goal is to infer the dictionary $\Phi$, the probabilities $\pi_i$ and then to
associate likely values of $Z_{ij}$ and $Y_{ij}$ with the data. In
particular, \cite{KnoGha07} make a number of alternative assumptions
on $p(y)$, namely either that it is iid Gaussian or that it is iid
Laplacian. Note that the scale of $y$ itself is not so important since
an equivalent model can always be found by re-scaling matrix $\Phi$ suitably. 

Note that in \eq{eq:linzy} we used the shorthand $\mathrm{.*}$ to
denote point-wise multiplication of two vectors in 'Matlab'
notation. While \eq{eq:lingalaf} and \eq{eq:linzy} appear rather
similar, the latter model is considerably more complex since it not
only amounts to a sparse signal but also to an additional
multiplicative scale. \cite{KnoGha07} refer to the model as Infinite
Sparse Factor Analysis (isFA) or Infinite Independent Component
Analysis (iICA) depending on the choice of $p(y)$ respectively.

\subsection{The Hierarchical Dirichlet Process (HDP)}
\label{sec:hdp}
The HDP mixture models are useful in modeling problems involving groups of
data, where each observation within a group is drawn from a mixture
model and it is desirable to share mixture components across all the
groups. A natural application with this property is topic modeling
for documents, possibly supplemented by an ontology.
The HDP \citep{TehJorBeaBle06} uses a Dirichlet Process (DP) \citep{Antoniak74,
  Ferguson73} $G_j$ for each group $j$ of data to handle uncertainty
in number of mixture components. At the same time, in order to share mixture
components and clusters across groups, each of these DPs is drawn from a
global DP $G_0$. The associated graphical model is given below:
\begin{center}
\begin{tikzpicture}[>=latex,text height=1.5ex,text depth=0.25ex]
  \matrix[row sep=7mm,column sep=0.4cm] {
    &
    \node (gamma) [observed] {$\gamma_0$}; & 
    \node (alpha) [observed] {$\gamma_1$}; & 
    \\
    \node (H) [observed] {$H$}; & 
    \node (G0) [latent] {$G_0$}; & 
    \node (Gj) [latent] {$G_i$}; & 
    \node (theta) [latent] {$\theta_{ij}$}; & 
    \node (x) [observed] {$x_{ij}$}; &
    \\
  };
  \path[->]
  (gamma) edge[thick] (G0)
  (alpha) edge[thick] (Gj)
  (H) edge[thick] (G0)
  (G0) edge[thick] (Gj)
  (Gj) edge[thick] (theta)
  (theta) edge[thick] (x)
  ;
  \begin{pgfonlayer}{background}
    \node (child) [plate, fit=(theta) (x)] {\
        \\[8.5mm]\tiny for all $j \in \cbr{N_i}$};
      \node (parent) [plate, fit=(child) (Gj)] {\
        \\[16.0mm]\tiny for all $i \in \cbr{I}$};
    \end{pgfonlayer}
\end{tikzpicture}
\end{center}
More formally, we have the following statistical description of a two
level HDP. Extensions to more than two levels are straightforward (we
provide a general multilevel HDP spectral inference algorithm). 
\begin{enumerate*}
\item Sample $G_0|\gamma_0,H \sim \mathrm{DP}(\gamma, H)$
\item For each $i \in \cbr{I}$ do
\begin{enumerate*}
\item Sample $G_i|\gamma_1, G_0 \sim \mathrm{DP}(\gamma_0,G_0)$
\item For each $j \in \cbr{N_i}$ do
\begin{enumerate*}
\item Sample $\theta_{ij} \sim G_i$
\item Sample $x_{ij}|\theta_{ij} \sim F(\theta_{ij}),$
\end{enumerate*}
\end{enumerate*}
\end{enumerate*}
Here $H$ is the base distribution which governs the a priori
distribution over data items, $\gamma_0$ is a concentration parameter
which controls the amount of sharing across groups and $\gamma_1$ is
a concentration parameter which governs the a priori number of clusters
and a parametric distribution $F(\theta)$. This process can be
repeated to achieve deeper hierarchies, as needed.

More formally, we have the following statistical description of a $L$-level HDP.

\begin{center}
\begin{figure}
\centering
\begin{tikzpicture}[>=latex,text height=1.4ex,text depth=0.25ex]
  \matrix[row sep=5mm,column sep=3mm] {
    \node (0) {$0$}; & & & 
    \node (G0) [latent] {$G_0$}; &
    \\
    \node (1) {$1$}; & 
    \node (G00) [latent] {$G_{00}$}; &
    \node (G01) [latent] {$G_{01}$}; & &
    \node (G02) [latent] {$G_{02}$}; &
    \\
    \node (2) {$2$}; & 
    \node (G010) [latent] {$G_{010}$}; &
    \node (G011) [latent] {$G_{011}$}; &
    \node (G020) [latent] {$G_{020}$}; &
    \node (G021) [latent] {$G_{021}$}; &
    \node (G022) [latent] {$G_{022}$}; &
    \\[-2mm]
    \node {$\vdots$}; &
    \node {$\vdots$}; &
    \node {$\vdots$}; &
    \node {$\vdots$}; &
    \node {$\vdots$}; &
    \node {$\vdots$}; &
    \\[-4mm]
    \node (L) {$L-1$}; &
    \node (Gi) [latent] {$G_{\ib}$}; &
    \node {$\ldots$}; &
    \node (Gii) [latent] {$G_{\ib'}$}; &
    \node {$\ldots$}; &
    \node (Giii) [latent] {$G_{\ib''}$}; &
    \\
    \node (docs) {docs}; &
    \node (thetaij) [latent] {$\theta_{\ib j}$}; & 
    \node {$\ldots$}; &
    \node (thetaiij) [latent] {$\theta_{\ib' j}$}; & 
    \node {$\ldots$}; &
    \node (thetaiiij) [latent] {$\theta_{\ib'' j}$}; & 
    \\
    &
    \node (xij) [observed] {$x_{\ib j}$}; & 
    \node {$\ldots$}; &
    \node (xiij) [observed] {$x_{\ib' j}$}; & 
    \node {$\ldots$}; &
    \node (xiiij) [observed] {$x_{\ib'' j}$}; & 
    \\
  };
  \path[->]
  (G0) edge[thick] (G00) 
  (G0) edge[thick] (G01) 
  (G0) edge[thick] (G02) 
  (G01) edge[thick] (G010) 
  (G01) edge[thick] (G011) 
  (G02) edge[thick] (G020) 
  (G02) edge[thick] (G021) 
  (G02) edge[thick] (G022) 
  (Gi) edge[thick] (thetaij) 
  (thetaij) edge[thick] (xij) 
  (Gii) edge[thick] (thetaiij) 
  (thetaiij) edge[thick] (xiij) 
  (Giii) edge[thick] (thetaiiij) 
  (thetaiiij) edge[thick] (xiiij) 
  ;
  \begin{pgfonlayer}{background}
    \node (childi) [plate, fit=(thetaij) (xij)] {~};
    \node (childii) [plate, fit=(thetaiij) (xiij)] {~};
    \node (childiii) [plate, fit=(thetaiiij) (xiiij)] {~};
    \end{pgfonlayer}
\end{tikzpicture}
\vspace{-2mm}
\caption{Hierarchical Dirichlet Process with observations at the leaf
  nodes.
  \label{fig:tree}}
\end{figure}
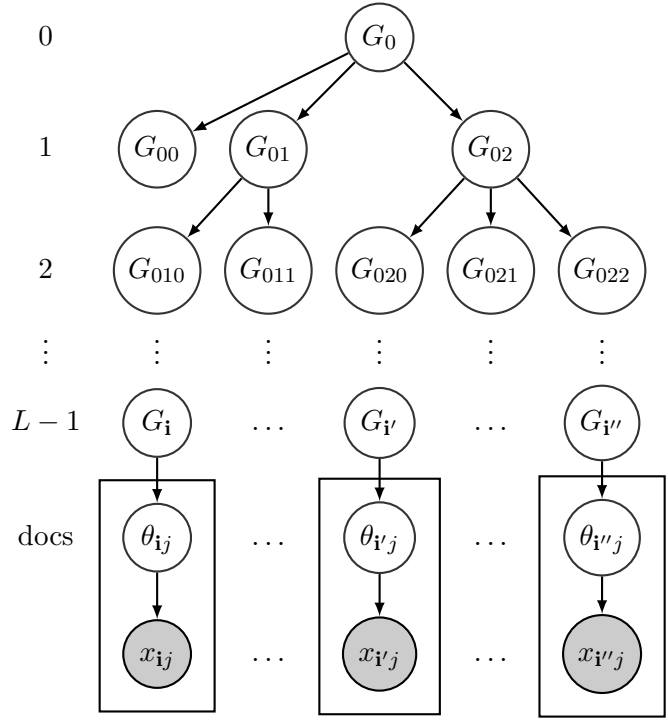
\end{center}
\paragraph{Trees}
Denote by $\Tcal = (V,E)$ a tree of depth $L$. For any vertex $\ib \in
\Tcal$ we use $p(\ib) \in V$, $c(\ib) \subset V$ and
$l(\ib)\in\cbr{0,1...,L-1}$ to denote the parent, the set of children
and level of the vertex respectively. When needed, we enumerate the
vertices of $\Tcal$ in dictionary order. For instance, the root node
is denoted by $\ib = (0)$, whereas $\ib = (0,4,2)$ is the node
obtained by picking the fourth child of the root node and then the
second child thereof respectively. Finally, we have sets of
observations $X_{\ib}$ associated with the vertices $\ib$ (in some
cases only the leaf nodes may contain observations). This yields
\begin{align*}
  G_0 & \sim \mathrm{DP}(H, \gamma_0) &
  \theta_{\ib j} & \sim G_{\ib} \\
  G_{\ib} & \sim \mathrm{DP}\rbr{G_{p(\ib)}, \gamma_{l(\ib)}} &
  x_{\ib j} & \sim \mathrm{Categorical}(\theta_{\ib j})
\end{align*}
% I think for G_{\ib} \sim \mathrm{DP}\rbr{G_{p(\ib)} the corresponding gamma should be \gamma_\{\ib}, e.g. for \ib = (0,0), G_{\ib} should be DP(G_0, \gamma_{00}) and not DP(G_0, \gamma_0)
Here $\gamma_{\ib}$ denotes the concentration parameter at vertex
$\ib$ and $H$ is the base distribution which governs the a priori
distribution over data items. Figure~\ref{fig:tree} illustrates the full model.

As explained earlier, the distributions $G_\ib$ have a stick breaking representation sharing common atoms:
\begin{equation}
G_\ib = \sum_{v=1}^\infty \pi_{\ib v}\delta_{\phi_v} \text{ with } \phi_v
\sim H.
\end{equation}

\section{Spectral Characterization}
\label{sec:moments}

We are now in a position to define the moments of the IBP and the HBP.
Our analysis begins by deriving moments
for the IBP proper. Subsequently we apply this to the two models
described above. Next, following the similar procedure, we derive the moments for the HDP. All proofs are deferred to the Appendix. For
notational convenience we denote by $\symm$ the symmetrized version of
a tensor where care is taken to ensure that existing multiplicities
are satisfied. That is, for a generic third order tensor 
we set $\symm_6[A]_{ijk} = A_{ijk} + A_{kij} + A_{jki} + A_{jik} + A_{kji} +
A_{ikj}$. However, if e.g.\ $A = B \otimes C$ with $B_{ij} = B_{ji}$,
we only need $\symm_3[A]_{ijk} = A_{ijk} + A_{jki}+ A_{kij}$ to
obtain a symmetric tensor.

\subsection{Tensorial Moments for the IBP}
In our approach we assume that $Z \sim
\mathrm{IBP}(\alpha)$. We assume that the number of nonzero attributes
$k$ is unknown (but fixed). 
In our derivation, a degeneracy in the third order tensor requires that we compute a
fourth order moment. We can exclude the cases of $\pi_i = 0$ and
$\pi_i = 1$ since the former amounts to a nonexistent feature and the
latter to a constant offset. We use $M_i$ to denote moments of order
$i$ and $S_i$ to denote diagonal(izable) tensors of order
$i$. Finally, we use $\pi \in \RR^{K_+}$ to denote the vector of probabilities $\pi_i$.
\begin{description*}
\item[Order 1] This is straightforward, since we have
  \begin{align}
    \label{eq:tensor-1}
    M_1 := \Eb_z\sbr{z} = \pi =: S_1.
  \end{align}
\item[Order 2] The second order tensor is given by
  \begin{align}
    M_2 & := \Eb_z \sbr{z \otimes z} = \pi \otimes \pi +
    \mathrm{diag}\rbr{\pi -\pi^2} = {S}_1 \otimes
    {S}_1 + \mathrm{diag}\rbr{\pi -\pi^2}.  
    \intertext{Solving for the diagonal tensor we have}
    \label{eq:tensor-2}
    S_2 & := {M}_2 - {S}_1 \otimes {S}_1 = 
    \mathrm{diag}\rbr{\pi - \pi^2}. 
  \end{align}
  The degeneracies $\cbr{0, 1}$ of $\pi - \pi^2 = (1-\pi) \pi$ can be
  ignored since they amount to non-existent and 
  degenerate probability distributions.
\item[Order 3] The third order moments yield
  \begin{align}
    {M}_3 := & \Eb_z \sbr{z \otimes z \otimes z} = \pi \otimes \pi
    \otimes \pi + \symm_3 \sbr{\pi \otimes \mathrm{diag}\rbr{\pi -
        \pi^2}}  + \mathrm{diag}\rbr{\pi - 3 \pi^2 + 2 \pi^3} \\ 
    = &  {S}_1 \otimes  {S}_1 \otimes
    {S}_1 + \symm_3 \sbr{ {S}_1 \otimes
      {S}_2} + \mathrm{diag}\rbr{\pi - 3 \pi^2 + 2 \pi^3}. \\
    {S}_3 := & {M}_3- \symm_3\sbr{{S}_1 \otimes {S}_2} - {S}_1 \otimes {S}_1 \otimes {S}_1 =  \mathrm{diag}\rbr{\pi - 3 \pi^2 + 2 \pi^3} .
    \label{eq:tensor-3}
  \end{align}
  Note that the polynomial $\pi - 3 \pi^2 + 2 \pi^3 = \pi (2\pi - 1)
  (\pi - 1)$ vanishes for $\pi = \frac{1}{2}$. This is undesirable for
  the power method --- we need to compute a
  fourth order tensor to exclude this.
\item[Order 4] The fourth order moments are 
  \begin{align}
    {M}_4 := & \Eb_z \sbr{z \otimes z \otimes z \otimes z} = {S}_1
    \otimes {S}_1 \otimes {S}_1 \otimes {S}_1 + \symm_6 \sbr{ {S}_2
      \otimes {S}_1 \otimes {S}_1} + \symm_3 \sbr{ {S}_2
      \times  {S}_2} \nonumber\\
    & + \symm_4 \sbr{ {S}_3 \otimes {S}_1} + \mathrm{diag} \rbr{
      \pi- 7\pi^2 + 12\pi^3  -6 \pi^4 } \nonumber \\
        \label{eq:tensor-4}
        {S}_4 := &  {M}_4 - {S}_1 \otimes  {S}_1 \otimes  {S}_1
        \otimes  {S}_1  - \symm_6 \sbr{ {S}_2 \otimes  {S}_1 \otimes
          {S}_1} - \symm_3 \sbr{ {S}_2 \times  {S}_2}  -  \symm_4
        \sbr{ {S}_3 \otimes  {S}_1} \nonumber \\ 
        =&  \mathrm{diag} \rbr{ \pi- 7\pi^2 + 12\pi^3  -6 \pi^4 }.
      \end{align}
      The roots of the polynomial are $\cbr{0, \frac{1}{2} - 1/\sqrt{12},
        \frac{1}{2} + 1/\sqrt{12}, 1}$. Hence the latent factors and
      their corresponding $\pi_k$ can be inferred either by ${S}_3$ or
      by ${S}_4$.
\end{description*}

\subsection{Applications of the IBP}

The above derivation showed that if we were able to access $z$ directly, we
could infer $\pi$ from it by reading off terms from a diagonal
tensor. Unfortunately, this is not quite so easy in practice since $z$
generally acts as a \emph{latent} attribute in a more complex
model. In the following we show how the models of \eq{eq:lingalaf} and
\eq{eq:linzy} can be converted into spectral form. We need some
notation to indicate multiplications of a tensor $M$ of order $k$ by a set of
matrices $A_i$. 
\begin{align}
  \label{eq:tensor-mat}
  \sbr{T(M, A_1, \ldots, A_k)}_{i_1, \ldots i_k} :=  \sum_{j_1, \ldots
    j_k} M_{j_1, \ldots j_k} \sbr{A_1}_{i_1 j_1} \cdot \ldots \cdot
  \sbr{A_k}_{i_k j_k}.
\end{align}
Note that this includes matrix multiplication. For instance, $A_1^\top
M A_2 = T(M, A_1, A_2)$. Also note that in the special case where the
matrices $A_i$ are vectors, this amounts to a reduction to a
scalar. Any such reduced dimensions are assumed to be dropped
implicitly. The latter will become useful in the context of the tensor
power method in \cite{AnaGeHsuKakTel12}. 

Here are two tensor operations that are frequently used in the derivation for linear applications of the IBP. First, for $x = Az$ (e.g. observation $x$ is a linear combination of some columns in matrix $A$ indicated by the IBP binary vector $z$), the $i$-th order moment
$M_i^x$ where the superscript denotes the variable for the moments,
can be obtained by multiplying the $i$-th order moment of $z,$ $M_i^z,$ with the affine matrix $A$ on all dimension, i.e.,
\begin{align}
M_i^z = T(M_i^z, A, \cdots, A)
\end{align}

%$ = T(M_i^z, A, \cdots, A),$ where the superscript denotes the %variable for the moments. 

Another property is addition. Suppose $y = x + \sigma$ (e.g. there exists some additional noise.), then, by using addition rule of expectation, we have 
\begin{align}
M_1^y = M_1^x + M_1^\sigma
\end{align}

Higher order moments can be obtained by taking the expansion of the polynomial expression $(x + \sigma)^{\otimes k},$ which yields
\begin{align}
&M_2^y = \Eb[(x+\sigma) \otimes (x+\sigma)] = M_2^x + M_2^\sigma + \symm_2\rbr{\Eb \sbr{x\sigma}} \\
&M_3^y = M_3^x + M_3^\sigma +  \symm_3 \sbr{x \otimes x \otimes \sigma + \sigma \otimes \sigma \otimes x} \\
&M_4^y = M_4^x + M_4^\sigma + \symm_4 \sbr{ \sigma \otimes \sigma \otimes \sigma \otimes x + \sigma \otimes x \otimes x \otimes x } + \symm_6 \sbr{\sigma \otimes \sigma \otimes x \otimes x}
\end{align}

 If $\sigma$ is Gaussian or some symmetric random variable, then its first and third moments become zero, thus the third order moment becomes $M_3^y = M_3^x + \symm_3 \sbr{ \sigma \otimes \sigma \otimes x}.$ Similarly, the forth-order moment
 reduces to  $M_4^y = M_4^x + M_4^\sigma + \symm_6 \sbr{\sigma \otimes \sigma \otimes x \otimes x}.$

\noindent {{\bfseries Linear Gaussian Latent Factor Model.}
When dealing with \eq{eq:lingalaf} our goal is to infer both $\Phi$ and
$\pi$. The main difference is that rather
than observing $z$ we have $\Phi z$, hence all tensors are
colored. Moreover, we also need to deal with the terms arising from
the additive noise $\epsilon$. This yields
\begin{align}
 \label{eq:gaussian-tensor-1}
 S_1 := & M_1 = T(\pi, \Phi) \\
 \label{eq:gaussian-tensor-2}
 S_2 := & M_2 - S_1 \otimes S_1  - \sigma^2 \one = 
 T(\mathrm{diag}(\pi - \pi^2), \Phi, \Phi) \\
 \label{eq:gaussian-tensor-3}
 S_3 := & M_3 - S_1 \otimes S_1 \otimes S_1 -\symm_3 \sbr{ S_1 \otimes
   S_2}- \symm_3 \sbr{ {m_1 \otimes \one}} =  T\rbr{\mathrm{diag}\rbr{\pi - 3\pi^2+2\pi^3}, \Phi, \Phi, \Phi}
 \\
 \label{eq:gaussian-tensor-4}
 S_4 := & M_4 - S_1 \otimes S_1 \otimes S_1 \otimes S_1 - \symm_6 \sbr{S_2 \otimes S_1 \otimes S_1} - \symm_3 \sbr{S_2 \otimes S_2} - \symm_4 \sbr{S_3 \otimes S_1} \\
&- \sigma^2\symm_6 \sbr{S_2 \otimes \one} - m_4 \symm_3 \sbr{ \one \otimes \one}
\nonumber \\
= & T\rbr{\mathrm{diag}\rbr{-6 \pi^4 +12 \pi^3 -7 \pi^2 +
    \pi}, \Phi, \Phi, \Phi, \Phi} 
\nonumber
\end{align}
Here we used the auxiliary statistics $m_1$ and $m_4$. Denote by
$v$ the eigenvector with the smallest eigenvalue of the covariance
matrix of $x$. Then the auxiliary variables are defined as
 \begin{align}
   \label{eq:gaussian-variance-m1}
   m_1 :=& \Eb_x \sbr{x \inner{v}{\rbr{x-\Eb\sbr{x}}}^2}  = \sigma^2 T(\pi, \Phi) \\
%  \label{eq:gaussian-variance-m2}
%   m_2 :=& \sigma^2 \sbr{M_2 - \sigma^2 \one} 
%  && = T(\Eb_z[z \otimes z] \sigma^2, \Phi, \Phi) \\
   \label{eq:gaussian-variance-m4}
   m_4 :=& \Eb_x \sbr{\inner{v}{\rbr{x-\Eb_x\sbr{x}}}^4} /3 = \sigma^4.
\end{align}
These terms are used in a tensor power method to infer both $\Phi$ and
$\pi.$ 

\begin{proof}
To easily apply the addition property of moments, we define $y = \Phi x.$ 
\begin{description*}
\item[Order 1 tensor:] By using Equation \eq{eq:tensor-1}, we have
 \begin{align}
  S_1 := M_1 &= \Eb_x \sbr{x} = M_1^y + M_1^\sigma =  T(\Eb[z], \Phi)= T( \pi, \Phi),
 \end{align}
 where we apply the addition property of moments in the third equation, and linear transformation property at the fourth equation.
 To infer the \emph{number} of latent variables $k$ and deal with the
 noise term, we need to determine the rank of the covariance matrix $\Eb_x
 \sbr{\rbr{x- \Eb_x{\sbr{x}}}\otimes \rbr{x-
     \Eb_x{\sbr{x}}}}$. Because there is additive noise, the smallest
 $(d-K)$ eigenvalues will not be exactly zero. Instead, they amount to
 the variance arising from $\epsilon$ since
 \begin{align}
   \mathrm{cov}[\Phi z + \epsilon] = \Phi^\top \mathrm{cov}[z] \Phi + \mathrm{cov}[\epsilon].
 \end{align}
 Consequently the smallest eigenvalues of the covariance matrix of $x$
 allow us to read off the variance $\sigma^2$: for any normal vector
 $v$ corresponding to the $d-k$ smallest eigenvalues we have
 \begin{align}
   \label{eq:gaussian-variance-1}
   \Eb_x \sbr{ \rbr{ v^\top \rbr{x-\Eb \sbr{X}}}^2}= 
   v^\top \Phi^\top \mathrm{cov}[z] \Phi v + v^\top \mathrm{cov}[\epsilon] v
   = \sigma^2.
 \end{align}
\item[Order 2 tensor:]
 Here we plug in Equation \eq{eq:tensor-2} and use
 independence of $z$ and $\epsilon$. Linear terms in $\epsilon$
 vanish. Thus we get
\begin{align}
  \nonumber
  M_2 =& M_2^y + M_2^\sigma + \Eb[y \otimes \sigma]  
         = T\rbr{ \Eb_z \sbr{z\otimes z},\Phi,\Phi}   + \sigma^2 \one \\
  =& T\rbr{ \pi \otimes \pi + \mathrm{diag}\rbr{\pi -\pi^2}, \Phi,\Phi}
        + \sigma^2 \one \\
  =& S_1 \otimes S_1 +  T\rbr{ \mathrm{diag}\rbr{\pi -
       \pi^2}, \Phi, \Phi} + \sigma^2 \one,
\end{align}
where the second equations follow the linear transformation property of moments.
This yields the statement in Equation \eq{eq:gaussian-tensor-2}.
\item[Order 3 tensor:] As before, denote by $v$ an eigenvector
  corresponding to the $(d-k)$ smallest eigenvalues, i.e.\ $v^\top \Phi=
  0$. We first define an auxiliary term
  \begin{align}
    \nonumber
    m_1 :=& \Eb_x \sbr{x \rbr{v^\top \rbr{x-\Eb{\sbr{x}}}}^2} =  \Eb_x
    \sbr{x\rbr{v^\top \rbr{\Phi\rbr{z-\pi}+\varepsilon}}^2} \\
    =& \Eb_x\sbr{x\rbr{v^\top \varepsilon}^2} = \sigma^2 T(\pi, \Phi). 
  \end{align}
  Since the Normal Distribution is symmetric, only even moments of
  $\epsilon$ survive. Using \eq{eq:tensor-3}, the third order moments yield
  \begin{align}
    M_3 & =    M_3^y  + \Eb_z \sbr{\symm_3 \sbr{\Phi z \otimes \epsilon \otimes \epsilon} }\\
    & = T\rbr{\Eb_z[z\otimes z \otimes z],\Phi,\Phi,\Phi} 
       + \symm_3 \rbr{m_1\otimes \one } \\
    &=S_1 \otimes S_1 \otimes S_1 +\symm_3 \sbr{ S_1 \otimes S_2}
    +T\rbr{ \mathrm{diag}\rbr{\pi - 3\pi_i^2+2\pi_i^3},\Phi,\Phi,\Phi}
    + \symm_3\rbr{m_1\otimes \one } 
    \nonumber
\end{align}
Thus, we get Equation \eq{eq:gaussian-tensor-3}.
\item[Order 4 tensor:] We obtain the fourth-order tensor by first
  calculating an auxiliary variable related to the additive noise term
\begin{align}
m_4 :=& \Eb_x \sbr{\rbr{v^\top \rbr{x-\Eb_x\sbr{x}}}^4}/3 =  \Eb[\rbr{v^\top \epsilon}^4]/3=   \sigma^4.
\end{align}
Here the last equality followed from the isotropy of Gaussians. 
With Equation \eq{eq:tensor-4}, the forth order moments are
\begin{align}
  \nonumber
M_4 =& M_4^y + M_4^\epsilon + \Eb_x \sbr{\symm_6 \sbr{y \otimes y \otimes \epsilon \otimes \epsilon} }\\ 
       =& T\rbr{ \Eb_z \sbr{z \otimes z \otimes z \otimes z}, \Phi, \Phi,\Phi,\Phi}
         +\sigma^2 \symm_6\sbr{S_2 \otimes \one} 
         +  \sigma^4 \symm_3 \sbr{\one\otimes \one} \nonumber\\
       =& S_1 \otimes S_1 \otimes S_1 \otimes S_1 \nonumber 
         + \symm_6 \sbr{S_2 \otimes S_1 \otimes S_1} 
         + \symm_3 \sbr{S_2 \times S_2} \nonumber 
         + \symm_4 \sbr{S_3 \otimes S_1} \\
      & +T\rbr{ \mathrm{diag}\rbr{-6 \pi^4 +12 \pi^3 -7 \pi^2 + \pi}, \Phi, \Phi, \Phi}  
         + \sigma^2\symm_6\sbr{S_2 \otimes \one} 
         +   m_4 \symm_3 \sbr{\one\otimes \one}.  \nonumber
\end{align}

\end{description*}

\end{proof}

\noindent {\bfseries Infinite Sparse Factor Analysis (isFA)}

Using the model of \eq{eq:linzy} it follows that $z$ is a
\emph{symmetric} distribution with mean $0$ provided that $p(y)$ has
this property. 
Here we state the property of moments by using such prior. For $x = z \odot y,$ 
$M_i^x = M_i^z \odot M_i^y.$ If $y$ is symmetric so that the first and the third 
order moments vanish, we have $M_1^x = M_3^x = 0.$ From that it follows that 
the first and third order
moments and tensors vanish, i.e.\ $S_1=0$ and $S_3=0$. We have the
following statistics:
\begin{align}
  \label{eq:isfa-tensor-1}
  S_2 := &  M_2 -\sigma^2 \one = T\rbr{c \cdot \mathrm{diag}(\pi), \Phi, \Phi}\\
  \label{eq:isfa-tensor-4}
  S_4 := & M_4 -\symm_3 \sbr{S_2 \otimes S_2} - 
  \sigma^2\symm_6 \sbr{S_2 \otimes \one} - m_4 \symm_3 \sbr{\one \otimes \one}
  = T\rbr{\mathrm{diag}(f(\pi)), \Phi, \Phi, \Phi, \Phi}.
\end{align}
Here $m_4$ is defined as in \eq{eq:gaussian-variance-m4}. Whenever $p(y)$ in
\eq{eq:linzy} is Gaussian, we have $c = 1$ and $f(\pi) = \pi -
\pi^2$. Moreover, whenever $p(y)$ follows the Laplace distribution, we
have $c=2$ and $f(\pi) = 24 \pi - 12 \pi^2$.

\begin{proof}
Since both $Y$ and $\epsilon$ are symmetric and have zero mean, the
odd order tensors vanish. That is $M_1 = 0$ and $M_3 = 0$. It suffices
for us to focus on the even terms.
\begin{description*}
\item[Order 2 tensor:] Using covariance matrix of \eq{eq:tensor-2} yields
  \begin{align} 
    M_2 &= \Eb_x \sbr{x \otimes x} = T\rbr{ \Eb_z{ \sbr{ \rbr{z \odot y} \otimes \rbr{z \odot y} }},\Phi,\Phi }+ \sigma^2\one \\
    &=   T\rbr{ \rbr{ \Eb_z{\sbr{z \otimes z}} \odot  \Eb_y \sbr{y^2} \one},\Phi, \Phi } + \sigma^2\one  \\
    &=T \rbr{ \rbr{ \pi \otimes \pi + \mathrm{diag}\rbr{\pi -\pi^2}} \odot \Eb_y \sbr{y^2} \one, \Phi,\Phi}  + \sigma^2\one \\
    \label{eq:y2}
    & =  T\rbr{\Eb_y \sbr{y^2} \mathrm{diag}\rbr{\pi}, \Phi,\Phi} +\sigma^2\one =  T\rbr{\mathrm{diag}\rbr{\pi}, \Phi,\Phi} +\sigma^2\one,
  \end{align}
  where the second equation follows the element-wise multiplication property of the moment.
  As before, the variance $\sigma^2$ of $\epsilon$ can be inferred by
  Equation \eq{eq:gaussian-variance-1}. Here we get Equation
  \eq{eq:isfa-tensor-1}. 
\item[Order 4 tensor:] With Equation \eq{eq:tensor-4} and $\Eb_y\sbr{y^4} = 3$, we have 
\begin{align}
  M_4 =& \Eb_x \sbr{x \otimes x \otimes x \otimes x} \nonumber\\
  =& \Eb_z{ \sbr{ \Phi \rbr{z \odot y} \otimes \Phi \rbr{z \odot y} \otimes
      \Phi \rbr{z \odot y} \otimes \Phi \rbr{z \odot y} }} \nonumber\\
  \nonumber
  & +\Eb_z{ \sbr{ \symm_6\sbr{ \Phi \rbr{z \odot y} \otimes \Phi \rbr{z \odot y} \otimes \epsilon \otimes \epsilon} }} +  \Eb \sbr{\epsilon\otimes \epsilon \otimes \epsilon\otimes \epsilon} \nonumber\\
  =& T\rbr{\Eb_z \sbr{z \otimes z \otimes z \otimes z} \odot \Eb_y \sbr{y^4} \one,\Phi , \Phi, \Phi, \Phi} + \sigma^2\symm_6\sbr{S_2 \otimes \one} +   \sigma^4 \symm_3 \sbr{\one\otimes \one}\nonumber\\
 =&   \symm_3 \sbr{S_2 \otimes S_2} +T\rbr{ \mathrm{diag}
   \rbr{\Eb_y\sbr{y^4} \pi_i - 3\Eb_y\sbr{y^2}^2 \pi_i^2}, \Phi, \Phi, \Phi, \Phi }
\nonumber\\
 \label{eq:y4}
 & + \sigma^2 \symm_6\sbr{S_2 \otimes \one} +  \sigma^4 \symm_3 \sbr{\one\otimes \one}  \\
 =&  \symm_3 \sbr{S_2 \otimes S_2} +T\rbr{3\rbr{ \pi_i - \pi_i^2}, \Phi, \Phi, \Phi, \Phi }
 + \sigma^2 \symm_6\sbr{S_2 \otimes \one}+ m_4 \symm_3 \sbr{\one\otimes \one}  \nonumber
\end{align}
where $m_4$ can be inferred by \eq{eq:gaussian-variance-m4}.
\end{description*}
If the prior on $Y$ is drawn from a Laplace distribution the model is
called an infinite Inde-pendent Component Analysis (iICA) \citep{KnoGha07}. The lower-order moments are similar to that of isFA, except for
$\Eb_y\sbr{y^2} = 2$ and $\Eb_y\sbr{y^4} = 24$. Replacing these terms
in Equation \eq{eq:y2} and \eq{eq:y4} yields the claim.
\end{proof}

\begin{lemma}
  \label{lem:bound}
  Any linear model of the form \eq{eq:lingalaf} or \eq{eq:linzy} with
  the property that $\epsilon$ is symmetric and satisfies
  $\Eb[\epsilon^i] = 0$ for $i \in \{1,3, 5,\cdots\},$ the same
  properties for $y$, will yield the same moments.
\end{lemma}
\begin{proof}
  This follows directly from the fact that $z$, $\epsilon$ and $y$
  are independent and that the latter two have zero mean and are
  symmetric. Hence the expectations carry through regardless of the
  actual underlying distribution.
\end{proof}
\subsection{Tensorial Moments for the HDP}
\label{sec:spec}
To construct tensors for the HDP,  a crucial step is to derive the orthogonally
decomposable tensors from the moments. 
%{\color{red} check $k!$}
\begin{description*}
\item[Order 1 tensor:] The first-order moment is equivalent to the weighted sum of latent topics using a topic distribution under node $\ib,$ so it is simply the weighted combination of $\Phi$ where the weight vector is $\pi_0,$ i.e,
\begin{align}
M_1 := \Eb \sbr{x} = \Eb \sbr{\Phi h_{\ib j}} = \Eb \sbr{\Phi \pi_{\ib}} = \Phi \pi_{0}.
\end{align}
The last equation uses the fact that, for $\pi \sim \mathrm{Dirichlet}( \gamma_0 \pi_0), $ $\Eb[\pi] = \pi_0.$
 \item[Order 2 tensor:] For such variable $\pi,$ using the definition of Dirichlet distribution, we have $\Eb[[\pi^2]_{ii}] = \frac{\gamma_0}{\gamma_0 + 1} \pi_{0i} (\pi_{0i}+1)$ and $\Eb[[\pi^2]_{ij}] = \frac{\gamma_0}{\gamma_0 + 1} \pi_{0i} \pi_{0j}.$ The second-order moment thus becomes
\begin{align}
M_2 :=\Eb \sbr{x_1 \otimes x_2} = \Eb \sbr{\Phi h_{\ib 1} h_{\ib 2}^T \Phi^T } = \Phi \Eb \sbr{ \pi_{\ib} \pi_{\ib}^T }\Phi^T = \Phi A \Phi^T, 
\end{align}
where $[A]_{ii} = \frac{1}{\gamma_0 + 1} \pi_{0i} (\gamma_0\pi_{0i}+1) $ and $[A]_{ij} = \frac{\gamma_0}{\gamma_0 + 1} \pi_{0i} \pi_{0j}.$ Matrix $A$ can be decompose as the summation of a diagonal matrix and a symmetric matrix, $\pi_0 \otimes \pi_0.$ By replacing $A$ with these two matrices, the second-order moment can be re-written as 
\begin{align}
 M_2 = \Phi A \Phi^T &= \Phi \rbr{\frac{\gamma_0}{\gamma_0 + 1}\pi_0 \otimes \pi_0 + \frac{1}{\gamma_0 + 1} \mathrm{diag}(\pi_0) }\Phi^T,
\end{align}
where $\Phi \pi_0$ in the first term can be further replaced with $M_1.$ Thus, we define the second term as the second-order tensor, which is a rank-$k$ matrix, 
\begin{align}
S_2 := M_2 - \frac{\gamma_0}{\gamma_0 + 1} M_1 \otimes M_1 = \Phi \rbr{\frac{1}{\gamma_0 + 1} \mathrm{diag}(\pi_0) }\Phi^T  = \sum \limits_{i = 1}^K \frac{\pi_{0i}}{\gamma_0+1} \phi_i \otimes \phi_i.
\end{align}
\item[Order 3 tensor:] The third-order tensor is defined in the form of $S_3 := \sum_{i = 1}^K C_6 \cdot \pi_{0i} \cdot \phi_i \otimes \phi_i \otimes \phi_i,$ and can be derived using $M_1,$ $M_2$ and $M_3$ by applying the same technique of decomposing matrix or tensor into the summation of symmetric tensors and diagonal tensor. The derivation details for a multi-layer HDP tensor is provided in the Appendix. 

\end{description*}

\noindent Before stating the generalized tensors for the HDP, we define $M_r^{\ib}$ as the $r$-th moment at node $\ib$. The moment can be obtained by averaging corresponding moments of its child nodes.

\begin{align}
\label{eq:mi-general}
M_r^{\ib} := \frac{1}{|c(\ib)|} \sum \limits_{\jb \in c(\ib)} M_r^{\jb}
\end{align}

starting with $M_r^{\ib} = \Eb \sbr{\otimes_{s=1}^r x_{\ib s}} $ whenever $\ib$ represents an leaf node. In other words, for a $L$-layer model, after obtaining moments at the leaf nodes (e.g. moments on layer $L-1$), we are able to calculate moments, $M_r^{\ib}$, for node $\ib$ on layer $L-2$, by averaging the associated moments over all of its children.

Lemma \ref{lem:hdp3moments_general} shows the generalized tensors for HDP with different number of layers. Using Lemma \ref{lem:hdp3moments_general}, we found that the coefficient and moment for different hierarchical tree can be derived recursively using a bottom-up approach, i.e., coefficient for $k-$layer HDP can be derived using the coefficient of $(k-1)$-layer HDP and moments at a node $\ib$ can be derived using the moments calculating under its children, $c(\ib)$. The recursive rule is provided in Lemma \ref{lem:hdp3moments_general}.

\vspace{-2mm}
\begin{lemma}[Symmetric Tensors of the HDP]
  \label{lem:hdp3moments_general}
Given a $L$-level HDP, with hyperparameters $\gamma_{\ib}$, the symmetric Tensors for a node $\ib$ at layer $l$ can be expressed as:
\begin{align*}
   	 S_1^{\ib} &:= M_1^{\ib} = T(\pi_{\ib}, \Phi ),\,\, S_2^{\ib} := M_2^{\ib} - C_2^{l} \cdot S_1^{\ib} {S_1^{\ib}}^T = T( C_3^{l} \cdot \mathrm{diag}\rbr{\pi_{\ib} }, \Phi,\Phi),\\
	S_3^{\ib} & :=  M_3^{\ib} - C_4^{l}  \cdot S_1^{\ib} \otimes S_1^{\ib} \otimes S_1^{\ib} -C_5^{l}  \cdot \symm_3\sbr{ S_2^{\ib} \otimes M_1^{\ib}}= T(  C_6^{l} \cdot  \mathrm{diag}\rbr{\pi_{\ib}}, \Phi,\Phi,\Phi),
  \end{align*}
where
\begin{align}
C_2^{(l)} &= \frac{\gamma_{l+1} }{ \gamma_{l+1} +1 } C_2^{(l+1)},\,\, C_3^{(l)} =  C_3^{(l+1)} + \frac{C_2^{(l)}}{\gamma_{l+1} }, \,\,C_4^{(l)} = \frac{{\gamma_{l+1}}^2  }{ \rbr{\gamma_{l+1} +1}\rbr{\gamma_{l+1} +2} }  C_4^{(l+1)} \nonumber\\
C_5^{(l)} &= \frac{\gamma_{l+1} }{ \rbr{\gamma_{l+1} +1} } \frac{C_3^{(l+1)}}{C_3^{(l)}} C_5^{(l+1)} + \frac{1}{\gamma_{l+1}C_3^{(l)}}C_4^{(l)}, \,\,C_6^{(l)} = C_6^{(l+1)} +  3 \cdot \frac{ C_5^{(l)}C_3^{(l)}}{\gamma_{l+1}} - \frac{C_4^{(l)}}{\gamma_{l+1}^2} \nonumber
\end{align}
with initialization on the bottom layer ($(L - 1)$-layer) being $C_2^{L-1} = 1, $ $C_3^{L-1} =0 ,$ $C_4^{L-1} =1 ,$ $C_5^{L-1} =0,$
and $C_6^{L-1} = 0.$
\end{lemma}
%

%Here $M_r^{\ib}$ denotes the $r$-th moment at node $\ib$
%\begin{align}
%\label{eq:mi-general}
%M_r^{\ib} &:=  \frac{1}{|c({\ib})|}\sum\limits_{\mathbf{j} \in c(\ib) } M_r^{\mathbf{j}}
%\end{align} 
%starting with $M_r^{\ib}= \mathbb{E}[\bigotimes_{s=1}^r x_{\ib s}]$
%whenever $\ib$ represents an leaf node. In other words, for a $L$-layer model, after obtaining moments at the leaf nodes (e.g. moments on layer $L-1$), we are able to calculate moments, $M_r^\ib$, for node $\ib$ on layer $L-2$, by averaging the associated moments over all of its children.  
%

\section{Spectral Algorithms for the IBP and the HDP}
\vspace{-2mm}
\label{sec:alg}
Here we introduce a way to estimate moments on the leaf nodes, which 
are used to estimate the diagonalized tensors.  
Next, we provide two simple methods for estimating number of topics, $k.$ 
Finally we review Excess Correlation Analysis (ECA) and several tensor 
decomposition techniques that are used obtained the estimated topic vectors.
\paragraph{Moment estimation} 
For the IBP, we can directly estimate the moments by replacing the theoretical moments  with its emperical version. The interesting part comes in the moment estimation for multi-layer HDP.
A $L$-level HDP could be viewed as a $L$-level tree, where each node
represents a DP.  The estimated moments for the whole model can
be calculated recursively by Equation \eqref{eq:mi-general} and the empirical
$r$-th order moments at the leaf node $\ib$ which are defined as:
\begin{align*}
  \hat{M}_r^{\ib} & := \varphi_r\rbr{\mathbf{x}_{\ib} } \text{ where } 
  \varphi_r(\mathbf{x}_{\ib}) := \frac{(m_\ib-r) !}{m_\ib !}\sum \limits_{j_1,j_2}x_{\ib j_1} \otimes x_{\ib j_2} \cdots  \otimes x_{\ib j_r},
\end{align*}
where $m_\ib$ is the number of words in the observation $x_{\ib}$.
Here $(x_{\ib j_1}, x_{\ib j_2}, \cdots, x_{\ib j_r})$ denote the ordered tuples in
$\mathbf{x_{\ib}}$, with $x_{\ib j}$ encoded as a binary vector, i.e.\ $x_{\ib j}= e_k$
iff the $j$-th data is $k$. The empirical tensors is obtained by plugging in these empirical 
moments to the tensor equations derived in the previous section.
The concentration of measure bounds for these estimated
quantities are given in Section \ref{sec:conc}.
%$S_i$ could be then estimated from Equation \ref{eq:S}. 
%ALEX: commented out since it doesn't show up in the proof and is a
%reference to the internal details of a proof ...

\paragraph{Inferring the number of mixtures} 
We first present the method of inferring the number of
latent features, $K$, which can be viewed as the rank of the covariance
matrix, for models with additive noise. 
An efficient way of avoiding eigen decomposition on a $d \times
d$ matrix is to find a low-rank approximation $R \in \mathbb{R}^{d \times K'}$ 
such that $K < K' \ll d$ and $R$ spans the same space as the covariance
matrix. One efficient way to find such matrix is to set $R$ to be
\begin{equation}
   R = \rbr{M_2 - M_1 \times M_1}\Theta,
\end{equation}
where $\Theta \in \mathbb{R}^{d \times K'}$ is a random matrix with
entries sampled independently from a standard normal. This is
described, e.g. by \cite{HalMarTro09}. Since there is noise in the
data, it is not possible that we get exactly $K$ non-zero eigenvalues
with the remainder being constant at noise floor $\sigma^2$. An
alternative strategy to thresholding by $\sigma^2$ is to determine $K$
by seeking the largest slope on the curve of sorted eigenvalues.

For the HDP, in contrast to the Chinese Restaurant Franchise where the number 
of mixture components, $k,$ is
settled by means of repeated sampling in the sampling-based algorithms, we use an approach that
directly infers $k$ from data itself.
The concatenation of all the first-order moments spans the space of $\Phi$ with high probability. Thus, the number of
linearly independent mixtures $k,$ is close to the rank of the matrix,
$\tilde{M}_1,$ where each column corresponds to the first order 
moments on one of the leaf nodes.
 While direct calculation of the rank of $\tilde{M}_1$
is expensive, one can estimate $k$ by the following procedure:
draw a random matrix $R \in \mathbb{R}^{n_l \times k'}$ for some $k'
\geq k,$ and examine the eigenvalues of $ \tilde{M_1'} =
\rbr{\tilde{M_1}R}^T\rbr{\tilde{M_1}R} \in \mathbb{R}^{k' \times
  k'}$. We estimate the rank of $\tilde{M_1}$ to be the point where the magnitude of eigenvalues decrease abruptly.

\paragraph{Excess Correlation Analysis (ECA)}
We then apply Excess Correlation Analysis (ECA) to infer hidden topics, $\Phi.$
 Dimensionality reduction and whitening is then performed on the diagnoalized tensor 
 at the root node, i.e., $\hat{S}_r^0,$ to make the eigenvectors of it orthogonal and 
 to project to a lower dimensional space. We whiten the observations by multiplying 
 data with a whitening matrix, $W \in \mathrm{R}^{d \times K}$. 
 This is computationally efficient, since we can apply this
directly to $x$, thus yielding third and fourth order tensors $W_3$
and $W_4$ of size $k \times k \times k$ and 
$k \times k \times k \times k$, respectively. 
Moreover, approximately factorizing $S_2$ is a
consequence of the decomposition and random projection techniques
arising from \cite{HalMarTro09}.

To find the singular vectors of $W_3$ and $W_4,$ we use tensor decomposition
techniques to obtain their eigenvectors. From the eigenvectors we
found in the last step, $\Phi$ could be recovered by multiplying a weighted 
inverse matrix, $W^\dag$. The fact that this algorithm
only needs projected tensors makes it very efficient. Streaming
variants of the robust tensor power method are subject of future
research. We introduce 
the tensor decomposition techniques for the need of our algorithms.

\paragraph{Tensor Decomposition} 
With the derived symmetric tensors, we need to separate the hidden
vectors $\Phi,$ the latent distribution $\pi$, and the additive
noise, as appropriate. In a nutshell the approach is as follows: we
first identify the noise floor using the assumption that the number of
nonzero probabilities in $\pi$ is lower than the dimensionality of the
data. Secondly, we use the noise-corrected second order tensor to
whiten the data. This is akin to methods used in ICA
\citep{Cardoso98}. Finally, we perform tensor decomposition on the 
data to obtain \smash{$S_3$ and $S_4$,} or rather, their applications to
data. Note that the eigenvalues in the re-scaled tensors differ slightly
since we use {$S_2^{\dag \frac{1}{2}} x$} directly rather than $x$.

There are several tensor decomposition algorithms that can be applied. 
 \cite{AnaGeHsuKakTel12} showed that robust  tensor power method has 
 nice theoretical convergence property. However, this approaches is slow in 
 practice. An alternative is alternating least square (ALS), which expend the 
 third order tensor into matrix and treat the tensor decomposition as 
  a least square problem. However, ALS is not stable and does not guarantee 
  to converges to the global minima. Recently, \cite{WanTunSmoAna15} proposed 
  a fast tensor power method using count sketch with FFT. The method is 
  shown to be faster than the robust tensor power method by a factor of $10$ to $100.$. 
  In this work, we show how different solvers affect the performance in both time and perplexity. We briefly review these solvers.

\paragraph{Tensor Decomposition 1: Robust Tensor Power Method}
\label{sec:rptm}

Our reasoning follows that of \cite{AnaGeHsuKakTel12}. It is our goal
to obtain an \emph{orthogonal} decomposition of the tensors $S_i$
into an orthogonal matrix $V$ together with a set of corresponding
eigenvalues $\lambda$ such that $S_i = T[\mathrm{diag}(\lambda),
V^\top, \ldots, V^\top]$. This is accomplished by generalizing the
Rayleigh quotient and power iterations
described in \citep[Algorithm 1]{AnaGeHsuKakTel12}:
\begin{align}
  \label{eq:tensorpower}
  \theta \leftarrow {T[S, \one, \theta, \ldots,
    \theta]}
  \text{ and }
  \theta \leftarrow \nbr{\theta}^{-1} \theta.
\end{align}
In a nutshell, we use a suitable number of random initialization $L$,
perform a few iterations ($T$) and then proceed with the most
promising candidate for another $T$ iterations. The rationale for
picking the best among $L$ candidates is that we need a high
probability guarantee that the selected initialization is
non-degenerate. After finding a good candidate and normalizing its
length we deflate (i.e.\ subtract) the term from the tensor $S$.

\paragraph{Tensor Decomposition 2: Alternating Least Square (ALS)}
\label{sec:td2}
Another commonly used method for
solving tensor decomposition is alternating least square method. The
main idea is to concatenate the tensors into a matrix and then
minimize the Frobenius norm of the difference:
\begin{align}
\label{eq:als}
  &\min \nbr{ [S_{3}(W,W,W)]_{(1)} - V \mathrm{diag}({\bf \lambda}) (V\odot V)^T}_F 
  \end{align}
where the definition of the operators used are:
\begin{align}
 &S_{(1)} = \Big[ S[:,:,1] \,\,\, S[:,:,2] \, \cdots \, S[:,:,K]\Big]  \\
 & V \odot V = [v_1 \boxplus v_1 \,\,\,  v_2 \boxplus v_2 \, \cdots \, v_K \boxplus v_K ].
\end{align}
The notation $\odot$ denotes the Khatri-Rao product and $\boxplus$ denotes 
the Kronecker product. Taking the second and third $V$ in the objective function
\eqref{eq:als} as some fixed matrices, we get the closed form solution of the
 optimization problem as:
\begin{align}
  \nonumber
  V \mathrm{diag}({\bf \lambda}) = [S_{3}(W,W,W)]_{(1)}\rbr{V \odot V}\rbr{(V^TV).\wedge2 }^{\dag}
\end{align}
where the notation $.\wedge$ denoting point-wise power. By iteratively updating matrix$V$ until it converges, we solve the optimization problem in \eqref{eq:als}.

\paragraph{Tensor Decomposition 3: Fast Tensor via sketching (FC) }
\label{sectd3}
\cite{WanTunSmoAna15} introduced a tensor CANDECOMP/PARAFAC (CP)
 decomposition algorithm based on tensor sketching. Tensor sketches are constructed 
 by hashing elements into fixed length sketches by their index. With the special property
  of count sketch, power iteration described in Equation \eqref{eq:tensorpower} is
   transformed into convolution operators and can be calculated using FFT and 
   inverse FFT. The method is faster than standard Robust Tensor Power Method by a factor of $10$ to $100$.
%\begin{equation}
%		\lambda_i = \frac{-2 \pi_i +1}{\sqrt{-\pi_i^2+\pi_i}} \mapsto \lambda_i = \frac{6 \pi_i^2 - 6\pi +1}{-\pi_i^2+\pi_i}.
%\end{equation}

\noindent {\bfseries Further Details on the projected tensor power method.}  
Explicitly calculating tensors $M_2,M_3,M_4$ is not practical in high
dimensional data. It may not even be desirable to compute the
projected variants of $M_3$ and $M_4$, that is, $W_3$ and $W_4$ (after
suitable shifts). Instead, we can use kernel tricks to
simplify the tensor power iterations to
\begin{align}
  \nonumber
  W^\top T(M_l, \one, Wu, \ldots,  W u) 
  = \frac{1}{m} \sum_{i=1}^m {W^\top x_i} 
  \inner{ x_i}{W u}^{l-1}
  = \frac{W^\top}{m} \sum_{i=1}^m {x_i} 
  \inner{W^\top x_i}{u}^{l-1}
\end{align}
By using incomplete expansions memory complexity and storage are
reduced to $O(d)$ per term. Moreover, precomputation is $O(d^2)$ and
it can be accomplished in the first pass through the data.
The overall algorithms for the spectral algorithms for linear-Gaussian models with IBP 
prior and the HDP are shown in Algorithm \ref{alg:eca} and 
Algorithm \ref{alg:hdpeca}, respectively. 

\begin{algorithm}[t]
\caption{Excess Correlation Analysis for Linear-Gaussian model with IBP prior
      \label{alg:eca}}
\textbf{Inputs: } the moments $M_1, M_2, M_3, M_4$.
\begin{algorithmic}[1]
   \STATE \textbf{Infer K and $\sigma^2$:}  
   \STATE Optionally find a subspace $R \in \mathbb{R}^{d \times K'}$ with
   $K<K'$ by random projection.
   \[
     \mathsf{Range}\rbr{R} = \mathsf{Range}\rbr{M_2 - M_1 \otimes M_1}
     \text{ and project down to } R
     \]
   \STATE Set $\sigma^2 := \lambda_{\mathrm{min}} \rbr{M_2 - M_1 \otimes M_1}$
   \STATE Set $S_2 = \rbr{M_2 - M_1 \otimes M_1 - \sigma^2 \one}_\epsilon$
   by truncating to eigenvalues larger than $\epsilon$
   \STATE Set $K = \mathop{\mathrm{rank}} S_2$ 
   \STATE Set $W = U\Sigma^{-\frac{1}{2}}$, where $[U, \Sigma] = \mathrm{svd}(S_2)$
   \STATE \textbf{Whitening:} (best carried out by preprocessing $x$)
   \STATE Set $W_3 := T(S_3, W, W, W)$ and $W_4 := T(S_4, W, W,
   W, W)$
   \STATE \textbf{Tensor Decomposition: } 
   \STATE Compute top $K_1$ (eigenvalues, eigenvector) pairs of $W_3$ such that all the eigenvalues has absolute value larger than $1$.
   \STATE Deflate $W_4$ with $(\lambda_i, v_i)$ for all $i \leq K_1$ and obtain top $K - K_1$ (eigenvalue, eigenvector) pairs $(\lambda_i, v_i)$ of deflated $W_4$
   \STATE \textbf{Reconstruction: } With corresponding eigenvalues $\cbr{\lambda_1, \cdots, \lambda_K}$, return the set $A$: 
   \begin{align}
     \Phi= \cbr{\frac{1}{Z_i} \rbr{W^{\dag}}^\top v_i: v_i \in \Lambda} \label{eq:findA}
   \end{align}
     where $Z_i = \sqrt{\pi_i - \pi_i^2}$ with $\pi_i = f^{-1}(\lambda_i)$.  $f(\pi) = \frac{-2\pi+1}{\sqrt{\pi - \pi^2}} $ if $i \in \sbr{K_1}$ and $f(\pi)= \frac{6\pi^2-6\pi+1}{\pi-\pi^2} $ otherwise. (The proof of Equation \eq{eq:findA} is provided in the Appendix.)
\end{algorithmic}
\end{algorithm}

\begin{algorithm}
\caption{Spectral Algorithm for HDP}
\label{alg:hdpeca}
\begin{algorithmic}[1]
    \REQUIRE Observations $\mathbf{x}$
    \STATE \textbf{Inferring mixture number} \\
    Using all leaf nodes $\ib_j$ of the HDP tree compute the rank $k$ of 
    $\tilde{M}_1:=\sbr{ M_1^{\ib_1}, M_1^{\ib_2},\cdots
      M_1^{\ib_{n^{L-2}}}}$.
    \STATE \textbf{Moment estimation} \\
    Compute moment estimates $\hat{M}_r^{0}$ and tensors $\hat{S}_r^{0}$.
   \STATE \textbf{Dimensionality reduction and whitening} \\
   Find $W\in \mathbb{R}^{d\times k}$ such that $W^T\hat{S}_2^0W = I_k$.
   \STATE \textbf{Tensor decomposition} \\
   Obtain eigenvectors $v_i$ and eigenvalues $\lambda_i$ of $\hat{S}_3^0$.
   \STATE \textbf{Reconstruction}
   \begin{align}
     \label{eq:find0}
     \text{Result set }
     \hat{\Phi} = \cbr{ \lambda_i \frac{C_3}{C_6} \rbr{W^{\dag}}^Tv_i},
   \end{align}
   where $C_3$ and $C_6$ are coefficients defined in Lemma
   \ref{lem:hdpgmoments}. 
\end{algorithmic}
\end{algorithm}

%=====================================================
%===Concentration of Measure Bounds
%=====================================================

\section{Concentration of Measure Bounds}
\label{sec:concentration}
There exist a number of concentration of measure inequalities for
\emph{specific} statistical models using rather specific moments
\citep{AnaFosHsuKakLiu12}. In the following we derive a general tool
for bounding such quantities, both for the case where the statistics
are bounded and for unbounded quantities alike. Our analysis borrows
from \cite{AltSmo06} for the bounded case, and from the average-median
theorem, see e.g. \cite{AloMatSze99}, for dealing with unbounded random variables with bounded higher order moments.
%\subsection{Concentration Measure of the IBP}
\subsection{Concentration measure of moments}
\subsubsection{Bounded Moments}
We begin with the analysis for bounded moments. Denote by $\phi: \Xcal
\to \Fcal$ a set of statistics on $\Xcal$ and let $\phi_r$ be the
$r$-times tensorial moments obtained from $x$. 
\begin{align}
  \phi_1(x) & := \phi(x);  \; \;\;  \phi_2(x) := \phi(x) \otimes
  \phi(x); \; \;\;  \phi_r(x) := \phi(x) \otimes \ldots \otimes \phi(x)
\end{align}
In this case we can define inner
products via 
\begin{align*}
  k_l(x,x') := \inner{\phi_r(x)}{\phi_r(x')} = 
  T[\phi_r(x), \phi(x'), \ldots, \phi(x')] = 
  \inner{\phi(x)}{\phi(x')}^r = k^r(x,x')
\end{align*}  
as reductions of the statistics of order $l$ for a kernel $k(x,x') :=
\inner{\phi(x)}{\phi(x')}$. Finally, denote by
\begin{align}
  \label{eq:mi-general2}
  M_r := \Eb_{x \sim p(x)}[\phi_r(x)]
  \text{ and }
  \hat{M}_r := \frac{1}{m} \sum_{j=1}^m \phi_r(x_j)
\end{align}
the expectation and empirical averages of $\phi_r$. Note that these
terms are identical to the statistics used in
\cite{GreBorRasSchetal12} whenever a polynomial kernel is used. It is
therefore not surprising that an analogous concentration of measure
inequality to the one proven by \cite{AltSmo06} holds:
\begin{theorem}
\label{th:momentbounds}
Assume that the sufficient statistics are bounded via $\nbr{\phi(x)}
\leq R$ for all $x \in \Xcal$. With probability at most $1 - \delta$
the following guarantee holds: 
  \begin{align}
    \nonumber
    \Pr\cbr{\sup_{u: \nbr{u} \leq 1} \abr{ T(M_r, u,\cdots, u) -T(\hat{M}_r, u, \cdots, u)} > \epsilon_l} \leq \delta
    \text{ where }
    \epsilon_r \leq \frac{\sbr{2 + \sqrt{-2\log\delta}}{R^r}}{\sqrt{m}}.
  \end{align}
\end{theorem}
\begin{proof}
Denote by $X$ the $m$-sample used in generating $\hat{M}_r$. Moreover,
denote by
\begin{align}
  \label{eq:xi-dev}
  \Xi[X] := \sup_{u: \nbr{u} \leq 1} \abr{T[M_r, u,\cdots, u] -T[\hat{M}_r, u, \cdots, u]}
\end{align}
the largest deviation between empirical and expected moments, when
applied to the test vectors $u$. Bounding this quantity directly is
desirable since it allows us to avoid having to derive
\emph{pointwise} bounds with regard to $M_r$. We prove that 
$\Xi[X]$ is concentrated using the bound of 
\cite{McDiarmid89}. Substituting single observations in $\Xi[X]$ yields
\begin{align}
  \abr{\Xi[X] - \Xi[(X \backslash \cbr{x_j}) \cup \cbr{x'}]} 
  & \leq  \frac{1}{m} \sbr{T\sbr{\phi_r(x_j) - \phi_r(x'), u, \ldots u}} \\ 
  & \leq  \frac{1}{m} \sbr{\nbr{\phi(x_j)}^r + \nbr{\phi(x')}^r} 
  \leq \frac{2}{m} R^r.
\end{align}
Plugging the bound of $2 R^r/m$ into McDiarmid's theorem shows that
the random variable $\Xi[X]$ is concentrated for $\Pr\cbr{\Xi[X] -
  \Eb_{X}[\Xi[X]] > \epsilon} \leq \delta$ with probability $\delta
\leq \exp\rbr{-\frac{m \epsilon^2}{2 R^{2r}}}$. Solving the bound for
$\epsilon$ shows that with probability at least $1-\delta$ we have
that $\epsilon \leq \sqrt{-2 \log \delta/m} R^r$.

The next step is to bound the expectation of $\Xi[X]$. For this we
exploit the ghost sample trick and the convexity of expectations. This
leads to the following:
\begin{align}
  \Eb_{X}\sbr{\Xi[X]} 
  \leq&  \Eb_{X, X'} \sbr{\sup_{u: \nbr{u} \leq 1}
    \abr{T[M_r, u,\cdots, u] -T[\hat{M}_r, u, \cdots, u]}} \nonumber \\
  =&  \Eb_{\sigma} \Eb_{X, X'} \sbr{\sup_{u: \nbr{u} \leq 1}
    \abr{\frac{1}{m} \sum_{j=1}^m \sigma_j \rbr{T[\phi_r(x_j),u,\cdots, u] -
          T[\phi_r(x_j')u,\cdots, u ]}}} \nonumber \\
    \leq & \frac{2}{m} \Eb_{\sigma} \Eb_{X} \sbr{\sup_{u: \nbr{u} \leq 1}
      \abr{\sum_{j=1}^m \sigma_j T[\phi_r(x_j), u,\cdots, u]}} \\
    \leq & \frac{2}{m} \Eb_{\sigma} \Eb_{X} \sbr{\nbr{\sum_{j=1}^m
        \sigma_j \phi_r(x_j)}}
    \leq \frac{2}{m} \Eb_{X} \sbr{\Eb_{\sigma} \sbr{\nbr{\sum_{j=1}^m
        \sigma_j \phi_r(x_j)}^2}}^{\frac{1}{2}} 
  \leq \frac{2 R^r}{\sqrt{m}}
  \end{align}
  Here the first inequality follows from convexity of the
  argument. The subsequent equality is a consequence of the fact that
  $X$ and $X'$ are drawn from the same distribution, hence a swapping
  permutation with the ghost-sample leaves terms unchanged; The
  following inequality is an application of the triangle
  inequality. Next we use the Cauchy-Schwartz inequality, convexity
  and last the fact that $\nbr{\phi(x)} \leq R$. Combining both bounds
  yields $\epsilon \leq \sbr{2 + \sqrt{-2 \log \delta}} R^r/\sqrt{m}$.
\end{proof}

\vspace{-5mm}
Using tensor equations derived in Section \ref{sec:moments}, this means that we have concentration of
measure immediately for the symmetric tensors $S_1, \ldots S_4$. 
 In particular, we need a chaining result that allows us to compute bounds for products of terms
efficiently. 
To prove the guarantees for tensors, we rely on the triangle
inequality on tensorial reductions
\begin{align}
  \sup_u \abr{T(A+B,u) -T(A'+B',u)}
 \leq
  \sup_u \abr{T(A,u) - T(A',u)} + 
  \sup_u \abr{T(B,u) - T(B',u)} \nonumber
\end{align}
and moreover, the fact that for products of bounded random variables
the guarantees are additive, as stated in the lemma below:
\begin{lemma}
  \label{lem:chaining}
  Denote by $f_i$ random variables and by $\hat{f}_i$ their
  estimates. Moreover, assume that each of them is bounded via $|f_i|
  \leq R_i$ and $|\hat{f}_i| \leq R_i$ and 
  \begin{align}
    \label{eq:bound-eff}
    \Pr\cbr{|\Eb[f_i] - \hat{f}_i| > \epsilon_i} \leq \delta_i.
  \end{align}
  In this case the product is bounded via
  \begin{align}
    \label{eq:product}
    \Pr\cbr{\abr{\prod_i \Eb[f_i] - \prod_i \hat{f}_i} > \epsilon} \leq
    \sum_i \delta_i \;\;\;
 \text{ where }
    \epsilon = \sbr{\prod_i R_i} \sbr{\sum_i \frac{\epsilon_i}{R_i}}
  \end{align}
\end{lemma}
\begin{proof}
  We prove the claim for two variables, say $f_1$ and $f_2$. We have
  \begin{align*}
    \abr{\Eb[f_1] \Eb[f_2] - \hat{f}_1 \hat{f}_2} &\leq
    \abr{(\Eb[f_1] - \hat{f}_1) \Eb[f_2]} + \abr{\hat{f}_1 (\Eb[f_2] - \hat{f}_2)}
    \leq
    \epsilon_1 R_2 + R_1 \epsilon_2
  \end{align*}
  with probability at least $1 - \delta_1 - \delta_2$, when applying
  the union bound over $\Eb[f_1] - \hat{f}_1$ and $\Eb[f_2] -
  \hat{f}_2$ respectively. Rewriting terms yields the claim for $n =
  2$. To see the claim for $n > 2$ simply use the fact that we can
  decompose the bound into a chain of inequalities involving exactly
  one difference, say $\Eb[f_i] - \hat{f}_i$ and $n-1$ instances of
  $\Eb[f_j]$ or $\hat{f}_j$ respectively. We omit details since they
  are straightforward to prove (and tedious).
\end{proof}
By utilizing an approach similar to
\cite{AnaFosHsuKakLiu12}, overall guarantees for 
reconstruction accuracy can be derived.

\subsubsection{Unbounded Moments}
We are interested in proving concentration of measure for the following four tensors in \eq{eq:gaussian-tensor-2}, \eq{eq:gaussian-tensor-3},
\eq{eq:gaussian-tensor-4} and one scalar in
\eq{eq:gaussian-variance-1}. Whenever the statistics are unbounded,
concentration of moment bounds are less trivial and require the use of
subgaussian and gaussian inequalities \citep{HsuKakZha09}. We derive a
bound for fourth-order subgaussian random variables (previous work
only derived up to third order bounds). Lemma \ref{lem:submoment} and
\ref{lem:moments} has details on how to obtain such guarantees. 

\paragraph{Concentration measure of unbounded moments for the spectral IBP}
 Here we demonstrate and example for linear model with Gaussian noise. The concentration behavior is more complicated than that of the bounded moments in Theorem \ref{th:momentbounds} due to the additive Gaussian noise. Here we restate the model as
  \begin{align}
    x= \Phi z + \epsilon.
  \end{align}
In order to utilize the bounds for Gaussian random vectors, we need to bound the difference between empirical moments and expectations. The bounds for observations generated by different $z$ are examined separately. Let $B = \cbr{x_1, x_2, \cdots, x_n}$ and, for a specific $z_i \in \{0,1\}^K$, write $B_{z_i} := \cbr{x \in B: z = z_i}$ and $\hat{w}_{z_i} = \abr{B_{z_i}}/\abr{B}$ for $i \in \cbr{0,1 \cdots 2^K-1}$ and $z_i = \mathrm{binary} (i)$. Define the conditional moments t and their corresponding empirical moments as
  \begin{align}
    M_{r, z_i} &:= \Eb \sbr{x^{\otimes r} | z = z_i}, \,\,\, \hat{M}_{r, z_i} := \abr{B_{z_i}}^{-1} \sum \limits_{x \in B_{z_i}} x^{\otimes r}.
    \end{align}

%%%%%%%%%%%%%%%%%%%%%%%%%%%%%%%%%%%%%%%%%%%
\begin{lemma}{(Concentration of conditional empirical moments)}
  \label{lem:submoment}
  Given scalars $r, K, \delta, w, n,$ and $l,$ we define four functions
  \begin{align*}
  b_1(r, K, \delta, w, n, l) &= \sqrt{ \frac{r +2\sqrt{r \ln \rbr{l \cdot 2^K/\delta} }  
  + 2 \ln \rbr {l \cdot 2^K/\delta} }{w \cdot n} },\\
  b_2(r, K, \delta, w, n, l) &=  \sqrt{\frac{128 \rbr{r \ln 9 + \ln \rbr{l \cdot 2^{K+1}/\delta}}}{w \cdot n} }+ 
   \frac{4 \rbr{r \ln9+ \ln \rbr{l \cdot 2^{K+1}/\delta}}}{w \cdot n} ,\\
  b_3(r, K, \delta, w, n, l) &=   \sqrt{ \frac{108e^3 \lceil r \ln13+\ ln \rbr{l \cdot2^K/\delta} \rceil^3}{w \cdot n} } ,\\
  b_4(r, K, \delta, w, n, l) &=      \sqrt{\frac{8192 \rbr{r\ln 17 +\ln \rbr{l \cdot \cdot 2^{K+1} /\delta}}^2 }{(w \cdot n)^2}  +\frac{32\rbr{r\ln 17 +\ln \rbr{l\cdot 2^{K+1}/\delta}}^3 }{(w \cdot n)^3} }  ,
  \end{align*} 
  With probability greater than $1 - \delta$, pick any $\delta \in \rbr{0,1}$ and any random matrix $V \in \mathrm{R}^{d \times r}$ of rank $r$, the following guarantee holds\\
1. For the first-order moments, we have, for $ i \in \cbr{0,1 \cdots 2^K-1},$ 
  \begin{align*}
    \nbr{T \rbr{\hat{M}_{1,z_i} - M_{1,z_i}, V}}_2 \leq \sigma \nbr{V}_2 b_1(r, K, \delta, \hat{w}_{z_i}, n, 1).
  \end{align*}
2. For the second-order moments, we have, for $ i \in \cbr{0,1 \cdots 2^K-1},$
  \begin{align*}
    \nbr{ T\rbr{\hat{M}_{2,z_i} - M_{2,z_i},V,V}}_2  \leq
    & \sigma^2  \nbr{V}_2^2 b_2(r, K, \delta, \hat{w}_{z_i}, n, 2)\\
     &+  2 \sigma \nbr{V}_2 \nbr{V^\top M_{1,z_i}}_2 b_1(r, K, \delta, \hat{w}_{z_i}, n, 2).
  \end{align*}
3. For the third-order moments, we have, for $ i \in \cbr{0,1 \cdots 2^K-1},$
  \begin{align*}
    \nbr{ T\rbr{\hat{M}_{3,z_i} - M_{3,z_i},V,V,V}}_2 \leq &\sigma^3 \nbr{V}_2^3  b_3(r, K, \delta, \hat{w}_{z_i}, n, 3) \\
    &+ 3 \sigma^2  \nbr{V^\top M_{1,z_i}}_2  \nbr{V}_2^2 b_2 (r, K, \delta, \hat{w}_{z_i}, n, 3)\\
& + 3\sigma \nbr{V^\top M_{1,z_i}}_2^2 \nbr{V}_2 b_1 (r, K, \delta, \hat{w}_{z_i}, n, 3).
  \end{align*}
4. For the fourth-order moments, we have, for $ i \in \cbr{0,1 \cdots 2^K-1},$
  \begin{align*}
    \nbr{ T \rbr{\hat{M}_{4,z_i} - M_{4,z_i},V,V,V,V}}_2 \leq& \sigma^4 \nbr{V}_2^4  b_4(r, K, \delta, \hat{w}_{z_i}, n, 4)\\
    &+ 4  \sigma^3 \nbr{V^\top M_{1,z_i}}_2 \nbr{V}_2^3 b_3(r, K, \delta, \hat{w}_{z_i}, n, 3) \nonumber\\
    & + 6\sigma^2 \nbr{V^\top M_{1,z_i}}_2^2  \nbr{V}_2^2 b_2(r, K, \delta, \hat{w}_{z_i}, n, 4)\nonumber\\
    & + 4 \sigma \nbr{V^\top M_{1,z_i}}_2^3  \nbr{V}_2 b_1(r, K, \delta, \hat{w}_{z_i}, n, 4).
  \end{align*} 
\end{lemma}
The proof is provided in the Appendix. We finish the proof by adding the bounds for every term. By using inequalities for conditional moments, we get the bounds for moments by the following Lemma.

% The other norm can be bounded by using the bounds for low-order %conditional moments. We finish the proof by adding the bounds for %every term. By using inequalities for conditional moments, we get %the bounds for completed moments
%by the following Lemma.
%%%%%%%%%%%%%%%%%%%%%%%%%%%%%%%%%%%%%%%%%%%%%%%%%%%%
\begin{lemma}{( Lemma 6 in \cite{HsuKak12}; Concentration of empirical moments)}
 \label{lem:moments}
For a fixed matrix $V \in \mathbb{R}^{d \times r}$,
  \begin{align*} 
    &\nbr{T\rbr{\hat{M}_i - M_i,V, \cdots, V}}_2 \\
    &\leq (1+ 2^{K/2} \epsilon_w) \max_{z_j } \nbr{T\rbr{\hat{M}_{i,z_j} - M_{i,z_j},V, \cdots, V}}_2  + 2^{K/2}\max_{z_j} \nbr{T\rbr{M_{i,z_j}, V, \cdots, V}}_2\varepsilon_{w} \\
    &\,\,\, \forall i \in [4],\forall j \in \cbr{0,1 \cdots 2^K-1}
  \end{align*}
where $\varepsilon_{w} = \rbr{ \sum \limits_{z_j} \rbr{ \hat{w}_{z_j}-w_{z_j} }^2}^{\frac{1}{2}} \leq \frac{1+\sqrt{\ln \rbr{1/\delta}}}{\sqrt{n}}$.

\end{lemma}

\subsection{Concentration of Measure of the IBP}
Using the results of unbounded moments, we further get the bounds for the tensors based on the concentration of
moment in Lemma \ref{lem:sigbounds} and
\ref{lem:spectralbounds}. Bounds for reconstruction accuracy of our
algorithm are provided. The full proof is given in the Appendix.
\begin{theorem}{(Reconstruction Accuracy)}
  \label{th:reconstruction} Let $\varsigma_k \sbr{S_2}$ be the $k-th$ largest singular value of $S_2$. Define  $\pi_{min} = \argmax_{i \in [K]} \abr{\pi_i - 0.5}$, $\pi_{max} = \argmax_{i \in [K]}  \pi_i$ and $\tilde{\pi} = \prod_{\{i:\pi_i \leq 0.5\}} \pi_i \prod_{\{i:\pi_i > 0.5\}} (1-\pi_i)$.  Pick any $\delta, \epsilon \in (0,1)$. There exists a polynomial $\mathrm{poly}(\cdot)$ such that if sample size $m$ statisfies 
\begin{align*}
  m \geq \mathrm{poly}\Biggl({d, K, \frac{1}{\epsilon}, \log(1/\delta), \frac{1}{\tilde{\pi}},\frac{ \varsigma_1\sbr{S_2}}{ \varsigma_K\sbr{S_2}}, \frac{\sum \limits_{i=1}^K \nbr{\Phi_i}_2^2}{  \varsigma_K\sbr{S_2} }, \frac{\sigma^2}{  \varsigma_K\sbr{S_2}},\frac{1}{\sqrt{\pi_{min}-\pi_{min^2}}}, \frac{\pi_{max}}{\sqrt{\pi_{max}-\pi_{max}^2 }}}\Biggr)
\end{align*}
with probability greater than $1-\delta$, there is a permutation $\tau$ on $[K]$ such that the $\hat{A}$ returns by Algorithm \ref{alg:eca} satifies
$\displaystyle
    \nbr{\hat{\Phi}_{\tau(i)} - \Phi_i} \leq \rbr{\nbr{\Phi_i}_2 + \sqrt{ \varsigma_1 \sbr{S_2}}}\epsilon
    \text{ for all } i \in [K]$.
\end{theorem}
\subsection{Concentration of measure of the HDP}
\label{sec:conc}

We derive theoretical guarantees for the spectral inference algorithms
in an HDP. Specifically we provide guarantees for moments $M_r^{\ib}$,
tensors $S_r^{\ib}$, and latent factors $\Phi$. The technical
challenge relative to conventional models is that the data are not
drawn iid. Instead, they are drawn from a predefined hierarchy and
they are only exchangeable within the hierarchy. We address this by
introducing a more refined notion of effective sample size which
borrows from its counterpart in particle filtering \citep{DouFreGor01}.
We define $n_\ib$ to be the effective sample
size, obtained by hierarchical averaging over the HDP tree. This yields
\begin{align}
  \label{eq:nibr}
  n_\ib := \begin{cases}
    1 & \text{ for leaf nodes}
    \\
    |c(\ib)|^2 \sbr{\sum_{\jb \in c(\ib)} \frac{1}{n_\ib}}^{-1}
    & \text{ otherwise}
  \end{cases}
\end{align}
One may check that in the case where all leaves have an equal number
of samples and where each vertex in the tree has an equal number of
children, $n_{\ib}$ is the overall sample size. The intuition is that, for a balanced tree, every leaf nodes should contribute equally to the overall moments, which can be viewed as a two layer model with all the leaf nodes connected directly to the root node. Using similar approach for obtaining concentration measure for bounded moments, we extend the results that apply to moments for different hierarchical structure as in Theorem \ref{theorem:momentbounds}. 

\begin{theorem}
\label{theorem:momentbounds}
For any node $\ib$ in a $L$-layer HDP with $r$-th order moment
$M_r^{\ib}$ and for any $\delta \in (0,1)$ the following bound holds
for the tensorial reductions $M_r(u) :=T(M_r^{\ib}, u, \cdots, u)$ and
its empirical estimate $\hat{M}_r := T(\hat{M}_r^{\ib}, u, \cdots, u)$.
\begin{align}
  \nonumber
  \Pr\cbr{\sup_{u: \nbr{u} \leq 1} \abr{ M_r(u)-\hat{M}_r(u)} \leq
    \frac{2 + \sqrt{-\ln \delta}}{\sqrt{n_\ib}}} \geq 1 - \delta 
\end{align}
\end{theorem}
As indicated, $n_\ib$ plays the role of an effective sample
size. Note that an unbalanced tree has a smaller effective sample size
compared to a balanced one with same number of leaves.  

%with probability at most $\delta,$ and $\nbr{u} \leq 1$

\begin{theorem}
\label{theorem:tensorbounds}
Given a $L$-layer HDP with symmetric tensor $S_r^{\ib}$. Assume that
$\delta \in (0, 1)$ and denote the tensorial reductions as before
$S_r(u) := T(S_r^{\ib} u, \cdots, u)$ and $\hat{S}_r(u) :=
T(\hat{S}_r^{\ib}, u, \cdots, u)$. Then we have for $r \in \cbr{2, 3}$ and any node $\ib$ in the $L$-layer HDP, 
\begin{align}
  \label{eq:bounddeviation2}
  \Pr\cbr{\sup_{u: \nbr{u} \leq 1} \abr{ S_r -\hat{S_r}} \leq c_r n_{\ib}^{-\frac{1}{2}} \sbr{2+\sqrt{\ln
      (3/\delta)}}} \geq 1-\delta
\end{align}
where $c_2 = 3$ and $c_3$ is some constant.

\end{theorem}
This shows that not only the moments but also the symmetric tensors
directly related to the statistics $\Phi$ are directly available. 
The following theorem guarantees the accurate reconstruction of the
latent feature factors in $\Phi$. Again, a detailed proof is relegated
to Appendix \ref{proof:reconstruction}.

\begin{theorem}
\label{theorem:reconstruction}
Given a $L$-layer HDP with hyperparamter $\gamma_{\ib}$ at node
$\ib$. Let $\sigma_k \rbr{\Phi}$ denote the smallest non-zero singular
value of $\Phi$, and $\phi_i$ denote the $i$-th column of $\Phi$. 
For sufficiently large sample size, and for suitably chosen $\delta
\in (0,1)$, i.e.
$$3 n_{\ib}^{-\frac{1}{2}}
\sbr{2+\sqrt{\ln (3/\delta)}}\leq\frac{ C_3  \gamma_0 \min_{j}
  \pi_{0j} \sigma_k \rbr{\Phi}^2}{6}$$ 
we have
$\displaystyle \Pr\cbr{\nbr{\Phi_i -\hat{\Phi}_{\sigma\left(i\right)}}
  \leq   \epsilon} \geq 1-\delta$ where
\begin{align*}
\epsilon :=\frac{ck^3  \rbr{\min_{\ib} \gamma_{\ib}+2}^2}{ \delta
  \min_j \pi_{0j} \sigma_k \rbr{\Phi}^3 } n_\ib^{-\frac{1}{2}} \sbr{2+\sqrt{\ln(3/\delta)}}
\end{align*}
Here $\cbr{ \hat{\Phi}_1, \hat{\Phi}_2, \cdots, \hat{\Phi}_k}$ is the
set that Algorithm \ref{alg:hdpeca} returns, for some permutation $\sigma$
of $\cbr{1,2,\cdots, k}$, $i \in {1,2, \cdots, k}$, and some constant
$c$.
\end{theorem}
The theorem gives the guarantees on $l_2$ norm accuracy for the
reconstruction of latent factors. Note that all the bounds above are
functions of the effective sample sizes $n_\ib$. The latter are a
function of both the number of data and the structure of the
tree.

\section{Experiments}
\label{sec:experiment}
\subsection{IBP}

We evaluate the algorithm on a number of problems suitable for the two
models of \eq{eq:lingalaf} and \eq{eq:linzy}. The problems are largely
identical to those put forward in \cite{GriGha11} in order to keep our
results comparable with a more traditional inference approach. We demonstrate 
that our algorithm is faster, simpler, and achieves comparable or superior accuracy.
\paragraph{Synthetic data}

Our goal is to demonstrate the ability to recover latent structure of
generated data. Following \cite{GriGha11} we generate images via
linear noisy combinations of $6 \times 6$ templates.  That is, we use
the binary additive model of \eq{eq:lingalaf}.  The goal is to recover
both the above images and to assess their respective presence in
observed data. Using an additive noise variance of $\sigma^2 = 0.5$ we
are able to recover the original signal quite accurately (from left to
right: true signal, signal inferred from 100 samples, signal inferred
from 500 samples). Furthermore, as the second row indicates, our
algorithm also correctly infers the attributes present in the images.

\noindent
\includegraphics[width=0.29\columnwidth]{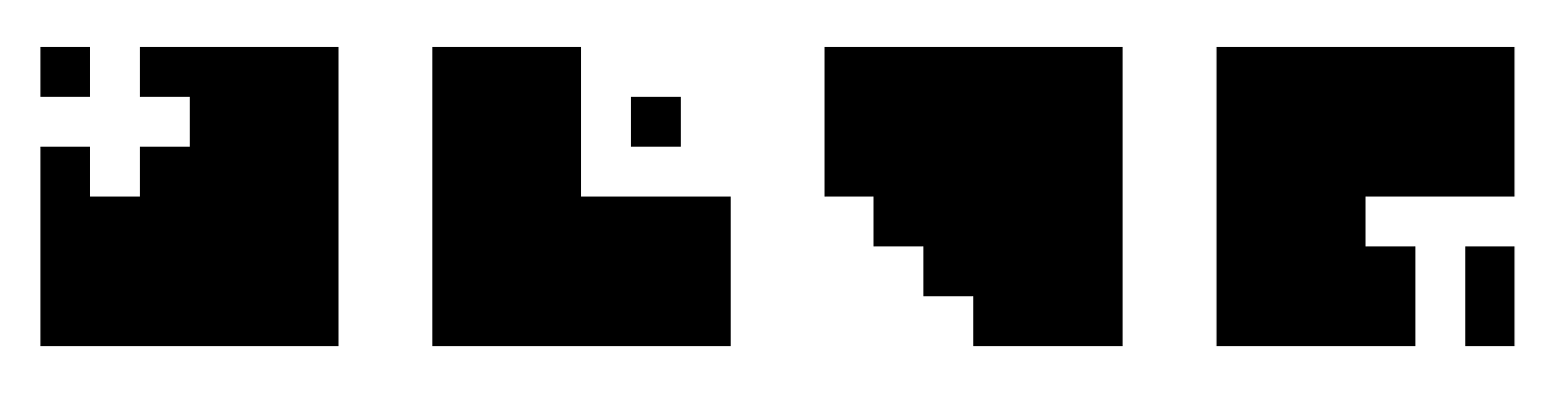}
\hspace{0.05\columnwidth}
\includegraphics[width=0.29\columnwidth]{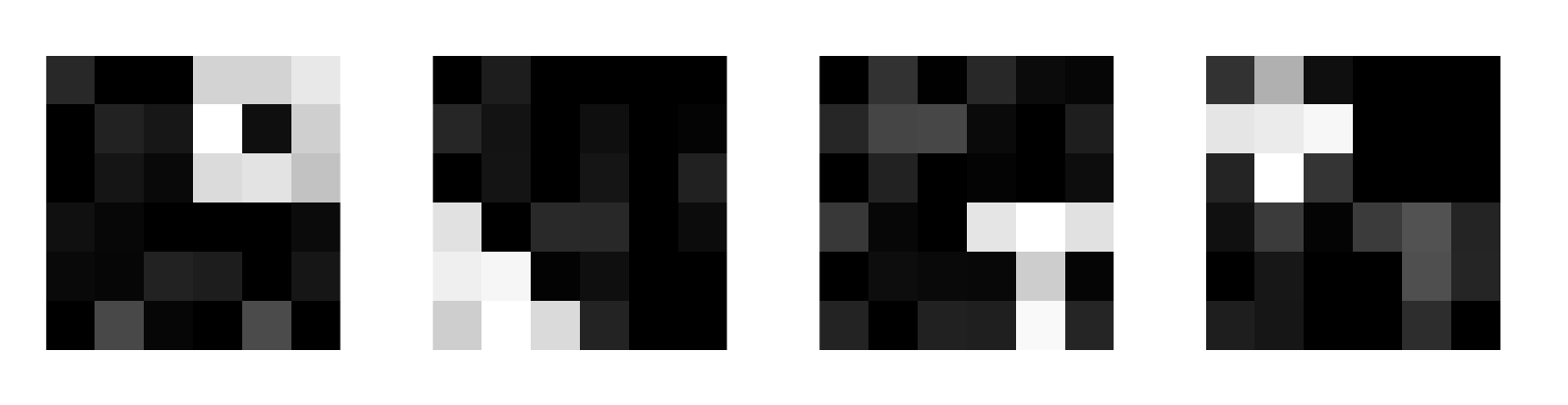} 
\hspace{0.05\columnwidth}
\includegraphics[width=0.29\columnwidth]{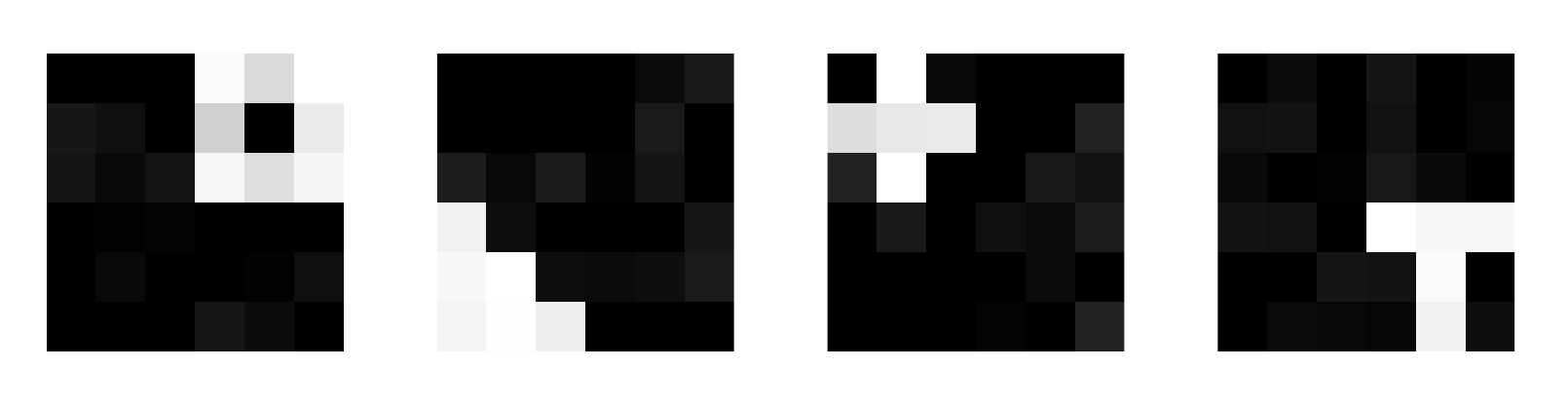} \\[-5mm]

\noindent
\hspace{0.34\columnwidth}
\includegraphics[width=0.29\columnwidth]{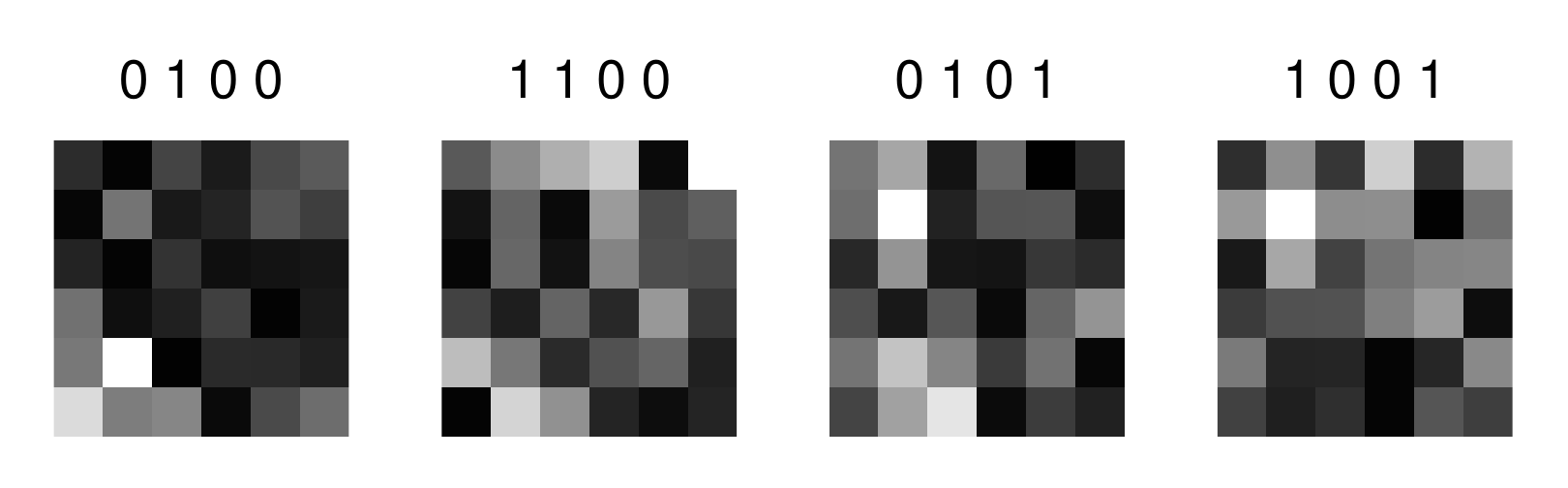}
\hspace{0.05\columnwidth}
\includegraphics[width=0.30\columnwidth]{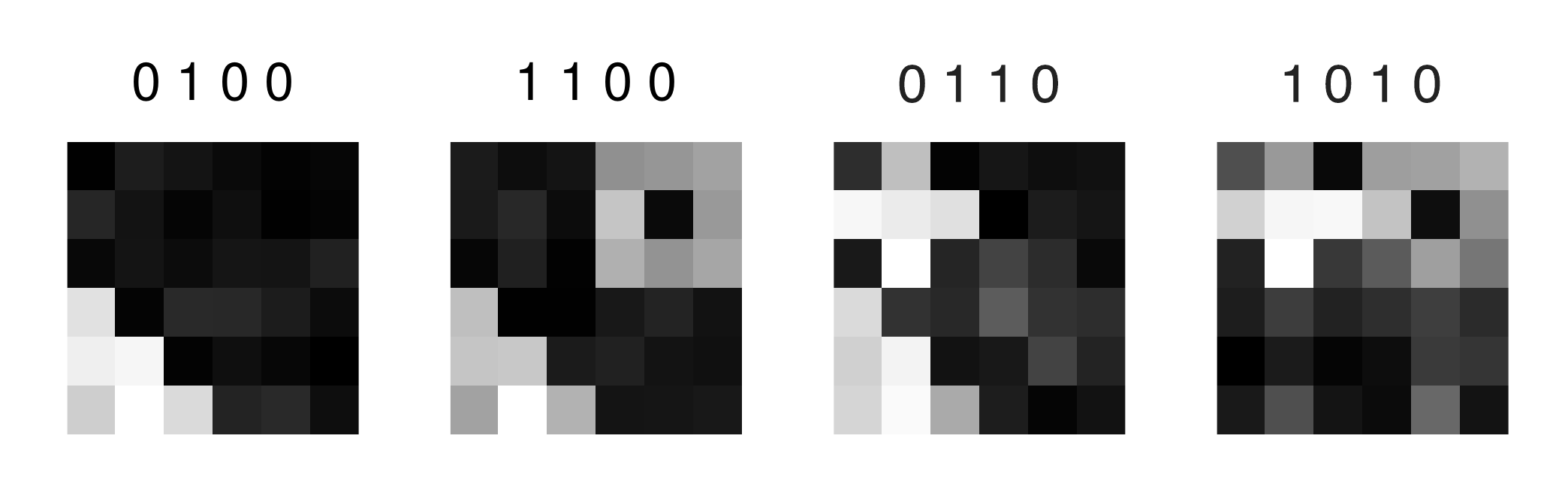}

For a more quantitative evaluation we compared our results to the
infinite variational algorithm of \cite{DosMilGaeTeh09}. The data is
generated using $\sigma \in \cbr{0.1, 0.2, 0.3, 0.4, 0.5}$ and with
sample size $n \in \cbr{100, 200, 300, 400, 500}$. Figure
\ref{fig:compare} shows that our algorithm is faster and comparatively accurate.

\begin{figure}[ht]
\begin{center}
\centerline{\includegraphics[width=1.2\columnwidth]{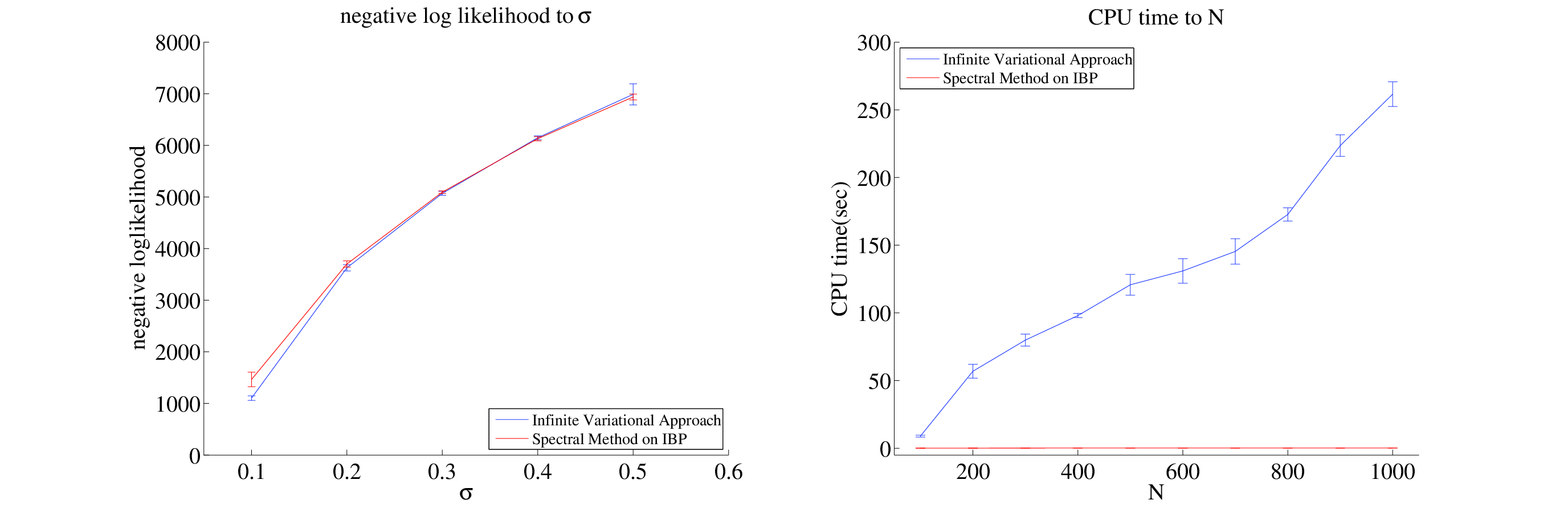}}
\caption{ Comparison to infinite variational approach. The first plot compares the test negative log likelihood training on $N=500$ samples with different $\sigma$. The second plot shows the CPU time to data size, $N$, between the two methods.
 }
\label{fig:compare}
\end{center}
\vskip -0.2in
\end{figure} 

\paragraph{Image Source Recovery}

We repeated the same test using $100$ photos from 
\cite{GriGha11}. We first reduce dimensionality on the data set by
representing the images with 100 principal components and apply our
algorithm on the 100-dimensional dataset (see Algorithm~\ref{alg:eca}
for details).  Figure~\ref{fig:griff} shows the result. We used $10$ initial
iterations $50$ random seeds and $30$ final
iterations $50$ in the Robust Power Tensor Method. The total runtime was $0.3$s on an intel Core i7 processor (3.2GHz).

\begin{figure}[tb]
\vspace{-5mm}
%  \begin{minipage}{0.50\textwidth}
\begin{center}
    \includegraphics[width=0.6\columnwidth]{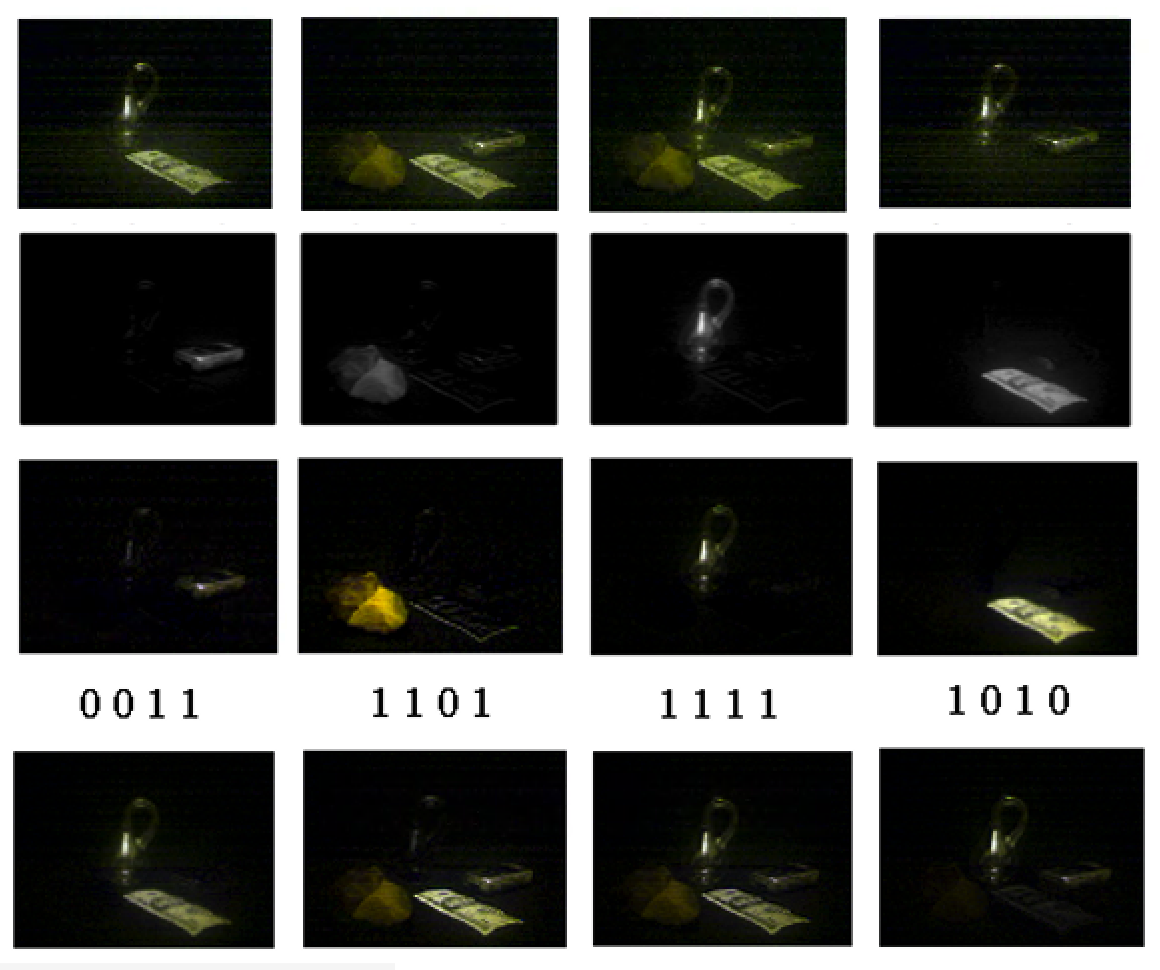}
    \end{center}
 % \end{minipage}
  \hfill
 % \begin{minipage}{0.48\textwidth}
    \vspace{-5mm}
    \caption{Results of modeling 100 images from \cite{GriGha11} of size $240 \times 320$
      by model \eq{eq:lingalaf}. Row 1: four sample
      images containing up to four objects (\$20 bill, Klein bottle,
      prehistoric handaxe, cellular phone). An object basically
      appears in the same location, but some small variation noise is
      generated because the items are put into scene by hand; 
      Row 2: Independent attributes, as determined by infinite
      variational inference of \cite{DosMilGaeTeh09} (note, the
      results in \cite{GriGha11} are black and white only);
      Row 3: Independent attributes, as determined by spectral IBP;
      Row 4: Reconstruction of the images via spectral IBP. The binary
      superscripts indicate the items identified in the image.
      \label{fig:griff}}
  %\end{minipage}
  \vspace{-3mm}
\end{figure} 

\paragraph{Gene Expression Data}

As a first sanity check of the feasibility of our model for
\eq{eq:linzy}, we generated synthetic data using $x \in \RR^7$ with
$k=4$ sources and $n=500$ samples, as shown in
Figure~\ref{fig:results_syn}.

\begin{figure}[tbh]
  \vspace{-1mm}
 % \begin{minipage}{0.50\textwidth}
 \begin{center}
    \includegraphics[width=0.4\columnwidth]{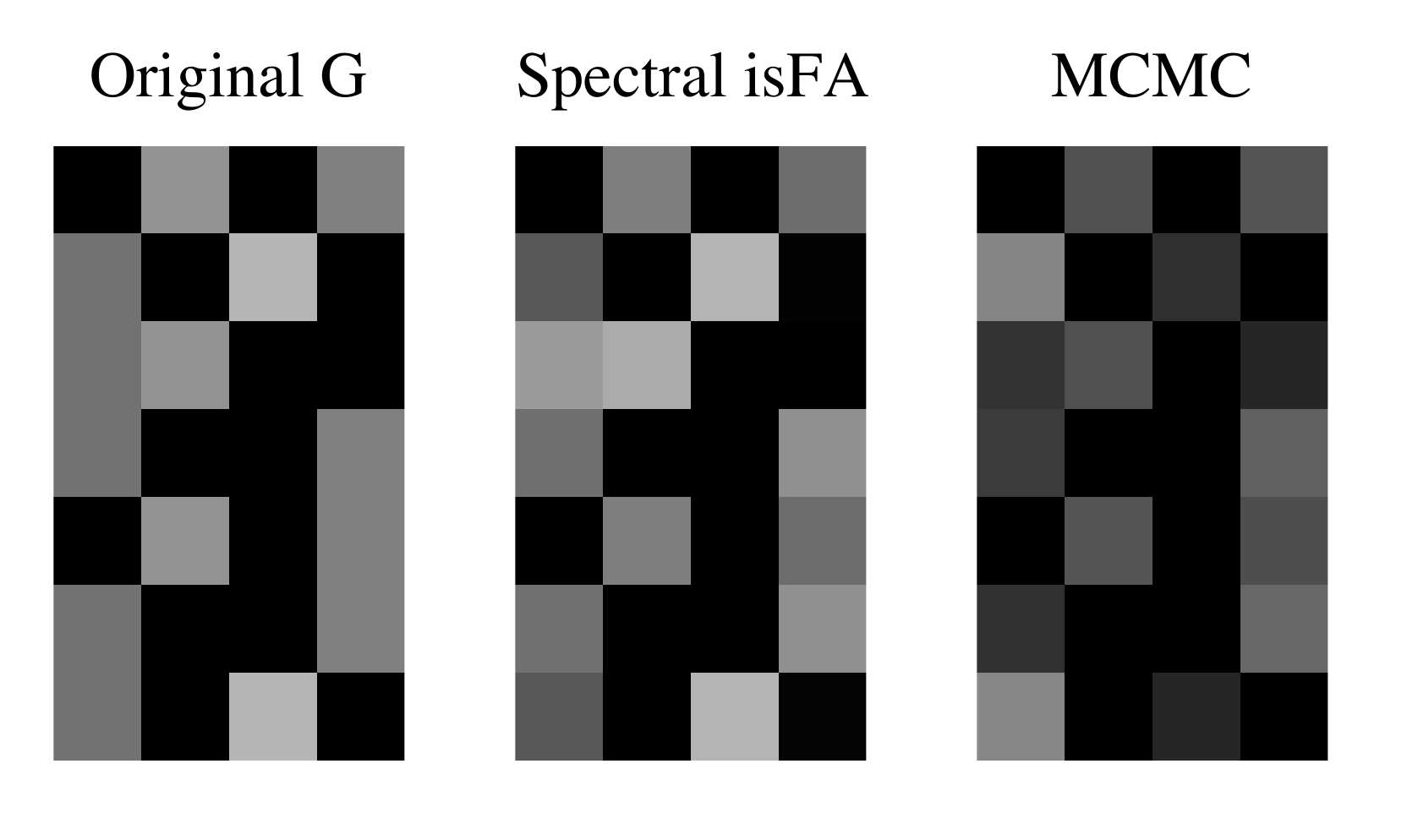}    
    \end{center}
  %\end{minipage}
  \hfill
  %\begin{minipage}{0.48\textwidth}
    \vspace{-10mm}
    \caption{Recovery of the source matrix $A$ in model \eq{eq:linzy}
      when comparing MCMC sampling and spectral methods. MCMC sampling
      required $1.72$ seconds and yielded a Frobenius distance $\nbr{A
        - A_{\mathrm{MCM}}}_F = 0.77$. Our spectral
      algorithm required $0.77$ seconds to achieve a distance $\nbr{A
        - A_{\mathrm{Spectral}}}_F = 0.31$.
      \label{fig:results_syn}
    }
 % \end{minipage}
% \end{figure}
% \begin{figure}

  \vspace{-2mm}
 % \begin{minipage}{0.7\textwidth}
 \begin{center}
    \includegraphics[width=0.7\columnwidth]{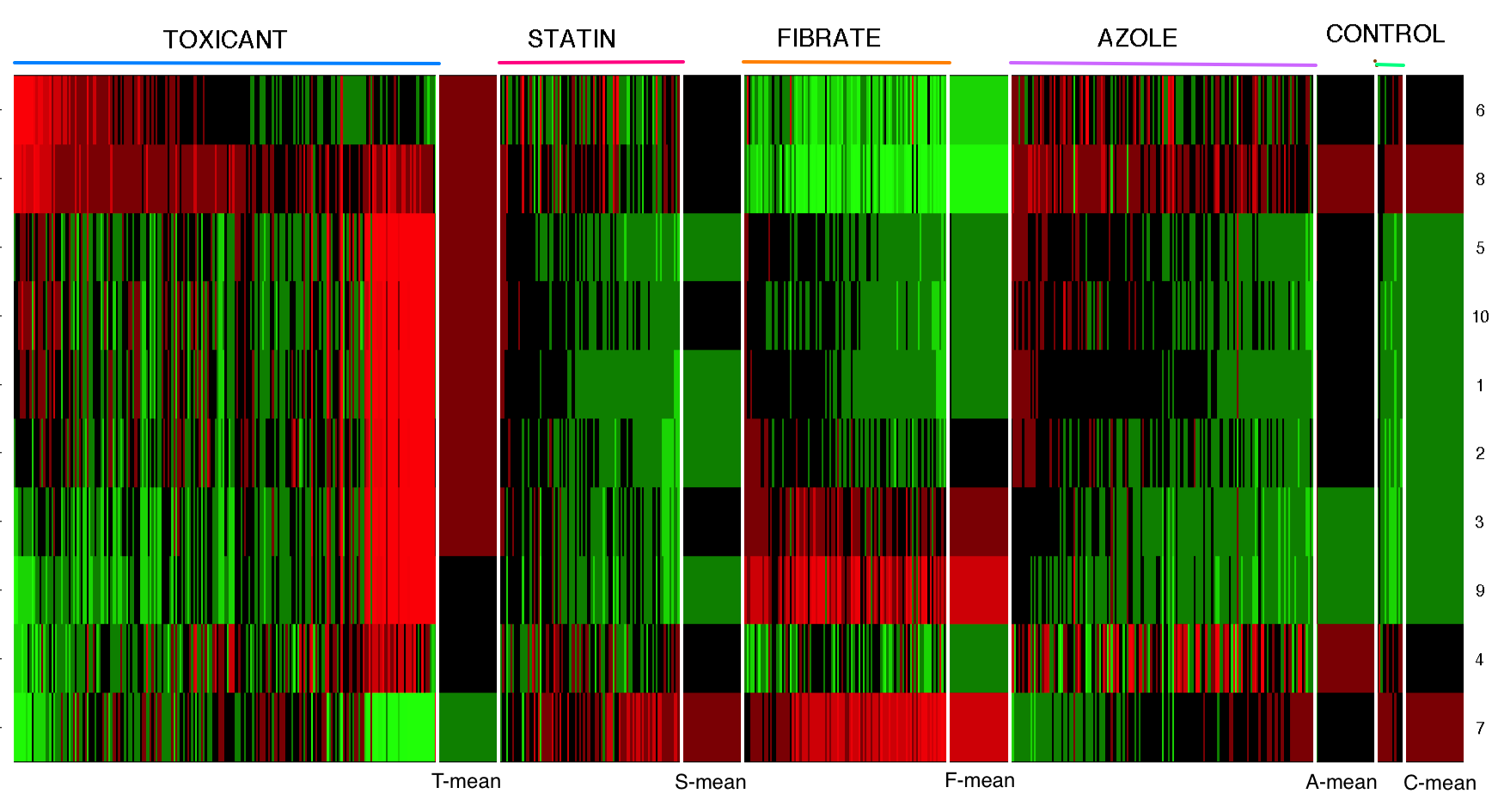}
    \end{center}
    \vspace{-10mm}
%  \end{minipage}
  \hfill
  %\begin{minipage}{0.23\textwidth}
    \caption{Gene signatures derived by the spectral IBP. They show that
  there are common hidden causes in the observed expression levels,
  thus offering a considerably simplified representation.
  \label{fig:ibp500}}
  \vspace{-5mm}
  %\end{minipage}
\end{figure} 

For a more realistic analysis we used a microarray dataset. The data
consisted of 587 mouse liver samples detecting 8565 gene probes,
available as dataset GSE2187 as part of NCBI's Gene Expression Omnibus
\url{www.ncbi.nlm.nih.gov/geo}. There are four main types of
treatments, including Toxicant, Statin, Fibrate and Azole. Figure
\ref{fig:ibp500} shows the inferred latent factors arising from
expression levels of samples on 10 derived gene signatures. According
to the result, the group of fibrate-induced samples and a small group
of toxicant-induced samples can be classified accurately by the
special patterns. Azole-induced samples have strong positive signals
on gene signatures 4 and 8, while statin-induced samples
have strong positive signals only on the 9 gene signatures.

% Acknowledgements should go at the end, before appendices and references

\subsection{HDP}
\label{sec:exp}

An attractive application of HDP is topic modelling in a corpus
where in documents are grouped naturally. We use the Enron email corpus \citep{KliYan04} and the Multi-Domain Sentiment Dataset \citep{BliDrePer07} to validate our algorithm. After the usual cleaning steps (stop word removal, numbers, infrequent
words), our training dataset for Enron consisted of $167,851$ emails sent with $10,000$ vocabulary size and average $91$ words in each email. Among these, $126,697$ emails are sent internally within Enron 
and $41154$ are from external sources. In order to show that the topics are able to cover topics from external and internal sources and are not biased toward the larger group, we have $537$ internal emails and $4,63$ external email in our test data. To evaluate the computational efficiency of the spectral algorithms
using fast count sketch tensor decomposition(FC) \citep{WanTunSmoAna15}, robust tensor method (RB) and alternating least square
(ALS), we compare the CPU time and per-word likelihood among these approaches.

% Apart from hierarchy in document organization, we also consider
% hierarchy in the emissions model.

\vspace{-3mm}

\begin{table}[tbh]\centering
\footnotesize
  \caption{Results on Enron
    dataset with different tree structures and different solvers: spectral method for the HDP using fast count sketch method (sketch length is set to $10$), alternating least square (ALS) and robust tensor power method (RB). 
    \label{tab:runtime} }
\begin{tabular}{@{}lccccc@{}}
Tree &K & &sHDP (FC) &sHDP (ALS) & sHDP (RB) \\ 
\hline
Enron 2-layer  &50		& like./time		&8.09/{\bf 67} 	&7.86/119&   7.86/2641  \\         
					&100		&like./time 		&8.16/{\bf 104} 	& 7.82/668&	7.82/5841 \\
Enron 3-layer &50		& like./time		&7.93/{\bf 68}		&  7.78/121&  {\bf 7.77}/2710 \\ 
		 		    &100		&like./time		&8.18/{\bf 101}	&7.69/852	&{\bf  7.68}/5782  \\
\end{tabular} 
\end{table}    
We further compare spectral method under balanced/unbalanced tree structure of data on Multi-Domain Sentiment Dataset. The dataset contains reviews from Amazon reviews that fall into four categories: books, DVD, electronics and kitchen. We generate $2$ training datasets where one has balanced number of reviews under each categories (1900 reviews for each category) and the other has highly unbalanced number of examples at the leaf node (1900/1500/700/300 reviews for the four categories), while the test dataset consisted of 100 reviews for each categories. The result in Table  \ref{tab:senti} show that spectral algorithm with multi-layers structure will perform even better than with flat model when the tree structure is unbalanced. 

\begin{table}[tbh]\centering
\footnotesize
  \caption{Results on Multi-Domain Sentimental dataset. Train data 1 is selected so that the there are balanced numbers of reviews under each category. Train data 2 is selected to have highly unbalanced child number at the leaf nodes.
    \label{tab:senti} }
\begin{tabular}{@{}lcccccc@{}}
Tree 						& train data 1	& K=50 	 	& K=100 	&train data 2 &K=50 	& k= 100 \\ 
\hline	
Sentiment 2-layer  	&  like./time		&{\bf 7.9}/38		& 7.99/151		&	like./time		&8.23/36	&  8.14/147\\         
Sentiment  3-layer  &like./time		&7.92/37		& {\bf 7.96}/150		&  like./time		&{\bf 8.17}/38&	{\bf 8.07}.148\\ 
\end{tabular} 
\end{table}

%The slight difference might be attributed to the absence of a prior on the word distributions. An algorithm exploiting priors on word distributions is left as a future work. Alternatively, using spectral algorithm as initialization of MCMC methods is another possibility to achieve the best of both worlds. 

%\paragraph{Comments on use of different structure and solvers}

The results of the experiments throw light on two key points. First, leveraging the
information in the form of hierarchical structure of documents, instead of
blinding grouping the documents into a single-layer model like LDA,
will result in better performance (i.e. higher log-likelihood) under
different settings. The tree structure is able to eliminate the
pernicious effects caused by unbalanced data. For example, a 2-layer model like LDA
considers every email to be equally important, and so for a topic
related to external events it will perform worse, as most of the
emails are exchanged within the company and they are unlikely to
possess topics related to the external emails. Second, although spectral method cannot obtain a solution that has higher performance in perplexity, it can be used as a tool for picking up a nice initial point.

\section{Conclusion}
The IBP and the HDP mixture models are useful and most popular nonparametric 
Bayesian tool. Unfortunately the computational complexity of the inference
algorithms is high. Thus we propose a spectral algorithm to alleviate
the pain. We first derived the low-order moments for both mixture
model, and then described the algorithm to recover the latent factors
of interest. Concentration of measure for this method is also
provided. We demonstrate the advantages of utilizing structure information.
High performance numerical linear algebra and more advanced
optimization algorithms will improve matters further.

\acks{We would like to acknowledge support for this project
from Oracle and Microsoft. }

% Note: in this sample, the section number is hard-coded in. Following
% proper LaTeX conventions, it should properly be coded as a reference:

%In this appendix we prove the following theorem from
%Section~\ref{sec:textree-generalization}:

\vskip 0.2in
%\bibliography{sample}
%\newpage
%\bibliographystyle{plain}
%\bibliography{../../../../bibfile/bibfile}
%\bibliography{../../../bibfile}
\bibliography{main}

\begin{thebibliography}{36}
\providecommand{\natexlab}[1]{#1}
\providecommand{\url}[1]{\texttt{#1}}
\expandafter\ifx\csname urlstyle\endcsname\relax
  \providecommand{\doi}[1]{doi: #1}\else
  \providecommand{\doi}{doi: \begingroup \urlstyle{rm}\Url}\fi

\bibitem[Alon et~al.(1999)Alon, Matias, and Szegedy]{AloMatSze99}
N.~Alon, Y.~Matias, and M.~Szegedy.
\newblock The space complexity of approximating the frequency moments.
\newblock \emph{Journal of Computers and System Sciences}, 58\penalty0
  (1):\penalty0 137--147, 1999.
\newblock URL \url{http://dx.doi.org/10.1006/jcss.1997.1545}.

\bibitem[Altun and Smola(2006)]{AltSmo06}
Y.~Altun and A.~J. Smola.
\newblock Unifying divergence minimization and statistical inference via convex
  duality.
\newblock In H.U. Simon and G.~Lugosi, editors, \emph{Proc.\ Annual Conf.\
  Computational Learning Theory}, LNCS, pages 139--153. Springer, 2006.

\bibitem[Aly et~al.(2012)Aly, Hatch, Josifovski, and Narayanan]{AlyHatJosNar12}
M.~Aly, A.~Hatch, V.~Josifovski, and V.K. Narayanan.
\newblock Web-scale user modeling for targeting.
\newblock In \emph{Conference on World Wide Web}, pages 3--12. ACM, 2012.

\bibitem[Anandkumar et~al.(2011)Anandkumar, Chaudhuri, Hsu, Kakade, Song, and
  Zhang]{AnaChaHsuKakSonZha2011}
A.~Anandkumar, K.~Chaudhuri, D.~Hsu, S.~Kakade, L.~Song, and T.~Zhang.
\newblock Spectral methods for learning multivariate latent tree structure.
\newblock In \emph{Neural Information Processing Systems}, 2011.

\bibitem[Anandkumar et~al.(2012{\natexlab{a}})Anandkumar, Foster, Hsu, Kakade,
  and Liu]{AnaFosHsuKakLiu12}
A.~Anandkumar, D.~P. Foster, D.~Hsu, S.~M. Kakade, and Y.-K. Liu.
\newblock Two svds suffice: Spectral decompositions for probabilistic topic
  modeling and latent dirichlet allocation.
\newblock \emph{CoRR}, abs/1204.6703, 2012{\natexlab{a}}.

\bibitem[Anandkumar et~al.(2012{\natexlab{b}})Anandkumar, Ge, Hsu, Kakade, and
  Telgarsky]{AnaGeHsuKakTel12}
A.~Anandkumar, R.~Ge, D.~Hsu, S.~M. Kakade, and M.~Telgarsky.
\newblock Tensor decompositions for learning latent variable models.
\newblock \emph{arXiv preprint arXiv:1210.7559}, 2012{\natexlab{b}}.

\bibitem[Anandkumar et~al.(2014)Anandkumar, Ge, Hsu, and Kakade]{AnaGeHsuKak13}
A.~Anandkumar, R.~Ge, D.~Hsu, and S.~M. Kakade.
\newblock A tensor approach to learning mixed membership community models.
\newblock \emph{J. Mach. Learn. Res.}, 15\penalty0 (1):\penalty0 2239--2312,
  January 2014.
\newblock ISSN 1532-4435.
\newblock URL \url{http://dl.acm.org/citation.cfm?id=2627435.2670323}.

\bibitem[Antoniak(1974)]{Antoniak74}
C.~Antoniak.
\newblock Mixtures of {D}irichlet processes with applications to {B}ayesian
  nonparametric problems.
\newblock \emph{Annals of Statistics}, 2:\penalty0 1152--1174, 1974.

\bibitem[Blei and Jordan(2005)]{BleJor05}
D.~Blei and M.~Jordan.
\newblock Variational inference for dirichlet process mixtures.
\newblock In \emph{Bayesian Analysis}, volume~1, pages 121--144, 2005.

\bibitem[Blei et~al.(2002)Blei, Ng, and Jordan]{BleNgJor02}
D.~Blei, A.~Ng, and M.~Jordan.
\newblock Latent dirichlet allocation.
\newblock In T.~G. Dietterich, S.~Becker, and Z.~Ghahramani, editors,
  \emph{Advances in Neural Information Processing Systems 14}, Cambridge, MA,
  2002. {MIT} Press.

\bibitem[Blitzer et~al.(2007)Blitzer, Dredze, and Pereira]{BliDrePer07}
J.~Blitzer, M.~Dredze, and F.~Pereira.
\newblock Biographies, bollywood, boom-boxes and blenders: Domain adaptation
  for sentiment classification.
\newblock In \emph{Association for Computational Linguistics}, Prague, Czech
  Republic, 2007.

\bibitem[Boots et~al.(2013)Boots, Gretton, and Gordon]{BooGreGeo13}
B.~Boots, A.~Gretton, and G.~J. Gordon.
\newblock Hilbert space embeddings of predictive state representations.
\newblock 2013.

\bibitem[Cardoso(1998)]{Cardoso98}
J.-F. Cardoso.
\newblock Blind signal separation: statistical principles.
\newblock \emph{Proceedings of the IEEE}, 90\penalty0 (8):\penalty0 2009--2026,
  1998.

\bibitem[Dempster et~al.(1977)Dempster, Laird, and Rubin]{DemLaiRub77}
A.~P. Dempster, N.~M. Laird, and D.~B. Rubin.
\newblock Maximum likelihood from incomplete data via the {EM} algorithm.
\newblock \emph{Journal of the Royal Statistical Society B}, 39\penalty0
  (1):\penalty0 1--22, 1977.

\bibitem[Doshi et~al.(2009)Doshi, Miller, Gael, and Teh]{DosMilGaeTeh09}
F.~Doshi, K.~Miller, J.~Van Gael, and Y.~W. Teh.
\newblock Variational inference for the indian buffet process.
\newblock \emph{Journal of Machine Learning Research - Proceedings Track},
  5:\penalty0 137--144, 2009.
\newblock URL \url{http://www.jmlr.org/proceedings/papers/v5/doshi09a.html}.

\bibitem[Doucet et~al.(2001)Doucet, de~Freitas, and Gordon]{DouFreGor01}
A.~Doucet, N.~de~Freitas, and N.~Gordon.
\newblock \emph{Sequential Monte Carlo Methods in Practice}.
\newblock Springer-Verlag, 2001.

\bibitem[Ferguson(1973)]{Ferguson73}
T.~S. Ferguson.
\newblock A bayesian analysis of some nonparametric problems.
\newblock \emph{The Annals of Statistics}, 1\penalty0 (2):\penalty0 209--230,
  1973.

\bibitem[Fox et~al.(2010)Fox, Sudderth, Jordan, and Willsky]{FoxSudJorWil10}
E.~B. Fox, E.~B. Sudderth, M.~I. Jordan, and A.~S. Willsky.
\newblock Sharing features among dynamical systems with beta processes.
\newblock \emph{nips}, 22, 2010.

\bibitem[Gretton et~al.(2012)Gretton, Borgwardt, Rasch, Schoelkopf, and
  Smola]{GreBorRasSchetal12}
A.~Gretton, K.~Borgwardt, M.~Rasch, B.~Schoelkopf, and A.~Smola.
\newblock A kernel two-sample test.
\newblock \emph{JMLR}, 13:\penalty0 723--773, 2012.

\bibitem[Griffiths and Ghahramani(2006)]{GriGha06}
T.~Griffiths and Z.~Ghahramani.
\newblock Infinite latent feature models and the indian buffet process.
\newblock pages 475--482, 2006.

\bibitem[Griffiths and Ghahramani(2011)]{GriGha11}
T.~Griffiths and Z.~Ghahramani.
\newblock The indian buffet process: An introduction and review.
\newblock 12:\penalty0 1185--1224, 2011.

\bibitem[Griffiths and Steyvers(2004)]{GriSte04}
T.L. Griffiths and M.~Steyvers.
\newblock Finding scientific topics.
\newblock \emph{Proceedings of the National Academy of Sciences}, 101:\penalty0
  5228--5235, 2004.

\bibitem[Halko et~al.(2009)Halko, Martinsson, and Tropp]{HalMarTro09}
N.~Halko, P.G. Martinsson, and J.~A. Tropp.
\newblock Finding structure with randomness: Stochastic algorithms for
  constructing approximate matrix decompositions, 2009.
\newblock URL \url{http://arxiv.org/abs/0909.4061}.
\newblock oai:arXiv.org:0909.4061.

\bibitem[Hsu and Kakade(2012)]{HsuKak12}
D.~Hsu and S.M. Kakade.
\newblock Learning mixtures of spherical gaussians: moment methods and spectral
  decompositions, 2012.
\newblock URL \url{arXiv:1206.5766}.

\bibitem[Hsu et~al.(2009)Hsu, Kakade, and Zhang]{HsuKakZha09}
D.~Hsu, S.~Kakade, and T.~Zhang.
\newblock A spectral algorithm for learning hidden markov models.
\newblock 2009.

\bibitem[Klimt and Yang(2004)]{KliYan04}
B.~Klimt and Y.~Yang.
\newblock The enron corpus: A new dataset for email classification research.
\newblock In \emph{ECML}, pages 217--226, 2004.

\bibitem[Knowles and Ghahramani(2007)]{KnoGha07}
D.~Knowles and Z.~Ghahramani.
\newblock Infinite sparse factor analysis and infinite independent components
  analysis.
\newblock In \emph{International Conference on Independent Component Analysis
  and Signal Separation}, 2007.

\bibitem[McDiarmid(1989)]{McDiarmid89}
C.~McDiarmid.
\newblock On the method of bounded differences.
\newblock In \emph{Survey in Combinatorics}, pages 148--188. Cambridge
  University Press, 1989.

\bibitem[Miller et~al.(2009)Miller, Griffiths, and Jordan]{MilGriJor09}
K.T. Miller, T.L. Griffiths, and M.I. Jordan.
\newblock Latent feature models for link prediction.
\newblock In \emph{Snowbird}, page 2 pages, 2009.

\bibitem[Neal(1998)]{Neal98b}
R.~Neal.
\newblock Markov chain sampling methods for dirichlet process mixture models.
\newblock Technical Report 9815, University of Toronto, 1998.

\bibitem[Pisier(1989)]{Pisier89}
G.~Pisier.
\newblock \emph{The Volume of Convex Bodies and {B}anach Space Geometry}.
\newblock Cambridge University Press, Cambridge, 1989.

\bibitem[Quattoni et~al.(2004)Quattoni, Collins, and Darrell]{QuaColDar04}
A.~Quattoni, M.~Collins, and T.~Darrell.
\newblock Conditional random fields for object recognition.
\newblock In \emph{Neural Information Processing Systems}, pages 1097--1104.
  2004.

\bibitem[Song et~al.(2010)Song, Boots, Siddiqi, Gordon, and
  Smola]{SonBooSidGorSmo10}
L.~Song, B.~Boots, S.~Siddiqi, G.~Gordon, and A.~J. Smola.
\newblock Hilbert space embeddings of hidden markov models.
\newblock In \emph{International Conference on Machine Learning}, 2010.

\bibitem[Teh et~al.(2006)Teh, Jordan, Beal, and Blei]{TehJorBeaBle06}
Y.~Teh, M.~Jordan, M.~Beal, and D.~Blei.
\newblock Hierarchical dirichlet processes.
\newblock \emph{Journal of the American Statistical Association}, 101\penalty0
  (576):\penalty0 1566--1581, 2006.

\bibitem[Wang et~al.(2015)Wang, Tung, Smola, and Anandkumar]{WanTunSmoAna15}
Y.~Wang, H.-Y. Tung, A.~J. Smola, and A.~Anandkumar.
\newblock Fast and guaranteed tensor decomposition via sketching.
\newblock \emph{NIPS}, 2015.

\bibitem[Wood et~al.(2006)Wood, Grifﬁths, and Ghahramani]{WooGriGha06}
F.~Wood, T.~L. Grifﬁths, and Z.~Ghahramani.
\newblock A non-parametric bayesian method for inferring hidden causes.
\newblock \emph{uai}, 2006.

\end{thebibliography}

\newpage
\appendix
%\section*{Appendix A.}
\label{app:theorem}
\section{Proof of Symmetric Tensors}

{\bf Symmetric Tensors for the HDP}\\
We begin our analysis by deriving the moments for a three layer
HDP. This allows us to provide detail without being hampered by
cumbersome notation. After that, we analyze the general expansion. 

\subsection{Three Layers}
The three layer HDP is structurally similar to LDA. Its tensors 
are derived in \cite{AnaFosHsuKakLiu12}. We begin by considering 
a three model to gain intuition of how to obtain the
general format of the tensors. The goal is to reconstruct the latent
factors in $\Phi$. In the case of topic modeling, the $j$-th column
denotes the word distribution of the $j$-th topic.  

\begin{lemma}[Symmetric tensors of $3$-layer HDP]
 \label{lem:hdp3moments}
 Given a $3$-layer HDP with hyperparameters $\gamma_1$ and $\gamma_2$
 at layers $1$ and $2$ respectively, the symmetric tensors are given by
 \begin{align}
   S_1 &:= M_1 = T(\pi_0, \Phi), \nonumber\\
	S_2  &:= M_2 - C_2 S_1 S_1^T = M_2 - \frac{\gamma_2 \gamma_1}{\rbr{\gamma_2+1}\rbr{\gamma_1+1}} S_1 S_1^T =  T\rbr{  C_3 \cdot \mathrm{diag}\rbr{ \pi_0} , \Phi, \Phi},\nonumber \\
	S_3  &:=  M_3 - C_4  \cdot S_1 \otimes S_1 \otimes S_1 -C_5  \cdot \symm_3\sbr{ S_2 \otimes M_1}  = T\rbr{    C_6 \cdot \mathrm{diag}\rbr{\pi_0} ,\Phi,\Phi,\Phi} \nonumber,\label{eq:S}
\end{align}
where 
\begin{align*}
C_3&=\frac{\gamma_2 + \gamma_1 + 1}{ \rbr{\gamma_1+1} \rbr{\gamma_2+1}},\\
C_4 &= \frac{\gamma_1^2\gamma_2^2}{\rbr{\gamma_1+1}\rbr{\gamma_1+2}\rbr{\gamma_2+1}\rbr{\gamma_2+2}},\\
C_5 &=  \frac{\gamma_2\gamma_1\rbr{\gamma_1+\gamma_2+2}}{\rbr{\gamma_1+2}\rbr{\gamma_2+2}\rbr{\gamma_1+\gamma_2+1}},\\
C_6 &= \frac{ 6\gamma_1 +6 \gamma_2 +2\gamma_1^2+ 2\gamma_2^2 +3 \gamma_2 \gamma_1 + 4 }{ \rbr{\gamma_1+1}\rbr{\gamma_1+2}\rbr{\gamma_2+1}\rbr{\gamma_2+2}}.
\end{align*}
\end{lemma}

\begin{proof}
  By definition of the Dirichlet Process, the means match
  that of the reference measure. This means that we can integrate over
  the hierarchy
  \begin{align}
\Eb[x] &= \Eb_{G_{p(\ib)} | G_{p(p(\ib))}, \gamma(p(\ib))}  \sbr{\Eb_{x|G_{p(\ib)}, \gamma(\ib)}\sbr{\Phi \pi_{\ib}}} \\
         &= \Eb_{G_{p(\ib)} |G_{p(p(\ib))}, \gamma(p(\ib))} \sbr{\Phi \pi_{p(i)}}\\
         &= \Phi \pi_{p(p(\ib))} = \Phi \pi_0 
\end{align}
Then deriving the first-order tensor is straightforward, 
  \begin{align}
     S_1 &:= M_1 = \Eb_x \sbr{x_1} = \Phi \pi_0  = T(\pi_0, \Phi)
  \end{align}
Similarly, to derive the second order tensor, we first need the
following terms: for $i \neq j$ we have
\begin{align}
\Eb_{G_{p(\ib)} | G_{p(p(\ib))}, \gamma(p(\ib))}  \sbr{\Eb_{x|G_{p(\ib)}, \gamma(\ib)}\sbr{\Phi_i \pi_{\ib i} \pi_{\ib j}^T \Phi_j^T}} &= \Eb_{G_{p(\ib)} | G_{p(p(\ib))}, \gamma(p(\ib))}   \sbr{\Phi_i\frac{\gamma_2 \pi_{p(\ib) i} \pi_{p(\ib) j}^T }{\gamma_2 + 1} \Phi_j^T}\\
&= T\rbr{\frac{\gamma_2}{\gamma_2 + 1}\frac{\gamma_1 \pi_{0i}
  \pi_{0j}}{\gamma_1 + 1}, \Phi_i, \Phi_j}
  \end{align}
  Likewise, when the indices match, we obtain
  \begin{align}
&\Eb_{G_{p(\ib)} | G_{p(p(\ib))}, \gamma(p(\ib))}  \sbr{\Eb_{x|G_{p(\ib)}, \gamma(\ib)}\sbr{\Phi_i \pi_{\ib i} \pi_{\ib i}^T  \Phi_i^T }} \\
&= \Eb_{G_{p(\ib)} | G_{p(p(\ib))}, \gamma(p(\ib))}  \sbr{\Phi_i\frac{(\gamma_2  \pi_{p(\ib) i}+ 1)  \pi_{p(\ib) i}}{(\gamma_2 + 1)} \Phi_i^T}\\
&=  T\rbr{\frac{\gamma_2}{\rbr{\gamma_2+1}} \frac{ \gamma_1 \pi_{0i}^2
    + \pi_{0i}}{\rbr{\gamma_1 +1}} +\frac{1}{\rbr{\gamma_2 + 1}}
  \pi_{0i}, \Phi_i, \Phi_i}
\end{align}
Then the moment $M_2$ could be written as
    \begin{align}
     M_2 =& \Eb_x \sbr{x_1\otimes x_2} \\
     =& \Eb\sbr{\Eb_x\sbr{x_1 \otimes x_2 | G_{\ib}}}  \\
            =& \Eb \sbr{\Eb_x \sbr{x_1| G_{\ib}} \otimes  \Eb_x \sbr{x_2| G_{\ib}} } \\
           =&  \Phi \Eb \sbr{ \Eb \sbr{\pi_{\ib}\pi_{\ib}^T}}\Phi ^T \\
         =&\frac{\gamma_2  \gamma_1}{\rbr{\gamma_2+1}\rbr{\gamma_1+1}} S_1 \otimes S_1 
         +  T\rbr{ \frac{\gamma_2 + \gamma_1 + 1}{ \rbr{\gamma_1+1}
             \rbr{\gamma_2+1} } \cdot \mathrm{diag}\rbr{\pi_0} , \Phi ,\Phi}
    \end{align}
   The second-order symmetric tensor $S_2$ could then be obtained by defining
  \begin{align}
    S_2 &:= M_2 - \frac{\gamma_2  \gamma_1}{\rbr{\gamma_2+1}\rbr{\gamma_1+1}} S_1 \otimes S_1= T\rbr{ \frac{\gamma_2 + \gamma_1 + 1}{\gamma_1
               \rbr{\gamma_1+1} \rbr{\gamma_2+1} } \cdot
             \mathrm{diag}\rbr{G_0} , \Phi,\Phi }.
  \end{align}
Before deriving the third-order tensor, we derive $\Eb[(\Phi_i \pi_{\ib i}) \otimes (\Phi_j \pi_{\ib j}) \otimes (\Phi_j \pi_{\ib j})]$ for the following three cases. First, for $i = j = k,$ we have:
%deal with the third-order moment, first, we focus on the elements on the diagonal of the cube $\Eb\sbr{h \otimes h \otimes h},$ we have:
\begin{equation}
\begin{aligned}
	&\Eb_{G_{p(\ib)} | G_{p(p(\ib))}, \gamma(p(\ib))}  \sbr{\Eb_{x|G_{p(\ib)}, \gamma(\ib)} \sbr{(\Phi_i \pi_{\ib i})^{\otimes 3} }} \\
	&= \Eb \sbr{T(\frac{\pi_{p(\ib) i} \rbr{\gamma_2\pi_{p(\ib) i}+1} \rbr{\gamma_2\pi_{p(\ib) i} +2}}{ \rbr{\gamma_2+1}\rbr{\gamma_2+2}}, \Phi_i, \Phi_i, \Phi_i)}\\
	&= T \bigg( \frac{\gamma_2^2}{\rbr{\gamma_2+1}\rbr{\gamma_2+2}}  \frac{\pi_{0i} \rbr{\gamma_1\pi_{0i}+1} \rbr{\gamma_1 \pi_{0i}+2}}{ \rbr{\gamma_1+1}\rbr{\gamma_1+2}} + \frac{3\gamma_2}{\rbr{\gamma_2+1}\rbr{\gamma_2+2}}  \frac{\pi_{0i}\rbr{\gamma_1 \pi_{0i}+1} }{ \rbr{\gamma_1+1}} \\\
	&\,\,\,\,\,\,\,\,\,\,\,\, \,\,\,\,\,\,\,\,\,+ \frac{2 \pi_{0i}}{\rbr{\gamma_2+1}\rbr{\gamma_2+2}}, \Phi_i,\Phi_i,\Phi_i \bigg)
\end{aligned}
\end{equation}
Second, for $i = j \neq k,$ we have:
\begin{equation}
\begin{aligned}
\label{eq:moment_2}
	&\Eb_{G_{p(\ib)} | G_{p(p(\ib))}, \gamma(p(\ib))} \sbr{\Eb_{x|G_{p(\ib)}, \gamma(\ib)}\sbr{(\Phi_i \pi_{\ib i})^{\otimes 2}\otimes (\Phi_k \pi_{\ib k})}}\\
	& = \Eb \sbr{T\bigg( \frac{ \pi_{p(\ib)i}\rbr{\gamma_2 \pi_{p(\ib)i}+1} \gamma_2 \pi_{p(\ib)k} }{ \rbr{\gamma_2+1}\rbr{\gamma_2+2}}, \Phi_i, \Phi_i, \Phi_k \bigg)}\\
	 &= T \bigg( \frac{\gamma_2^2}{\rbr{\gamma_2+1}\rbr{\gamma_2+2}}  \frac{ \pi_{0i} \rbr{\gamma_1 \pi_{0i}+1}\gamma_1 \pi_{0k}}{ \rbr{\gamma_1+1}\rbr{\gamma_1+2}}+ \frac{\gamma_2}{\rbr{\gamma_2+1}\rbr{\gamma_2+2}}  \frac{\gamma_1 \pi_{0i}\pi_{0k}}{ \rbr{\gamma_1+1}} , \Phi_i,\Phi_i,\Phi_k \bigg)
\end{aligned}
\end{equation}
Third, for $i \neq j \neq k,$ we have:
%for elements $\Eb\sbr{h \otimes h \otimes h}_{ijk},$ where $i \neq j \neq k,$ by the same method, we get:

\begin{align}
\label{eq:moment_3}
	&\Eb \sbr{\Eb_{x|G_{p(\ib)}, \gamma(\ib)}\sbr{(\Phi_i \pi_{\ib i})\otimes (\Phi_j \pi_{\ib j})\otimes (\Phi_k \pi_{\ib k})}} \\
	&= \Eb \sbr{T \bigg( \frac{ \gamma_2^2   \pi_{p(\ib)i}  \pi_{p(\ib)j}  \pi_{p(\ib)k} }{ \rbr{\gamma_2+1}\rbr{\gamma_2+2}} , \Phi_i,\Phi_j,\Phi_k \bigg) }\\
	 &= T
        \bigg(\frac{\gamma_2^2}{\rbr{\gamma_2+1}\rbr{\gamma_2+2}}
        \frac{\gamma_1^2  \pi_{0i}  \pi_{0j}  \pi_{0k} }{
          \rbr{\gamma_1+1}\rbr{\gamma_1+2}}, \Phi_i,\Phi_j,\Phi_k \bigg).
\end{align}
Defining
\begin{align}
S_3 &:=  M_3 - C_4  \cdot S_1 \otimes S_1 \otimes S_1 -C_5  \cdot \symm_3\sbr{ S_2 \otimes M_1}\\
        &= T\rbr{    C_6 \cdot \mathrm{diag}\rbr{\pi_0} ,\Phi,\Phi,\Phi}
\end{align}
we solve $C_4, C_5, C_6$ as follows.

Note that for $i\neq j \neq k,$ $[S_3]_{ijk} = 0$ and $\sbr{\symm_3[S_2\otimes M_1]}_{ijk} = 0$. Thus
\begin{align}
	C_4 &= \frac{[M_3]_{ijk}}{[S_{1}]_i [S_{1}]_j [S_{1}]_k}\\
	&= \frac{\gamma_2^2}{\rbr{\gamma_2+1}\rbr{\gamma_2+2}}  \frac{\gamma_1^2 \pi_{0i}  \pi_{0j} \pi_{0k} }{ \rbr{\gamma_1+1}\rbr{\gamma_1+2}} \cdot \frac{1}{\pi_{0i}  \pi_{0j} \pi_{0k}}\\
	&=  \frac{\gamma_2^2 \gamma_1^2}{\rbr{\gamma_2+1}\rbr{\gamma_2+2}\rbr{\gamma_1+1}\rbr{\gamma_1+2}}\label{eq:c2}
\end{align}
Similarly, for $i = j \neq k,$ $[S_3]_{ijk} = 0$. Thus
\begin{align}
C_5 &= \frac{[M_3]_{iik} - C_4 [S_1]_i [S_1]_i [S_1]_k }{[S_2]_{ii}[{M_1}]_k}\\
=&\frac{\gamma_2^2 \gamma_1 \pi_{0i}  \pi_{0k} + \gamma_2 \rbr{\gamma_1+2} \gamma_1 \pi_{0i}  \pi_{0k}}{\rbr{\gamma_2+1}\rbr{\gamma_2+2} \rbr{\gamma_1+1}\rbr{\gamma_1+2}} \cdot \frac{ \rbr{\gamma_1+1}\rbr{\gamma_2+1}}{\pi_{0i}  \pi_{0k} \rbr{\gamma_2 +\gamma_1 + 1}}\\
=&\frac{\gamma_2\gamma_1\rbr{\gamma_1+\gamma_2+2}}{\rbr{\gamma_1+2}\rbr{\gamma_2+2}\rbr{\gamma_1+\gamma_2+1}}\label{eq:c3}
\end{align}
Finally, 
\begin{align}
C_6 &= \frac{[M_3]_{iii} - C_4[S_1]_i [S_1]_i [S_1]_i - 3 C_5[S_2]_{ii} [M_1]_i}{\gamma_i}\nonumber\\
&= \frac{2\gamma_2^2 \pi_{0i}  + 3\gamma_2\rbr{\gamma_1+2}\pi_{0i} + 2\rbr{\gamma_1+1}\rbr{\gamma_1+2}\pi_{0i}}{\rbr{\gamma_2+1}\rbr{\gamma_2+2} \rbr{\gamma_1+1} \rbr{\gamma_1 +2}\gamma_1 \pi_{0i}}\nonumber \\
&=  \frac{ 6\gamma_1 +6 \gamma_2 +2\gamma_1^2+ 2\gamma_2^2 +3 \gamma_2 \gamma_1 + 4 }{\gamma_1 \rbr{\gamma_1+1}\rbr{\gamma_1+2}\rbr{\gamma_2+1}\rbr{\gamma_2+2}}\label{eq:c4}
\end{align}
\end{proof}

\subsection{Multiple Layer HDP}

\begin{lemma}[Symmetric Tensors of HDP]
\label{lem:hdpgmoments}
   For an $L$-level HDP, with hyperparameters, $\gamma_1, \gamma_2$,
   \ldots $\gamma_{L-1}$ we have
   \begin{align*}
    	 S_1 &:= M_1 = T(\pi_0, \Phi),\\
 	S_2 &:= M_2 - C_2 \cdot S_1 S_1^T = T( C_3 \cdot \mathrm{diag}\rbr{\pi_0}, \Phi,\Phi),\\
 	S_3 & :=  M_3 - C_4  \cdot S_1 \otimes S_1 \otimes S_1 -C_5  \cdot \symm_3\sbr{ S_2 \otimes M_1}= T(  C_6 \cdot  \mathrm{diag}\rbr{\pi_0}, \Phi,\Phi,\Phi) ,
   \end{align*}
   The key difference is that here the coefficients $C_i$ are
   recursively defined since we need to take expectations all the way
   up to the root node. This yields
 \begin{align*}
 C_2 =& \frac{\prod \limits_{i=1}^{L-1} \gamma_i }{ \prod
   \limits_{i=1}^{L-1} \rbr{\gamma_i +1} }, \,\,\,\,\,\,\,C_3 =   \sum \limits_{i=1}^{L-1} \frac{\prod \limits_{j=1}^{i-1} \gamma_{L-j} }{ \prod \limits_{j=1}^i \rbr{\gamma_{L-j} +1}  }; \nonumber \\
 C_4 =& \frac{\prod \limits_{i=1}^{L-1} {\gamma_i}^2  }{ \prod \limits_{i=1}^{L-1} \rbr{\rbr{\gamma_i +1}\rbr{\gamma_i +2}} }, \,\,\,\,\,\,C_5 = \sum \limits_{i=1}^{L-1}  \frac{\prod \limits_{j=1}^{i-1} {\gamma_{L-j}} \prod \limits_{j=1}^{L-1} \gamma_j }{ \prod \limits_{j=1}^i \rbr{\gamma_{L-j} +2}  \prod \limits_{j=1}^{L-1} \rbr{\gamma_j +1} }/C_3 \nonumber\\
 C_6 =& 3 \sum \limits_{i=1}^{L-2} \sum \limits_{j = i+1}^{L} \frac{ \prod \limits_{k=1}^{i-1}\gamma_{L-k} \prod \limits_{k=1}^{j-1} \gamma_{L-k} }{ \prod \limits_{k=1}^i \rbr{\gamma_{L-k} +2}  \prod \limits_{k=1}^j \rbr{\gamma_{L-k}+1} } + 2 \sum \limits_{i=1}^{L-1} \frac{\prod \limits_{j=1}^{i-1} {\gamma_{L-j}}^2  }{ \prod \limits_{j=1}^i \rbr{ \rbr{\gamma_{L-j} +1} \rbr{\gamma_{L-j} +2}}}\nonumber
 \end{align*}
 \end{lemma}

\section{Proof of reconstruction formula for the spectral HDP}

\subsection{ Proof of reconstruction formula}

\label{sec:eigenvalue}
 For simplicity in the proof, in Equation \eq{eq:gaussian-tensor-2} \eq{eq:gaussian-tensor-3} \eq{eq:gaussian-tensor-4}, we define the diagonal coefficients for $S_i$ to be $C_i \in \mathbb{R}^K$, i.e., $C_2 = \pi -\pi^2$, $C_3 = \pi -3\pi^2+ 2\pi^3$ and $C_4 = \pi - 7\pi^2 +12 \pi^3 -6 \pi^4$, so that 
  \begin{align*}
    S_2 &= T(\mathrm{diag} \rbr{C_2}, \Phi, \Phi), \; \;\; 
    S_3 =T(\mathrm{diag} \rbr{C_3}, \Phi, \Phi, \Phi), \;\;\; 
    S_4 =T(\mathrm{diag} \rbr{C_4}, \Phi, \Phi, \Phi, \Phi).
  \end{align*}
Following step 6 in Algorithm \ref{alg:eca}, we obtain whitening matrix $W$ by doing svd on $S_2$. Suppose the svd of matrix $ T( \mathrm{diag} \rbr{\sqrt{C_2}},\Phi) = U\Sigma^{1/2}V^\top$,  we have $S_2 = U\Sigma^{1/2}V^\top V\Sigma^{1/2}U^\top = USU^T$ and $W = U\Sigma^{-1/2}$. Using the fact that $$S_3 = T\rbr{\mathrm{diag} \rbr{C_3C_2^{-3/2}}, \mathrm{diag} \rbr{ \sqrt{C_2}}\Phi,\mathrm{diag}  \rbr{ \sqrt{C_2}} \Phi,\mathrm{diag}  \rbr{ \sqrt{C_2}}\Phi},$$ we have
  \begin{align}
    W_3 &=T \rbr{S_3, W,W,W} \nonumber \\
    &= T\rbr{\mathrm{diag} \rbr{C_3C_2^{-3/2}}, \Sigma^{-1/2}U^\top( U\Sigma^{1/2}V^\top),\Sigma^{-1/2}U^\top( U\Sigma^{1/2}V^\top),\Sigma^{-1/2}U^\top( U\Sigma^{1/2}V^\top)} \nonumber\\
    & = T\rbr{\mathrm{diag} \rbr{C_3C_2^{-3/2}}, V^\top,V^\top,V^\top}.
  \end{align}
The diagonalized tensor $W_3$, with some permutation $\tau$ on $[K]$ and $s_i \in \cbr{\pm 1}$, has eigenvalues and eigenvectors:
  \begin{align}
    \label{eq:lambda-s3}
    \lambda_i &= s_i C_{3,i}C_{2,i}^{-3/2}, \,\,\, v_i = s_i (V^\top), e_{\tau \rbr{i}},
  \end{align}
 where $C_{i,j}$ representing the $j$-th element in $C_i$.
After obtaining $v_i$, we multiply $v_i$ by $(W^{\dag})^\top$ to rotate it back to $\Phi_i$ as describing in step 15 in Algorithm \ref{alg:eca}, where $W^{\dag} = (W^\top W)^{-1}W^\top = \Sigma^{1/2}U^\top$, we get 
  \begin{align}
    \label{eq:W}
    (W^{\dag})^\top v_i &= s_i U\Sigma^{1/2}V^\top e_{\tau \rbr{i}} = s_i T( \mathrm{diag} \rbr{\sqrt{C_2}}, \Phi) e_{\tau \rbr{i}} = s_i \sqrt{C_{2,i}} \Phi_{\tau \rbr{i}},  
  \end{align}
which yields $\Phi_{\tau \rbr{i}} =  \frac{(W^{\dag})^\top v_i }{s_i \sqrt{C_{2,i}}}$. With the fact that $s_i = C_{3,i}C_{2,i}^{-3/2} \lambda_i^{-1}$ from Equation \eq{eq:lambda-s3}, we have
  \begin{align}
    \Phi_{\tau \rbr{i}} &=  \frac{\lambda_i }{\rbr{ C_{3,i}C_{2,i}^{-1}}}  (W^{\dag})^\top v_i = C_{2,i}^{-1/2} (W^{\dag})^\top v_i .
  \end{align}
Plug in the definition of $C_2$, we get the scale factor for $i \in [K_1]$. For $\Phi_i$ which are recovered by conducting tensor decomposition on $W_4$, we first examine
  \begin{align}
    W_4 = T \rbr{S_4, W,W,W,W} &=T\rbr{\mathrm{diag} \rbr{C_{4,i} C_{2,i}^{-2}}, V^\top ,V^\top ,V^\top ,V^\top},
  \end{align}
and obtain
  \begin{align}
    \label{eq:lambda-s4}
    \lambda_i & =  C_{4,i} C_{2,i}^{-2},\,\,\, v_i = s_i (V^\top ) e_{\tau \rbr{i}}.
  \end{align}
By using the fact that $s_i = s_i C_{4,i} C_{2,i}^{-2} {\lambda_i}^{-1}$ and Equation \eq{eq:W}, we have
  \begin{align}
    \Phi_{\tau \rbr{i}}& =  \frac{(W^{\dag})^\top v_i }{s_i \sqrt{C_{2,i}}}= \frac{s_i }{\rbr{C_{4,i} C_{2,i}^{-3/2}} \lambda_i^{-1} } = s_i (W^{\dag})^\top  v_i =s_i  C_{2,i}^{-1/2} (W^{\dag})^\top v_i , \,\,\,\, \forall i \in \sbr{K_1+1, \cdots, K}.
  \end{align}
   Note that the value of $\pi_i$ used to construct $C_j$ can be recovered by Equation \eq{eq:lambda-s3} and \eq{eq:lambda-s4} after obtaining $\lambda_i$.

\subsection{ Proof of reconstruction formula for the spectral HDP}

\label{proof:reconstruction2}
Using the results in the previous section, the corresponding eigenvalues $\lambda_i$ and eigenvectors $v_i$, with some permutation ${\bf \pi}$, are
\begin{align}
\lambda_i &= s_i \frac{C_6}{C_3\sqrt{C_3}}\mathrm{diag}\rbr{\frac{1}{\sqrt{\gamma_{\pi_i}}}}, \,\,v_i = s_iV^Te_{\pi_i}
\label{eq:lambda_i}
\end{align}
Therefore
\begin{align}
W^+ &= \rbr{W^TW}^{-1}W^T = SU^T\\
\rbr{W^+}^Tv_i &= s_i(US)V^Te_{\pi_i} = s_i \sqrt{C_3}\sqrt{\gamma_{\pi_i}}\Phi_{\pi_i}\\
\Phi_{\pi_i} &= \frac{\rbr{W^+}^Tv_i}{s_i \sqrt{C_3}\sqrt{\gamma_{\pi_i}}}
\end{align}
Rearranging Equation \ref{eq:lambda_i}, we have
\begin{align}
s_i &= \frac{C_6}{C_3\sqrt{C_3}\sqrt{\gamma_{\pi_i}}\lambda_i};\,\,\,\,\,\Phi_{\pi_i} = \frac{\lambda_i (W^+)^T v_i}{C_6/C_3}
\end{align}

\section{ Proof of Lemma \ref{lem:submoment}}

\label{proof:reconstruction3}
\begin{proof}
Here we only show the derivation of the fourth-order conditional moments. The other inequalities can be found in \cite{HsuKak12}. Under the stated model, the fourth-order conditional moment can be expended as
  \begin{align*}
    M_{4,z_i} &= M_{1,z_i} \otimes M_{1,z_i} \otimes M_{1,z_i} \otimes M_{1,z_i} + \sigma^2 \symm_6\sbr{M_{1,z_i} \otimes M_{1,z_i} \otimes \one} + \Eb \sbr{\epsilon  \otimes \epsilon  \otimes \epsilon \otimes \epsilon },
  \end{align*}
 which yields
  \begin{align}
  \label{eq:M_4expension}
    &\hat{M}_{4,z_i}-  M_{4,z_i} \nonumber\\
    &= \frac{1}{\hat{w}_{z_i}n} \bigg(\sum \limits_{x \in B_{z_i}} \rbr{x_j-M_{1,z_i}}\otimes\rbr{x_j-M_{1,z_i}}\otimes\rbr{x_j-M_{1,z_i}}\otimes\rbr{x_j-M_{1,z_i}} - \sigma^4 \symm_3 \sbr{\one \otimes \one}\nonumber\\
    &\,\,\,\,\,\,+  \sum \limits_{x \in B_{z_i}} \rbr{ \symm_4 \sbr{M_{1,z_i} \otimes \rbr{x_j - M_{1,z_i}}\otimes \rbr{x_j - M_{1,z_i}}\otimes \rbr{x_j - M_{1,z_i}} } }\nonumber\\
    &\,\,\,\,\,\,+\sum \limits_{x \in B_{z_i}} \rbr{ \symm_6 \sbr{M_{1,z_i} \otimes M_{1,z_i} \otimes \rbr{ \rbr{x_j - M_{1,z_i}}\otimes \rbr{x_j - M_{1,z_i}} - \sigma^2 \one }}} \nonumber\\
    &\,\,\,\,\,\,+\sum \limits_{x \in B_{z_i}}\symm_4 \sbr{ M_{4,z_i}\otimes  M_{4,z_i}\otimes M_{4,z_i}\otimes\rbr{x_j-M_{1,z_i}}  } 
\bigg)
  \end{align}
Suppose $V = V_1 \Sigma V_2^\top $ is the SVD of $V$, where $V_1 \in \mathrm{R}^{d \times r}$ consists of orthonormal columns. With $y_{j,z_i} = V_1^\top \rbr{x_j-M_{1,z_i}}$, applying triangle inequalities to Equation \eqref{eq:M_4expension} yields
  \begin{align*}
    &\nbr{T\rbr{\hat{M}_{4,z_i}-  M_{4,z_i},V,V,V,V}}_2 \\
    &\leq \nbr{V}_2^4 \nbr{  \frac{1}{\hat{w}_{z_i}n} \sum \limits_{x \in B_{z_i}} \rbr{ y_{j,z_i}\otimes y_{j,z_i}\otimes y_{j,z_i}\otimes y_{j,z_i}  -  \sigma^4 \symm_3 \sbr{\one \otimes \one}  }  }_2\\
    &\,\,\,\,\,\,+ 4 \nbr{V^\top M_{1,z_i} }_2 \nbr{  \frac{1}{\hat{w}_{z_i}n} \sum \limits_{x \in B_{z_i}}  y_{j,z_i}\otimes y_{j,z_i}\otimes y_{j,z_i} }_2\\
    &\,\,\,\,\,\,+ 6 \nbr{V^\top M_{1,z_i} }_2^2  \nbr{  \frac{1}{\hat{w}_{z_i}n} \sum \limits_{x \in B_{z_i}} \rbr{ y_{j,z_i}\otimes y_{j,z_i} -  \sigma^2\one}  }_2+ 4 \nbr{R^\top M_{1,z_i} }_2^3 \nbr{  \frac{1}{\hat{w}_{z_i}n} \sum \limits_{x \in B_{z_i}}  y_{j,z_i} }_2,
  \end{align*}
 By using Lemma \ref{lem:random4}, we bound the first term by
  \begin{align}
    \Pr \Bigg[ &\nbr{ \frac{1}{\hat{w}_{z_i}n} \sum \limits_{x \in B_{z_i}} \rbr{ y_{j,z_i}\otimes y_{j,z_i}\otimes y_{j,z_i}\otimes y_{j,z_i}  -  \Eb \sbr{\epsilon \otimes \epsilon \otimes \epsilon \otimes \epsilon }  } }_2\nonumber \\
    & \,\,\,\,\,\,\,\,\,\,\,\,\,\, \,\,\,\,\,\,\,\,\,\,\,\,\,\,\,\,\,\, > \sigma^4 \sqrt{\frac{8192 \rbr{r\ln 17 +\ln \rbr{2K/\delta}}^2 }{n^2}  +\frac{32\rbr{r\ln 17 +\ln \rbr{2K/\delta}}^3 }{n^3} }  \Bigg] \leq \delta.
  \end{align}
\end{proof}

%%%%%%%%%%%%Theorem3%%%%%%%%%%%%%%%%%%%%%%
\section{Concentration of Measure for the spectral IBP}
In this section, we provides bounds for tensors of linear gaussian latent feature model. The concentration behavior is more complicated than that of the bounded moments in Theorem \ref{th:momentbounds} due to the additive Gaussian noise. Here we restate the model as
  \begin{align}
    x= \Phi z + \epsilon 
  \end{align}
where $x \in \mathbb{R}^d$ is the observation, $z \in \{0,1\}^K$ is a binary vector indicating the possession of certain latent vector and $\epsilon$ is gaussian noise drawn from $N(0, \sigma^2 \one)$. Using the results for bounded moments, we derive the concentration measure for tensors.

%%%%%%%%%%%%Theorem3%%%%%%%%%%%%%%%%%%%%%%

%%%%%%%%%%%%%%%%%%%%%%%%%%%%%%%%%%%%%%%%%%%%%%%%%%%%
\subsection{Estimation of $\sigma$, $S_2$, $S_3$, $S_4$}
 Note that we have $\sigma^2  = \lambda_{min}\sbr{{M_2} - {M_1} \otimes {M_1}} =
\varsigma_K{\sbr{{M_2} - {M_1} \otimes {M_1}}}$, where $\varsigma_t \sbr{M}$ denoting the $t-th$ singular value of matrix $M$ which is defined in Theorem \ref{th:reconstruction}. Here we define $\hat{S}_{2,K}$ to be the best rank $k$ approximation of $ \hat{M_2} - \hat{M_1} \otimes \hat{M_1} - \hat{\sigma}^2\one$, which is the truncated matrix $S_2$ in Algorithm \ref{alg:eca}. $\hat{S}_i$ denotes the empirical tensors derived from summation of $\hat{M}_i$ and $\hat{\sigma}$. $S_i$ denotes the theoretical values.
\begin{lemma}{(Accuracy of $\sigma^2$, $\sigma^4$ and $M_{2,K}$)}
  \label{lem:sigbounds}
  \begin{align}
    &\abr{\hat{\sigma}^2 - \sigma^2 }  \leq \nbr{\hat{M}_2-M_2}_2 +   \nbr{\hat{M}_1-M_1}_2^2+2\nbr{\hat{M}_1-M_1}_2\nbr{M_1}_2\\
    &\abr{\hat{\sigma}^4 -  \sigma^4}  \leq \abr{ \hat{\sigma}^2 - \sigma^2 }^2 + 2 \sigma^2 \abr{\hat{\sigma}^2 - \sigma^2}\\
    &\nbr{\hat{S}_{2,k}-S_2}_2 \leq 4 \rbr{\nbr{\hat{M}_2 - M_2}_2 + \nbr{\hat{M_1}-M_1}_2^2 + 2\nbr{M_1}_2 \nbr{\hat{M_1}-M_1}_2}
  \end{align}
\end{lemma}

\begin{proof}

For the first order tensor, the inequality holds trivially due to the
guarantees for $\nbr{\hat{M}_1 - M_1}_2$. Next we bound the difference in
variance estimates. Using the fact that differences in the $k$-th
eigenvalues are bounded by the matrix norm of the difference we have
that
  \begin{align}
    \abr{\hat{\sigma}^2 - \sigma^2 } 
    & = \abr{\varsigma_k \sbr{\hatm_2-\hatm_1 \otimes \hatm_1 } -
    \varsigma_k{\sbr{{M_2} - {M_1} \otimes {M_1}}} }  \\
    & \leq \nbr{\sbr{\hatm_2-\hatm_1 \otimes \hatm_1 } - {\sbr{{M_2} -
        {M_1} \otimes {M_1}}} }_2 \\
    & \leq \nbr{\hat{M}_2-M_2}_2 +   \nbr{\hat{M}_1-M_1}_2^2+2\nbr{\hat{M}_1-M_1}_2\nbr{M_1}_2.
  \end{align}
The second inequality follows the Weyl's inequality and the last inequality is obtained by the triangle inequality.
For estimation of $\hat{\sigma^4}$,
  \begin{align}
    \abr{\hat{\sigma}^4 - \sigma^4} \leq \abr{ ( \hat{\sigma}^2 - \sigma^2)^2 + 2\sigma^2(\hat{\sigma}^2 - \sigma^2) } \leq  \abr{ \hat{\sigma}^2 - \sigma^2 }^2 + 2 \sigma^2 \abr{\hat{\sigma}^2 - \sigma^2}.
  \end{align}
For the last claimed inequality, with Weyl's inequality, 
\begin{align}
 &\nbr{\hat{S}_{2,k} - \rbr{\hat{M_2} - \hat{M_1} \otimes \hat{M_1} - \hat{\sigma}^2 \one}}_2 \leq \varsigma_{k+1}\sbr{\hat{M_2} - \hat{M_1} \otimes \hat{M_1} - \hat{\sigma}^2 \one } \\ &=\nbr{  \varsigma_{k+1}\sbr{\hat{M_2} - \hat{M_1} \otimes \hat{M_1} - \hat{\sigma}^2 \one } -  \varsigma_{k+1}\sbr{M_2 - M_1 \otimes M_1 - \sigma^2 \one } }_2 \\&\leq \nbr{\hat{M_2} - \hat{M_1} \otimes \hat{M_1} - \hat{\sigma}^2 \one - \rbr{M_2 - M_1 \otimes M_1 - \sigma^2 \one}}_2
\end{align}
, which yields
  \begin{align}
    \nbr{\hat{S}_{2,k}-S_2}_2& \leq \nbr{\hat{S}_{2,k} - \rbr{\hat{M_2} - \hat{M_1} \otimes \hat{M_1} - \hat{\sigma}^2 \one}}_2\nonumber\\
    &\,\,\,\,\,\, + \nbr{\hat{M_2} - \hat{M_1} \otimes \hat{M_1} - \hat{\sigma}^2 \one - \rbr{M_2 - M_1 \otimes M_1 - \sigma^2 \one}}_2\\
    &\leq 2 \rbr{\nbr{\hat{M}_2 - M_2}_2 + \nbr{\hat{M_1}-M_1}_2^2 + 2\nbr{M_1}_2 \nbr{\hat{M_1}-M_1}_2 + \abr{\hat{\sigma}^2  -\sigma^2} }\\
    &\leq 4 \rbr{\nbr{\hat{M}_2 - M_2}_2 + \nbr{\hat{M_1}-M_1}_2^2 + 2\nbr{M_1}_2 \nbr{\hat{M_1}-M_1}_2}.
  \end{align} 
\end{proof}
 The inequalities for $\sigma$ can be used for bounding the tensors $S_2$, $S_3$ and $S_4$, which will be shown next, and the inequality for $S_{2,k}$ will be used in bounding whitened tensor in Section \ref{app:whiten}.

%%%%%%%%%%%%%%%%%%%%%%%%%%%%%%%%%%%%%%%%%%%%%%%%%
\begin{lemma}{(Accuracy of $S_2$, $S_3$ and $S_4$)}
For a fixed matrix $V \in \mathbb{R}^{d \times K}$
  \label{lem:spectralbounds}
  \begin{align}
    &\nbr{T\rbr{\hat{S}_2- S_2,V,V}}_2 \leq \nbr{T\rbr{\hat{M}_2 - M_2,V,V}}_2 + \nbr{T\rbr{ \hat{M}_1 - M_1,V} }_2^2\nonumber\\
    &\,\,\,\,\,\,\,\,\,\,\,\,\,\,\,\,\,\,\,\,\,\,\,\,\,\,\,\,\,\,\,\,\,\,\,\,\,\,\,\,\,\,\,\,\,\,\,\,\,\,\,\,\,\,\,\,\,\,\,\,\,\,\,+ 2 \nbr{T\rbr{M_1,V}}_2\nbr{T\rbr{ \hat{M}_1 - M_1 ,V} }_2+ \nbr{V}_2^2 \abr{\hat{\sigma}^2 - \sigma^2}\\
    &\nbr{T\rbr{\hat{S}_3-S_3,V,V,V}}_2\nonumber \\
    &\leq \nbr{T\rbr{\hat{M}_3-M_3, V,V,V}}_2 + \rbr{ \nbr{T\rbr{\hat{M}_1-M_1,V}}_2 +  \nbr{T\rbr{M_1,V}}_2}^3- \nbr{T\rbr{M_1,V}}_2^3\nonumber\\
    &\,\,\,\,\,\,+ 3\bigg(\nbr{T\rbr{ \hat{M}_1-M_1 ,V  }}_2  \nbr{T\rbr{ \hat{S}_2-S_2 ,V,V  }}_2 +  \nbr{T\rbr{ M_1,V  }}_2  \nbr{T\rbr{ \hat{S}_2-S_2 ,V,V  }}_2 \nonumber\\
    &\,\,\,\,\,\,+ \nbr{T\rbr{ \hat{M}_1-M_1 ,V  }}_2  \nbr{T\rbr{ S_2 ,V,V  }}_2 \bigg) +3\nbr{V}_2^2 \bigg( \abr{\hat{\sigma}^2-\sigma^2} \nbr{T\rbr{\hat{M}_1-M_1,V}}_2\nonumber\\
    &\,\,\,\,\,\, + \sigma^2\nbr{T\rbr{\hat{M}_1-M_1,V}}_2+ \abr{\hat{\sigma}^2-\sigma^2} \nbr{T\rbr{M_1,V}}_2  \bigg)\\
    &\nbr{T\rbr{\hat{S}_4- S_4,V,VV,V}}_2 \nonumber\\
    &\leq  \nbr{T\rbr{\hat{M}_4-M_4,V, V,V,V}}_2+ \rbr{ \nbr{T\rbr{\hat{M}_1-M_1,V}}_2 +  \nbr{T\rbr{M_1,V}}_2}^4 - \nbr{T\rbr{M_1,V}}_2^4\nonumber\\
    &\,\,\,\,\,\,+ 6\nbr{T\rbr{\hat{S}_2-S_2,V,V}}_2 \nbr{T\rbr{M_1,V}}_2^2 + 6 \bigg( \nbr{T\rbr{\hat{S}_2-S_2,V,V}}_2+\nbr{T\rbr{S_2,V,V}}_2 \bigg)\nonumber\\
    &\,\,\,\,\,\, \bigg( 2\nbr{T\rbr{M_1,V}}_2\nbr{T\rbr{\hat{M}_1-M_1,V}}_2+\nbr{T\rbr{\hat{M}_1-M_1,V}}_2^2 \bigg) 
+ 3\bigg(\nbr{T\rbr{\hat{S}_2-S_2,V,V}}_2^2\nonumber\\
    &\,\,\,\,\,\,+2\nbr{T\rbr{\hat{S}_2-S_2,V,V}}_2\nbr{T\rbr{S_2,V,V}}_2 \bigg) + 6\nbr{V}_2^2 \bigg(
\sigma^2\nbr{T\rbr{\hat{S}_2-S_2,V,V}}_2+\nonumber\\
    &\,\,\,\,\,\, +\abr{\hat{\sigma}^2-\sigma^2} \rbr{\nbr{T\rbr{\hat{S}_2-S_2,V,V}}_2+\nbr{T\rbr{S_2,V,V}}_2 } \bigg)+3 \abr{\hat{\sigma}^4 -  \sigma^4}\nbr{V}_2^4 \nonumber\\
    &\,\,\,\,\,\,+4 \bigg( \nbr{T\rbr{\hat{S}_3-S_3,V,V,V}}_2  \nbr{T\rbr{\hat{M}_1-M_1,V}}_2 +\nbr{T\rbr{M_1,V}}_2 \nbr{T\rbr{\hat{S}_3-S_3,V,V,V}}_2  \nonumber\\
    &\,\,\,\,\,\,\,+  \nbr{T\rbr{S_3,V,V,V}}_2  \nbr{T\rbr{\hat{M}_1-M_1,V}}_2
    \bigg) 
  \end{align}
\end{lemma}

\begin{proof}
To bound the second order tensor, we use the inequality for bounding $\hat{\sigma}$ in Lemma \ref{lem:sigbounds} and get
  \begin{align}
    &\nbr{T\rbr{\hat{S}_2, V,V} - T\rbr{S_2,V,V}}_2 \nonumber \\
    & \leq \nbr{T\rbr{\hat{M}_2 - M_2,V,V}}_2 + \nbr{T\rbr{ (\hat{M}_1 - M_1) \otimes (\hat{M}_1 - M_1) ,V,V} }_2 \\
    &\,\,\,\,\,\,+ 2\nbr{T\rbr{ M_1 \otimes (\hat{M}_1-M_1),V,V} }_2 + \nbr{V}_2^2 \abr{\hat{\sigma}^2 - \sigma^2}\nonumber\\
    & \leq \nbr{T\rbr{\hat{M}_2 - M_2,V,V}}_2 + \nbr{T\rbr{ \hat{M}_1 - M_1 ,V,V} }_2^2 + 2 \nbr{T\rbr{M_1,V}}_2\nbr{T\rbr{ \hat{M}_1 - M_1 ,V,V} }_2 \nonumber\\
    &\,\,\,\,\,\,+ \nbr{V}_2^2 \abr{\hat{\sigma}^2 - \sigma^2} .
  \end{align}
Similarly, for $\hat{S}_3$, we have that
  \begin{align}
    &\nbr{T\rbr{\hat{S}_3, V,V,V} - T\rbr{S_3,V,V,V}}_2 \nonumber\\
    & \leq \nbr{T\rbr{\hat{M}_3-M_3, V,V,V}}_2 + \nbr{ T\rbr{ \hat{M}_1 \otimes \hat{M}_1 \otimes \hat{M}_1 - M_1 \otimes M_1 \otimes M_1,V,V,V} }_2\nonumber\\
    &\,\,\,\,\,\,+3 \nbr{T\rbr{\hat{S}_1\otimes \hat{S}_2-S_1 \otimes S_2, V,V,V} }_2 +3 \nbr{T\rbr{\rbr{\hat{\sigma}^2\hat{M}_1 - \sigma^2M_1} \otimes \one,V,V,V}}_2. 
  \end{align}
  Note that the second term can be written as
  \begin{align}
  \label{eq:m1m1m1}
    &\hat{M}_1 \otimes \hat{M}_1 \otimes \hat{M}_1 - M_1 \otimes M_1 \otimes M_1 \nonumber\\
    &= \rbr{\hat{M_1}-M_1} \otimes \rbr{\hat{M_1}-M_1} \otimes \rbr{\hat{M_1}-M_1}\nonumber \\
    &\,\,\,\,\,\,+ \symm_3 \sbr{ M_1\otimes \rbr{\hat{M_1}-M_1} \otimes \rbr{\hat{M_1}-M_1}} + \symm_3 \sbr{ M_1 \otimes M_1\otimes \rbr{\hat{M_1}-M_1} }.
  \end{align}
Using the same expansion trick, the third term becomes
  \begin{align}
  \label{eq:s2s2}
    \hat{S}_1\otimes \hat{S}_2-S_1 \otimes S_2 = (\hat{S}_1-S_1) \otimes(\hat{S}_2-S_2) + S_1 \otimes(\hat{S}_2-S_2) + (\hat{S}_1-S_1) \otimes S_2.
  \end{align}
Using triangle inequality, the bound for Equation \eqref{eq:m1m1m1} is
  \begin{align}
    &\nbr{ T\rbr{ \hat{M}_1 \otimes \hat{M}_1 \otimes \hat{M}_1 - M_1 \otimes M_1 \otimes M_1,V,V,V} }_2 \nonumber\\
    & \leq \nbr{T\rbr{\hat{M}_1-M_1,V}}_2^3 +3 \nbr{T\rbr{M_1,V}}_2\nbr{T\rbr{\hat{M}_1-M_1,V}}_2^2 \nonumber\\
    &\,\,\,\,\,\,+ 3 \nbr{T\rbr{M_1,V}}_2^2\nbr{T\rbr{\hat{M}_1-M_1,V}}_2,
    \end{align}
    and the bound for Equation \eqref{eq:s2s2} is
    \begin{align}
    & \nbr{T\rbr{\hat{S}_1\otimes \hat{S}_2-S_1 \otimes S_2, V,V,V} }_2 \nonumber\\
    &\leq  \nbr{T\rbr{ \hat{S}_1-S_1,V  }}_2  \nbr{T\rbr{ \hat{S}_2-S_2 ,V,V  }}_2 +  \nbr{T\rbr{ S_1,V  }}_2  \nbr{T\rbr{ \hat{S}_2-S_2 ,V,V  }}_2 \nonumber\\
    &\,\,\,\,\,\,+ \nbr{T\rbr{ \hat{S}_1-S_1 ,V  }}_2  \nbr{T\rbr{ S_2 ,V,V  }}_2\\
    & \nbr{T\rbr{\rbr{\hat{\sigma}^2\hat{M}_1 - \sigma^2M_1} \otimes \one,V,V}}_2 \nonumber\\
    &\leq  \nbr{V}_2^2  \rbr{ \abr{\hat{\sigma}^2-\sigma^2} \nbr{T\rbr{\hat{M}_1-M_1,V}}_2 +\sigma^2\nbr{T\rbr{\hat{M}_1-M_1,V}}_2 + \abr{\hat{\sigma}^2-\sigma^2} \nbr{T\rbr{M_1,V}}_2 }.
  \end{align}
By combining all the inequalities, we get the bound for $S_3$. The bound for $S_4$ can be derived by similar procedure. 
\end{proof}
To complete the bounds, we need to examine the bounds for the whitening matrix and also the whitened tensors.

%%%%%%%%%%%%%%%%%%%%%%%%%%%%%%%%%%%%%%%%%%%%%%%%%%%%%
\subsection{Properties with whitening matrix}
\label{app:whiten}

Note that in Algorithm \ref{alg:eca} we have $W_3:= T\rbr{S_3,W,W,W}$, $W_4:= T\rbr{S_4,W,W,W,W}$. To bound $\nbr{W_3}$ and $\nbr{W_4}$, we use the fact stated in Section \ref{sec:eigenvalue} that these tensor are diagonalized so that finding the norm is actually equivalent to finding the largest eigenvalue of $T(S_3, W, W, W)$ and $T(S_4, W, W, W, W)$, respectively. Note that in Algorithm $\ref{alg:eca}$, the first $K_1$ eigenvectors and their corresponding eigenvalues are solved by conducting tensor decomposition on $W_3$, while the others are extracted from $W_4$. With Equation \eq{eq:lambda-s3} and \eq{eq:lambda-s4}, 
  \begin{align}
  \lambda_i = 
    \begin{cases}
      \frac{-2\pi_i+1}{\sqrt{\pi_i-\pi_i^2}}  \quad \phantom{ \frac{6 \pi_i^2-6 \pi_i+1}{\pi_i-\pi_i^2}}\text{if} \,\, i \leq K_1\\
      \frac{6 \pi_i^2-6 \pi_i+1}{\pi_i-\pi_i^2}  \quad \phantom{\frac{-2\pi_i+1}{\sqrt{\pi_i-\pi_i^2}}}\text{otherwise}.
\end{cases}
\end{align}
As we have mentioned previously, eigenvalues of $S_3$ degenerate to zero at the value of $\pi_i = 0.5$ while eigenvalues of $S_4$ degenerate to zero at the value of $\pi_i \approx 0.2, 0.8$. So here we define thresholds, $\pi_{Th_{up}}$ and $\pi_{Th_{down}}$, such that
  \begin{align}
    \frac{-2\pi_{Th_{down}}+1}{\sqrt{\pi_{Th_{down}}-\pi_{Th_{down}}^2}} = 1, \,\,\, \frac{-2\pi_{Th_{up}}+1}{\sqrt{\pi_{Th_{up}}-\pi_{Th_{up}}^2}} = -1.
  \end{align}
In other words, we solve the latent factors by the third-order moments if $\pi_i < \pi_{Th_{down}}$ or $\pi_i >\pi_{Th_{up}}  $, otherwise we turn to the fourth-order moments. Since $\lambda_i$ is a symmetric function of $\pi_i$ on the $\pi_i = 0.5$ axis for $i \in [K]$, we set $\pi_{Th} = \pi_{Th_{down}}$ to simplify the proof. Here we have
  \begin{align}
	1=\abr{\frac{-2\pi_{Th}+1}{\sqrt{\pi_{Th}-\pi_{Th}^2}} }&\leq \abr{\lambda_i } \leq \frac{-2\pi_{min}+1}{\sqrt{\pi_{min}-\pi_{min}^2}} \quad \,\,\, \text{if}\,\,  i \leq K_1\\
         -2  &\leq \lambda_i \leq   \frac{6 \pi_{Th}^2-6 \pi_{Th}+1}{\pi_{Th}-\pi_{Th}^2} \approx -1\quad\,\, \text{otherwise},
  \end{align}
where $\pi_{min} = \argmax_{i \in [K_1]} \abr{\pi_i - 0.5}$. Since $W_3$ and $W_4$ are diagonalized tensor, we have that 
\begin{align}
\label{eq:W3bounds}
\nbr{ W_3}_2  \leq \frac{-2\pi_{min}+1}{\sqrt{\pi_{min}-\pi_{min}^2}},\,\,\, \nbr{W_4}_2 \leq 2.
\end{align}
Next, in order to bound $\sbr{\hat{W}_i - W_i}$, we need to consider the bounds using empirical whitening matrix.
Let $\hat{W}$ denotes the empirical whitening matrix in our algorithm. Here we define $W := \hat{W}(\hat{W}S_2\hat{W})^{-\frac{1}{2}}$ and $\epsilon_{S_2} := \nbr{\hat{S}_{2,k} - S_2}_2/\varsigma_k \sbr{S_2}$ in order to use the bounds for whitening matrix stated in lemma 10 in \cite{HsuKak12}.
\begin{lemma}{(Lemma 10 in \cite{HsuKak12})} 
  \label{lem:whiten}
Assume $\epsilon_{S_2} \leq 1/3.$ We have
  \begin{align*}
    &1.\,\, W^\top S_2W=I, \,\,\,\,\,\, 2.\,\, \nbr{\hat{W}}_2 \leq \frac{1}{\sqrt{(1-\epsilon_{S_2}) \varsigma \sbr{S_2} }},\\
    &3.\,\, \nbr{\rbr{\hat{W}S_2\hat{W}}^{1/2}-I }_2 \leq 1.5 \epsilon_{S_2},   \,\,\,\,\,\,\nbr{\rbr{\hat{W}S_2\hat{W}}^{-1/2}-I }_2 \leq 1.5 \epsilon_{S_2}\\
    &  \,\,\,\,\,\,\nbr{(\hat{W})^\top A\mathrm{diag}(\pi-\pi^2)^{1/2}}_2 \leq \sqrt{1+1.5 \epsilon_{M_2}}, \\
    & \,\,\,\,\,\, \nbr{(\hat{W}-W)^\top A\mathrm{diag}(\pi-\pi^2)^{1/2}}_2 \leq \sqrt{1+1.5 \epsilon_{M_2}}.
  \end{align*}

\end{lemma} 

Using Lemma \ref{lem:whiten}, we can complete the bounds for empirical whitened tensors.
%%%%%%%%%%%%%%%%%%%%%%%%%%%%%%%%%%%%%%%%%%%%%%%%%%%%
\begin{lemma}{}
  \label{lem:Wbounds}
Assume $\epsilon_{S_2} \leq 1/3.$ Then
  \begin{align*}
    &\nbr{\hat{W}_3- W_3  }_2\leq \nbr{T\rbr{S_3- \hat{S}_3,\hat{W}, \hat{W},\hat{W}} }_2 + 3  \frac{-2\pi_{min}+1}{\sqrt{\pi_{min}-\pi_{min}^2}} \\
    &\nbr{\hat{W}_4- W_4 }_2\leq \nbr{T\rbr{S_4- \hat{S}_4,\hat{W}, \hat{W},\hat{W},\hat{W}} }_2 + 10
  \end{align*}
\end{lemma}
\begin{proof}
Here we only show the second inequality, the first one can be derived with similar procedure. 
  \begin{align}
    \nbr{\hat{W}_4- W_4 }_2 &= \nbr{T\rbr{S_4,W, W,W,W} -T\rbr{ \hat{S}_4,\hat{W},\hat{W}, \hat{W},\hat{W}} }_2\nonumber \\
    &\leq \nbr{T\rbr{S_4,W, W,W,W} -T\rbr{S_4,\hat{W},\hat{W}, \hat{W},\hat{W}} }_2 + \nbr{T\rbr{S_4-\hat{S}_4,\hat{W},\hat{W}, \hat{W},\hat{W}} }_2
  \end{align}
For the first term, using Lemma \ref{lem:whiten} and Equation \eqref{eq:W3bounds}, we have:
  \begin{align}
    &\nbr{T\rbr{S_4,W, W,W,W} -T\rbr{S_4,\hat{W},\hat{W}, \hat{W},\hat{W}} }_2\nonumber\\
     &\leq  \nbr{T\rbr{S_4,\hat{W}-W,\hat{W}, \hat{W},\hat{W}} }_2 +\nbr{T\rbr{S_4,W,\hat{W}-W, \hat{W},\hat{W}} }_2\nonumber\\ 
  &\,\,\,\,\,\,+\nbr{T\rbr{S_4,W,W, \hat{W}-W,\hat{W}} }_2+\nbr{T\rbr{S_4,W,W, W,\hat{W}-W} }_2\\
 &\leq \nbr{T\rbr{S_4,W,W,W,W} }_2\nbr{(\hat{W}^\top S_2\hat{W})^{1/2}-\one }\rbr{\nbr{(\hat{W}^\top S_2\hat{W})^{1/2}}_2^3+\cdots+\nbr{(\hat{W}^\top S_2\hat{W})^{1/2}}_2^0}\nonumber\\
 &\leq \nbr{T\rbr{S_4,W,W,W,W} }_2 \cdot (1.5 \epsilon_{S_2}) \rbr{(1+1.5 \epsilon_{S_2})^3+\cdots(1+1.5 \epsilon_{S_2})+1 } \leq 5 \cdot 2 = 10
  \end{align}
\end{proof}
%%%%%%%%%%%%%%%%%%%%%%%%%%%%%%%%%%%%%%%%%%%%%%%%%%%%
\subsection{Reconstruction analysis}
Before putting everything together, we utilize the eigendecomposition analysis in Appendix C.7 of \cite{HsuKak12}. First, we consider the case where $A_i$ is recovered by applying tensor decomposition on $W_3$, i.e., for $i \leq K_1$. Note that in Algorithm \ref{alg:eca}
  \begin{align}
    Z_i = \frac{\pi_i - 3\pi_i^2 + 2\pi_i^3}{\rbr{-\pi_i^2 + \pi_i} \nbr{W_3}}  
          = \frac{-2 \pi_i + 1}{ \lambda_i} =   \sqrt{\pi_i-\pi_i^2}.
  \end{align}
Similarly, for $i \in \cbr{K_1, \cdots, K}$,
  \begin{align}
    Z_i = \frac{6\pi_i^2 - 6\pi_i + 1}{\sqrt{-\pi_i^2 + \pi_i} \nbr{ W_4}}  
         =  \frac{6\pi_i^2 - 6\pi_i + 1}{\sqrt{-\pi_i^2 + \pi_i} \lambda_i} =   \sqrt{\pi_i-\pi_i^2}.
  \end{align}
Following the approached in \cite{HsuKak12}, define 
\begin{align*}
\gamma_{S_3} &:= \frac{1}{2\max_{i \in [K_1]} \sqrt{  \rbr{\pi_i-\pi_i^2} }\sqrt{eK}\binom{K+1}{2}  },
 \,\,\,\epsilon_{S_3} := \frac{\nbr{W_3 - \hat{W}_3}}{  \gamma_{S_3}}
,\\
 \gamma_{S_4} &:= \frac{1}{2\max_{i > K_1 }\sqrt{  \rbr{\pi_i-\pi_i^2} }\sqrt{eK}  \binom{K+1}{2} },\,\,\,
  \epsilon_{S_4} := \frac{\nbr{W_4 - \hat{W}_4}}{ \gamma_{S_4}}
\end{align*}
 We derive the overall guaranteed bounds using the same approach in \cite{HsuKak12}. Before stating the inequality, we define
\begin{align*}
    \kappa\sbr{S_2} &:= \varsigma_1\sbr{S_2} / \varsigma_K \sbr{S_2},\\
    \epsilon_{0,i} &:= \begin{cases} \rbr{5.5 \epsilon_{S_2} + 7 \epsilon_{S_3}}/\sqrt{\pi_{min}-\pi_{min}^2 } \,\,\,\,\,\,& \text{if}\,\,\, i \in [K_1]\\
                         13.75 \epsilon_{S_2} + 17.5 \epsilon_{S_4}   & \text{otherwise} \end{cases},\\
\epsilon_{1,i} &:= \begin{cases} \rbr{\rbr{6.875 \kappa \sbr{S_2}^{1/2}+2 } \epsilon_{S_2} + \rbr{ 8.75 \kappa \sbr{S_2}^{1/2} + \gamma_{S_3} \sqrt{\pi_{min}-\pi_{min}^2 }}\epsilon_{S_3}}\\
             \,\,\,\,\,\,\,\,\,\,\,\,\,\,\,\,\,\,\,\,\,\,\,\,\,\,\,\,\,\,\,\,\,\,\,\,\,\,\,\,\,\,\,\,\,\,\,\,\,\,\,\,\,\,\,\,\,\,\,\,\,\,\,\,\,\,\,\,\,\,\,\,\,\,\,\,\,\,\,\,\,\,\,\,\,\,\,\,\,\,\,\,\,\,\,\,\,\,\,\,\,\,\,\,\,\,\,\,\,\,\,\ /\rbr{\gamma_{S_3} \sqrt{\pi_{min}-\pi_{min}^2 }}
& \text{if}\,\,i \in [K_1]\\
                          2.5 \rbr{\rbr{6.875 \kappa \sbr{S_2}^{1/2}+2 } \epsilon_{S_2} + \rbr{ 8.75 \kappa \sbr{S_2}^{1/2} + 0.4 \gamma_{S_4} } \epsilon_{S_4}}/\gamma_{S_4}    & \text{otherwise} \end{cases},
\end{align*}
where $\pi_{min} = \argmax_{i \in [K_1]} \abr{\pi_i - 0.5}$ as we have defined previously.
%%%%%%%%%%%%%%%%%%%%%%%%%%%%%%%%%%%%%%%%%%%%%%%%%%%%
\begin{lemma}{(Reconstruction Accuracy)}
  \label{lem:overall}
 Assume $\epsilon_{S_2} \leq 1/3$, $\epsilon_{S_3} \leq 1/4$ and $\epsilon_{S_4} \leq 1/4$, and $\epsilon_1 \leq 1/3$. There exists a permutation $\pi$ on $[K]$ such that
  \begin{align*}
      \nbr{\Phi_{\pi(i)} - \hat{\Phi}_i } \leq 3 \nbr{\Phi_{\pi(i)} }_2 \epsilon_{1,i} + 2 \nbr{S_2}_2^{1/2} \epsilon_{0,i},\,\,\,\,\,\,   \forall i \in [K]
  \end{align*}
\end{lemma}

\subsection{Proof of Theorem \ref{th:reconstruction}}

  We follow the similar approaches in \cite{HsuKak12}. In this proof, we use $c, c_1, c_2, \cdots$ to denote some positive constant. First we assume sample size $n \geq c \cdot K \log \rbr{1/\delta}$. By Lemma \ref{lem:submoment} and \ref{lem:moments}, with probability greater than $1-\delta$,
  \begin{align}
    \nbr{\hat{M}_1-M_1}_2 \leq& c_1 \sigma \sqrt{ \frac{d+\log\rbr{2^k/\delta}}{\tilde{\pi}n}} + c_1 \sum \limits_{i=1}^K \nbr{A_i}_2 \sqrt{\frac{2^K \log\rbr{1/\delta}}{n}} \\
     \nbr{\hat{M}_2-M_2}_2 \leq
    & c_1\rbr{\sigma^2 \sqrt{ \frac{d+\log\rbr{2^k/\delta}}{\tilde{\pi}n} }+\sigma^2 \frac{d+\log\rbr{2^k/\delta}}{\tilde{\pi}n}+\sigma \sqrt{ \frac{d+\log\rbr{2^k/\delta}}{\tilde{\pi}n} }  } \\
    &+c_1\rbr{  \sum \limits_{i=1}^K \nbr{A_i}_2^2 + \sigma^2} \sqrt{\frac{2^K \log\rbr{1/\delta}}{n}}\\
    \leq& c_1\rbr{ 2\rbr{  \sum \limits_{i=1}^K \nbr{A_i}_2^2 + \sigma^2}  \sqrt{ \frac{d+\log\rbr{2^k/\delta}}{\tilde{\pi}n} }+ \sigma^2 \frac{d+\log\rbr{2^k/\delta}}{\tilde{\pi}n}}
  \end{align}
Using Lemma \ref{lem:sigbounds}, 
  \begin{align}
    \max& \cbr{ \abr{\hat{\sigma}^2-\sigma^2}, \nbr{\hat{S}_{2,K}-S_2}_2  } \nonumber\\
    \leq& 4c_1\rbr{ 2\rbr{  \sum \limits_{i=1}^K \nbr{A_i}_2^2 + \sigma^2}  \sqrt{ \frac{d+\log\rbr{2^k/\delta}}{\tilde{\pi}n} }+ \sigma^2 \frac{d+\log\rbr{2^k/\delta}}{\tilde{\pi}n}} \nonumber\\
    &+ 8c_1^2 \rbr{  \sigma \sqrt{ \frac{d+\log\rbr{2^k/\delta}}{\tilde{\pi}n}} + \sum \limits_{i=1}^K \nbr{A_i}_2 \sqrt{\frac{2^K \log\rbr{1/\delta}}{n}} }^2\nonumber\\
    &+8c_1 \nbr{M_1}_2 \rbr{  \sigma \sqrt{ \frac{d+\log\rbr{2^k/\delta}}{\tilde{\pi}n}} + \sum \limits_{i=1}^K \nbr{A_i}_2 \sqrt{\frac{2^K \log\rbr{1/\delta}}{n}} }\\
    \leq &  c_2\rbr{  \sum \limits_{i=1}^K \nbr{A_i}_2^2 + \sigma^2}\rbr{   \sqrt{ \frac{d+\log\rbr{2^k/\delta}}{\tilde{\pi}n}} +  \frac{d+\log\rbr{2^k/\delta}}{\tilde{\pi}n}}  .
  \end{align}
We have 
\begin{align}
  \max \cbr{\frac{\abr{\hat{\sigma}^2-\sigma^2} }{\varsigma_K\rbr{S_2}}, \epsilon_{S_2} } \leq c_3 \frac{\gamma_{S_3}^2\tilde{\pi}}{\kappa\sbr{S_2}^{1/2}} \leq 1/3
\end{align}
Set sample size as 
  \begin{align*}
    n \geq c \frac{d+\log\rbr{2^k/\delta}}{\tilde{\pi}} \rbr{ \sbr{ \frac{\kappa \sbr{S_2}^{1/2}  \rbr{  \sum \limits_{i=1}^K \nbr{A_i}_2^2 + \sigma^2}}{\gamma_{S_3}^2 \tilde{\pi} \varsigma_K\sbr{S_2} \epsilon} }^2 + \sbr{ \frac{\kappa \sbr{S_2}^{1/2}  \rbr{  \sum \limits_{i=1}^K \nbr{A_i}_2^2 + \sigma^2}}{\gamma_{S_3}^2 \tilde{\pi} \varsigma_K\sbr{S_2} \epsilon} }}.
  \end{align*}
To examine the moments after multiplying whitening matrix $W$, by Lemma \ref{lem:whiten},
  \begin{align}
    \nbr{\hat{W}}_2 &\leq \sqrt{1.5 /\ \varsigma_K \sbr{S_2}} \\
    \max_{z_i \in [2^K]} \nbr{T\rbr{M_{1,z_i}, \hat{W}}} \leq& \nbr{\hat{W}^\top A\mathrm{diag}\rbr{\pi-\pi^2}^{1/2}}_2/\sqrt{\pi_{min}-\pi_{min}^2}\\
                                                       \leq& \sqrt{1.5/\rbr{\pi_{min}-\pi_{min}^2}} \\
    \max_{z_i \in [2^K]} \nbr{T\rbr{M_{2,z_i}, \hat{W},\hat{W}}} \leq&  1.5/\rbr{\pi_{min}-\pi_{min}^2} + \sigma^2 \rbr{1.5/\varsigma_K\sbr{S_2}}  \\
    \max_{z_i \in [2^K]} \nbr{T\rbr{M_{3,z_i}, \hat{W},\hat{W},\hat{W}}} \leq& \rbr{1.5/\rbr{\pi_{min}-\pi_{min}^2}}^{3/2}\nonumber\\
    & + 3 \sigma^2 \sqrt{1.5/\rbr{\pi_{min}-\pi_{min}^2}} \rbr{1.5/\varsigma_K\sbr{S_2}}  \\
    \max_{z_i \in [2^K]} \nbr{T\rbr{M_{4,z_i}, \hat{W},\hat{W},\hat{W},\hat{W}}} &\leq \rbr{1.5/\rbr{\pi_{min}-\pi_{min}^2}}^2 + 6 \sigma^2 \frac{2.25}{\rbr{\pi_{min}-\pi_{min}^2}\varsigma_K\sbr{S_2} } \nonumber\\
    &+ 3\sigma^4 \rbr{ 1.5/\varsigma_K\sbr{S_2}}^2
  \end{align}
Using Lemma \ref{lem:moments}, 
\begin{align}
  \nbr{T\rbr{\hat{M}_1-M_1,\hat{W}}} \leq& c_4 \frac{\sigma}{\varsigma_K\sbr{S_2}^{1/2}} \sqrt{\frac{K+\log \rbr{2^K/\delta}}{\tilde{\pi}n}} + c_4 \frac{1}{\sqrt{\pi_{min}-\pi_{min}^2}}  \sqrt{\frac{2^K\log \rbr{1/\delta}}{n}}\\
 \nbr{T\rbr{\hat{M}_2-M_2,\hat{W},\hat{W}}} \leq& c_4 \frac{\sigma^2}{\varsigma_K\sbr{S_2}}\rbr{ \sqrt{\frac{K+\log \rbr{2^K/\delta}}{\tilde{\pi}n}}+\frac{K+\log \rbr{2^K/\delta}}{\tilde{\pi}n} }\nonumber\\
  &+ c_4 \frac{1}{\sqrt{\pi_{min}-\pi_{min}^2} \varsigma_K\sbr{S_2}^{1/2}}  \sqrt{\frac{2^K\log \rbr{1/\delta}}{n}}\nonumber \\
  &+ c_4\rbr{\frac{1}{\varsigma_K\sbr{S_2}}+ \frac{1}{ \sqrt{\pi_{min}-\pi_{min}^2}}} \sqrt{\frac{K\log \rbr{1/\delta}}{n}}
\end{align}
\begin{align}
 \nbr{T\rbr{\hat{M}_3-M_3,\hat{W},\hat{W},\hat{W}}}& \leq c_4 \frac{\sigma^3}{\varsigma_K\sbr{S_2}^{3/2}} \sqrt{\frac{\rbr{K+\log \rbr{2^K/\delta}}^3}{\tilde{\pi}n}}\nonumber\\
 &+c_4 \frac{\sigma^2}{\varsigma_K\sbr{S_2}\sqrt{\pi_{min}-\pi_{min}^2}}\rbr{ \sqrt{\frac{K+\log \rbr{2^K/\delta}}{\tilde{\pi}n}}+\frac{K+\log \rbr{2^K/\delta}}{\tilde{\pi}n} }\nonumber\\
  &+ c_4 \frac{\sigma}{ \varsigma_K\sbr{S_2}^{1/2}\rbr{\pi_{min}-\pi_{min}^2}}  \sqrt{\frac{K+\log \rbr{2^K/\delta}}{\tilde{\pi}n}}\nonumber \\
  &+ c_4\rbr{\frac{\sigma^2}{\varsigma_K\sbr{S_2}\sqrt{\pi_{min}-\pi_{min}^2}}+ \frac{1}{ \rbr{\pi_{min}-\pi_{min}^2}^{3/2}}} \sqrt{\frac{K\log \rbr{1/\delta}}{n}}
\end{align}
\begin{align}
 \bigg\lVert T\bigg(\hat{M}_4-M_4,,\hat{W},\hat{W},\hat{W},\hat{W} \bigg) \bigg\rVert \leq& c_4 \frac{\sigma^4}{\varsigma_K\sbr{S_2}^2} \rbr{ \rbr{\frac{K+\log \rbr{2^K/\delta}}{\tilde{\pi}n}}+\rbr{\frac{K+\log \rbr{2^K/\delta}}{\tilde{\pi}n}}^{3/2}}\nonumber\\
  &+c_4 \frac{\sigma^3}{\varsigma_K\sbr{S_2}^{3/2}} \sqrt{\frac{\rbr{K+\log \rbr{2^K/\delta}}^3}{\tilde{\pi}n}}\nonumber\\
 &+c_4 \frac{\sigma^2}{\varsigma_K\sbr{S_2}}\rbr{ \sqrt{\frac{K+\log \rbr{2^K/\delta}}{\tilde{\pi}n}}+\frac{K+\log \rbr{2^K/\delta}}{\tilde{\pi}n} }\nonumber\\
  &+ c_4 \frac{\sigma}{ \varsigma_K\sbr{S_2}^{1/2}}   \sqrt{\frac{K+\log \rbr{2^K/\delta}}{\tilde{\pi}n}} \nonumber\\
  &+ c_4\rbr{\frac{\sigma^4}{\varsigma_K\sbr{S_2}}+\frac{\sigma^2}{\varsigma_K\sbr{S_2}}} \sqrt{\frac{K\log \rbr{1/\delta}}{n}}.
\end{align}
With Lemma \ref{lem:spectralbounds} and \ref{lem:Wbounds}, 
\begin{align}
 \nbr{T\rbr{\hat{S}_2- S_2,\hat{W},\hat{W}}}_2 \leq& \nbr{T\rbr{\hat{M}_2 - M_2,\hat{W},\hat{W}}}_2 + \nbr{T\rbr{ \hat{M}_1 - M_1 ,\hat{W}} }_2^2\nonumber\\
    &+ 2 \frac{\pi_{max}}{\sqrt{\pi_{max}-\pi_{max}^2 }}\nbr{T\rbr{ \hat{M}_1 - M_1,\hat{W}} }_2+ \frac{1.5}{\varsigma_K\sbr{S_2}} \abr{\hat{\sigma}^2 - \sigma^2}
\end{align}
\begin{align}
 \bigg \lVert T\bigg(&\hat{S}_3-S_3,\hat{W},\hat{W},\hat{W}\bigg) \bigg \rVert \leq \nbr{T\rbr{\hat{M}_3-M_3, \hat{W},\hat{W},\hat{W}}}_2 \nonumber\\
    &+ \rbr{ \nbr{T\rbr{\hat{M}_1-M_1,\hat{W}}}_2 +  \frac{\pi_{max}}{\sqrt{\pi_{max}-\pi_{max}^2 }}}^3- \rbr{\frac{\pi_{max}}{\sqrt{\pi_{max}-\pi_{max}^2 }}}^3\nonumber\\
    &+ 3\bigg(\nbr{T\rbr{ \hat{M}_1-M_1,\hat{W} }}_2  \nbr{T\rbr{ \hat{S}_2-S_2 ,\hat{W},\hat{W}  }}_2 +  \frac{\pi_{max}}{\sqrt{\pi_{max}-\pi_{max}^2 }}  \nbr{T\rbr{ \hat{S}_2-S_2 ,\hat{W},\hat{W}  }}_2 \nonumber\\
    &+ \nbr{T\rbr{ \hat{M}_1-M_1 ,\hat{W}  }}_2  \frac{-2\pi_{min}+1}{\sqrt{\pi_{min}-\pi_{min}^2}} \bigg) +\frac{4.5}{\varsigma_K\sbr{S_2}} \bigg( \abr{\hat{\sigma}^2-\sigma^2} \nbr{T\rbr{\hat{M}_1-M_1,\hat{W}}}_2\nonumber \\
    &+ \sigma^2\nbr{T\rbr{\hat{M}_1-M_1,\hat{W}}}_2+ \abr{\hat{\sigma}^2-\sigma^2}   \frac{\pi_{max}}{\sqrt{\pi_{max}-\pi_{max}^2 }} \bigg).
\end{align}
Plug this in Lemma \ref{lem:Wbounds}, we get the overall bounds for $\nbr{W_3-\hat{W}_3}$. To get $\epsilon_{S_3} \leq c_5 \frac{\gamma_{S_3} \sqrt{\tilde{\pi}}}{\kappa\sbr{S_2}^{1/2}} \epsilon$, we set 
\begin{align}
\label{eq:n1}
   n \geq \mathrm{poly}\rbr{d, K, \frac{1}{\epsilon}, \log(1/\delta), \frac{1}{\tilde{\pi}},\frac{ \varsigma_1\sbr{S_2}}{ \varsigma_K\sbr{S_2}}, \frac{\sum \limits_{i=1}^K \nbr{A_i}_2^2}{  \varsigma_K\sbr{S_2} }, \frac{\sigma^2}{  \varsigma_K\sbr{S_2}},\frac{1}{\sqrt{\pi_{min}-\pi_{min^2}}}, \frac{\pi_{max}}{\sqrt{\pi_{max}-\pi_{max}^2 }}}
\end{align}
Similarly, for $\Phi_i$ reconstructed by $\hat{W}_4$, n should be set to 
\begin{align}
\label{eq:n2}
  n \geq \mathrm{poly}\rbr{d, K, \frac{1}{\epsilon}, \log(1/\delta), \frac{1}{\tilde{\pi}},\frac{ \varsigma_1\sbr{S_2}}{ \varsigma_K\sbr{S_2}}, \frac{\sum \limits_{i=1}^K \nbr{A_i}_2^2}{  \varsigma_K\sbr{S_2} }, \frac{\sigma^2}{  \varsigma_K\sbr{S_2}}},
\end{align}
in order to  $\epsilon_{S_4} \leq c_6 \frac{\gamma_{S_4} \sqrt{\tilde{\pi}}}{\kappa\sbr{S_2}^{1/2}} \epsilon$.
The overall bounds can be obtained by Equation \ref{eq:n1}, \ref{eq:n2} and Lemma \ref{lem:overall}.
%%%%%%%%%%%%%Tail Inequality%%%%%%%%%%%%%%%%%%%%%%%%%%%%%%%%%%
\section{Tail Inequalities}
Here we derive the tail inequality for the fourth-order subgaussian random tensor.
\begin{lemma}
  \label{lem:normal4}
  Let $x_1, x_2, \cdots,x_n$ be $i.i.d.$ random variables such that 
  \begin{align}
    \Eb_i \sbr{\exp\rbr{\eta x_i}} \leq \exp\rbr{\gamma \eta^2/2}\,\,\,  \forall \eta \in \mathbb{R}
  \end{align} 
  Then for any $t > 0$ and $\frac{\gamma t}{n} < \frac{1}{4}$,
  \begin{align}
  \label{eq:tail}
    \Pr \sbr{\frac{1}{n} \sum \limits_{i=1}^n \rbr{x_i^4 - \Eb_i \sbr{x_i^4}} > \gamma \sqrt{\frac{64 t^2 }{n^2} \rbr{ 8\gamma - \frac{16 \gamma^2 t}{n} } \frac{1}{\rbr{1-4\gamma \frac{t}{n}}^2} } } \leq e^{-t},\\
   \Pr \sbr{\frac{1}{n} \sum \limits_{i=1}^n \rbr{x_i^4 - \Eb_i \sbr{x_i^4}} < - \gamma  \sqrt{\frac{8 t^2 }{n^2} \rbr{ 2\gamma + \frac{ \gamma^2 t}{n} } \frac{1}{\rbr{1+\gamma \frac{t}{n}}^2} }} \leq e^{-t},
  \end{align}
\end{lemma}
\begin{proof}
  We use Chernoff's bounding method to derive the inequality. For $\eta < \frac{1}{2\epsilon\gamma}$, set $\eta = \frac{1-\sigma}{2\gamma\epsilon}$ for some $\sigma > 0$, we have
  \begin{align}
    \Eb_i \sbr{\mathrm{exp}(\eta x_i^4)} &= 1+\eta \Eb_i \sbr{x_i^4} + \eta \int_{0}^{\infty} \rbr{\mathrm{exp}\rbr{\eta \epsilon^2}-1} \Eb_i\sbr{\mathds{1}_{\{ x_i^4 > \epsilon^2 \}}} d \epsilon^2\\
    & \leq  1 + \eta \Eb_i \sbr{x_i^4} + 2\eta  \int_{0}^{\infty} \rbr{\exp \rbr{\eta \epsilon^2}-1} \exp \rbr{\frac{-\epsilon}{2\gamma}}  2\epsilon d \epsilon\\
    & \leq 1 + \eta \Eb_i \sbr{x_i^4} + 4 \eta \rbr{\int_{0}^{\infty} \epsilon \exp \rbr{\frac{-\sigma \epsilon}{2\gamma} } d\epsilon - \int_{0}^{\infty} \epsilon \exp \rbr{\frac{- \epsilon}{2\gamma} } d\epsilon } \\
    & \leq 1 + \eta \Eb_i \sbr{x_i^4} + 4 \eta \rbr{ 4 \gamma^2 \rbr{\frac{1}{\sigma^2}-1 } } \\
    & \leq \exp \rbr{ \eta \Eb_i \sbr{x_i^4} + 4 \eta \rbr{ 4 \gamma^2 \rbr{\frac{1}{\sigma^2}-1 } }  }
  \end{align}
  The second line uses the fact that $\Pr \sbr{ x_i^4 > \epsilon^2 } \leq \frac{ \Eb \sbr{\exp\rbr{\alpha \abr{x_i}}}}{ \exp\rbr{\alpha \epsilon^{1/2}}} \leq 2 \frac{\exp\rbr{\gamma \alpha^2/2}}{ \exp\rbr{\alpha \epsilon^{1/2}}} = 2 \exp \rbr{-\frac{\epsilon}{2\gamma} } $ with $\alpha = \frac{\epsilon^{1/2}}{\gamma}$. Since the above inequality holds for $i = 1, 2, \cdots, n$, 
  \begin{align}
    &\Eb \sbr{\exp \rbr{\eta \sum_{i=1}^n \rbr{x_i^4 - \Eb_i \sbr{x_i^4}} } } = \prod \limits_{i=1}^n \Eb_i \sbr{ \exp \rbr{\eta \rbr{x_i^4 - \Eb_i \sbr{x_i^4}} }}\\
    & \leq \exp \rbr{16 n\eta  \gamma^2 \rbr{\frac{1}{\sigma^2}-1 } }
  \end{align}
With Chernoff's inequality, for $0 \leq \eta < \frac{1}{2 \epsilon \gamma}$ and $\epsilon \geq 0$, 
  \begin{align}
    &\Pr \sbr{\frac{1}{n} \sum \limits_{i=1}^n \rbr{x_i^4 - \Eb_i \sbr{x_i^4}} > \epsilon } \leq \exp \rbr{-\eta n \epsilon+16 n\eta  \gamma^2 \rbr{\frac{1}{\sigma^2}-1 }}.
  \end{align}
Setting $\eta = \frac{1-\sigma}{2\gamma\epsilon}$ and $\sigma = 1- \frac{4\gamma t}{n}$, for $\frac{\gamma t}{n} < \frac{1}{4}$, we get the first inequality. For $\eta <0 $ and $\epsilon \geq 0$,
  \begin{align}
    &\Pr \sbr{\frac{1}{n} \sum \limits_{i=1}^n \rbr{x_i^4 - \Eb_i \sbr{x_i^4}} < -\epsilon } \leq \exp \rbr{\eta n \epsilon+16 n\eta  \gamma^2 \rbr{\frac{1}{\sigma^2}-1 }}.
  \end{align}
Setting $\sigma = 1 + \gamma t$ gives the claimed inequality.
  
\end{proof}
\vspace{-10mm}
\begin{lemma}{(Fourth-order normal random vectors).} 
  \label{lem:random4}
  Let $y_1,y_2, \cdots y_n \in \mathbb{R}^d$ be $i.i.d.$ $\mathrm{N}(0,I)$ random vectors. For $\epsilon_0 \in (0,1/4)$ and $\delta \in (0,1)$,
  \begin{align}
    \Pr \sbr{ \nbr{\frac{1}{n} \sum \limits_{i=1}^n y_i \otimes y_i \otimes y_i \otimes y_i -  \Eb \sbr{\epsilon  \otimes \epsilon  \otimes \epsilon  \otimes \epsilon } }_2 > \frac{1}{1-4\epsilon_0}  \epsilon_{\epsilon_0,t,n}  } \leq 2\delta
  \end{align}  
where
  \begin{align}
    \epsilon_{\epsilon_0,t,n} = \sqrt{\frac{2048{\ln\rbr{(1+2/\epsilon_0)^d/\delta }}^2 }{n^2}  +\frac{8 {\ln\rbr{(1+2/\epsilon_0)^d/\delta }}^3 }{n^3} }
  \end{align}
\end{lemma}
\begin{proof}
  We follow the approach of \citep{HsuKakZha09}. Let $Y:=\frac{1}{n} \sum \limits_{i=1}^n y_i \otimes y_i \otimes y_i \otimes y_i -  \Eb \sbr{\epsilon \one \otimes \epsilon \one \otimes \epsilon \one \otimes \epsilon \one}  $. By \cite{Pisier89}, there exists $Q \subseteq \mathcal{S}^{d-1} := \{ \alpha \in \mathrm{R}^d : \nbr{\alpha}_2 = 1\}$ with cardinality at most $(1+2\epsilon)^d$ such that $\forall \alpha \in \mathcal{S}^{d-1} \exists q \in Q$ $\nbr{\alpha-q}_2 \leq \epsilon_0$. Since, for any $q \in Q$, $y_i^\top q$ is distributed as $N(0,1)$, with union bounds and Lemma \ref{lem:normal4}, for ,$\Pr \sbr{ {\exists q \in Q}  \abr{T\rbr{Y,q,q,q,q}  } > \epsilon_{\epsilon_0,t,n}  } \leq 2\delta$. So we assume with probability greater than $1 - 2\delta$, $\forall q \in Q$, $\abr{T\rbr{Y,q,q,q,q}  } \leq \epsilon_{\epsilon_0,t,n}$.
Let $\alpha_0 = \argmax_{\alpha \in \mathcal{S}^{d-1}} \abr{T\rbr{Y,\alpha,\alpha,\alpha,\alpha}}$, we have
  \begin{align}
    \nbr{Y}_2 =& \abr{T\rbr{Y, \alpha_0, \alpha_0, \alpha_0, \alpha_0}}\\
    \leq& \min_{q \in Q} \abr{T\rbr{Y, q,q,q,q}} + \abr{T\rbr{Y, \alpha_0 - q,q,q,q}} +\abr{T\rbr{Y, \alpha_0 , \alpha_0 - q,q,q}} \nonumber\\
    &+\abr{T\rbr{Y, \alpha_0 ,\alpha_0, \alpha_0 - q,q}}+\abr{T\rbr{Y, \alpha_0,\alpha_0,\alpha_0, \alpha_0 - q}} \\
    & \leq  \min_{q \in Q} \abr{T\rbr{Y, q,q,q,q}} + 4 \nbr{\alpha_0 -q} \nbr{Y}_2\\
    & \leq  \epsilon_{\epsilon_0,t,n} + 4 \epsilon_0 \nbr{Y}_2,
  \end{align}
which yields
  \begin{align}
    \nbr{Y}_2 \leq \frac{1}{1-4\epsilon_0} \cdot \epsilon_{\epsilon_0,t,n}
  \end{align}
\end{proof}

\section{Concentration of Measure for the HDP}

\subsection{Effective sample size}

In the following it will be useful to keep track of the explicit
weighting inherent in the definition of the moments $M_r^\ib$. In this
context recall that $M_r^{\ib} = \frac{1}{|c(\ib)|} \sum_{\jb \in
  c(\ib)}  M_r^\jb$ and that furthermore for leaf nodes $M_r^\ib$ is the weighted
average over all combinations of occurring attributes.

\begin{definition}[Effective sample size]
  For any average $x := \sum_i \eta_i x_i$, we denote by 
  $n_\mathrm{eff} := \frac{\nbr{\eta}_1^2}{\nbr{\eta}_2^2}$ its
  effective sample size.
\end{definition}
To see that this definition is sensible, consider the case of $\eta_i
= l^{-1}$ and $\eta \in \RR^l$. In this case we obtain $n_\mathrm{eff}
= l$, as desired for even weighting. 
\begin{lemma}
  Denote by $\eta_i \in \RR^{l_i}$ normalized vectors with $\eta_{ij}
  \geq 0$ and $\nbr{\eta_i}_1 = 1$. Moreover, let $\lambda_i \geq 0$
  with $\sum_i \lambda_i = 1$. Then the effective sample size of the
  concatenated vector $\eta := \uplus_{i} \lambda_i \eta_i$ satisfies
  \begin{align*}
    \frac{1}{n_{\mathrm{eff}}} = \sum_i \frac{\lambda_i^2}{n_{\mathrm{eff},i}}
  \end{align*}
\end{lemma}
This follows by direct calculation. In particular, note that
$\nbr{\eta}_1=1$. Hence $\nbr{\eta}_2^2 = \sum_i \lambda_i^2
\nbr{\eta_i}_2^2$. Taking the inverse yields the claim. 

We now explicitly construct an auxiliary weighting vector
$\eta^{(\ib,r)}$ of dimensionality $\rho^{(\ib,r)}$. At the leaf level we
use a vector of dimensionality $1$ and weights
$1$. As we ascend through the tree, all children are
given weights $1/|c(\ib)|$ and a weighting vector $\eta^{(\ib,r)} =
|c(\ib)|^{-1} \uplus_{\jb \in c(\ib)} \eta^{(\jb,r)}$ is assembled. For
convenience we will sometimes also make use of $d(\ib,r)$, the set of
all index vectors used in $\eta_\ib$, which is the same as the number of documents under this node. 

% Denote by $d(\ib,r)$ all samples in the documents under node $\ib$. 
% For each leaf node $\ib$ define a vector 
% $$\eta^{(\ib,r)} = \mathbf{1} (n_\ib-r)!/n_\ib!$$ 
% For all other non-leaf nodes $\ib$ we recursively define vector $\eta^{(\ib,r)} \in
% \Delta_{\abr{d(\ib,r)}}$ as $\eta^{(\ib,r)} = \frac{1}{|c(\ib)|}
% \biguplus \limits_{\jb \in c(\ib)} \eta^{(\jb,r)} $, where $d(\ib,
% r)$ represents all the samples in the documents under node $\ib$ and
% thus we have $\abr{d(\ib,r)} := n_\ib!/(n_\ib-r)!$. The symbol
% $\biguplus$ represents vector concatenation. For example, for $a_1 =
% [1, 2]^T$, $a_2 = [3, 4, 5]^T$, the concatenation will be $\biguplus
% \limits_{i=1}^2 a_i = [1, 2, 3, 4, 5]^T$.
%The proof is given in Appendix \ref{proof:momentbounds} and \ref{proof:tensorbounds}.

\subsection{Proof of Theorem \ref{theorem:momentbounds}}
\label{proof:momentbounds}
\begin{proof}
  Recall that for both empirical estimate and expectation of moment at node $\ib$, we have:
  \begin{equation}
  \begin{aligned}
	M_r^{\ib} = \frac{1}{|c(\ib)|} \sum_{\jb \in c(\ib)}  M_r^\jb
	= \sum_{\sbb \in d(\ib)} \eta_\sbb^{(\ib,r)} \varphi_r(x_\sbb)  
  \end{aligned}
  \end{equation}
% where $\eta^{(\ib,r)} \in \Delta_{\abr{d(\ib,r)}}$ and $\eta^{(\ib,r)} = \frac{1}{|c(\ib)|} \biguplus \limits_{\jb \in c(\ib)} \eta^{(\jb,r)} $.
Now define
  \begin{align}
%    \label{eq:xi-dev}
    \Xi[X] := \sup_{u: \nbr{u} \leq 1} \abr{T(M_r^{\ib}, u, \cdots, u) -T(\hat{M}_r^{\ib}, u, \cdots, u)}\nonumber.
  \end{align}
  The deviation between empirical average and expectation observed
  when using $X$. Then $\Xi[X]$ is concentrated. This follows from
  the inequality of \cite{McDiarmid89} since for any $\rb\in d(\ib,k)$
  \begin{equation}
  \begin{aligned}
    \label{eq:xi-difference}
     &|\Xi[X] - \Xi[(X \backslash \cbr{x_\sbb}) \cup \cbr{x'_\sbb}]| \leq  \eta_\sbb^{(\ib,r)} \nbr{\varphi_r(x_\sbb) - \varphi_r(x'_\sbb)}
    \leq \sqrt{2} \eta_\sbb^{(\ib,r)}.
  \end{aligned}
  \end{equation}
  Hence the random variable $\Xi[X]$ is concentrated in the sense that $\Pr\cbr{\Xi[X] - \Eb_{X}[\Xi[X]] < \epsilon} \geq 1-\delta$ with  $\delta = \exp\rbr{-\frac{\epsilon^2}{\|\eta^{(\ib,r)}\|_2^2}}$ or, in other words, $\epsilon = \|\eta^{(\ib,r)}\|_2 \sqrt{\ln (1/\delta)}$.

The next step is to bound the expectation of $\Xi[X]$. This is accomplished as follows:
  \begin{align*}
\Eb_{X}\sbr{\Xi[X]} \leq&   \Eb_{X, X'} \sbr{\sup_{u: \nbr{u} \leq 1}
      \abr{ T(\hat{M}_r^{\ib}, u, \cdots, u) -T(\tilde{M}_r^{\ib}, u, 			\cdots, u) }}\\
= &  \Eb_{\sigma} \Eb_{X, X'} \left[\sup_{u: \nbr{u} \leq 1}
      \left| \sum_{\sbb \in d(\ib)} \sigma_\sbb \eta_\sbb^{(\ib,r)} \left( 			T(\varphi_r(x_{\sbb}), u, \cdots, u) \vphantom{ - 							T(\varphi_r(x_{\sbb}'), u, \cdots, u)} - 		 		T(\varphi_r(x_{\sbb}'), u, \cdots, u) \right) \right| \right]\\      
\leq & 2 \Eb_{\sigma} \Eb_{X} \sbr{\sup_{u: \nbr{u} \leq 1}
      \abr{\sum_{\sbb \in d(\ib)} \sigma_\sbb \eta_\sbb^{(\ib,r)} 						T(\varphi_r(x_{\sbb}), u, \cdots, u)}} \\
\leq & 2 \Eb_{\sigma} \Eb_{X} \sbr{\nbr{\sum_{\sbb \in d(\ib)}
        \sigma_\sbb \eta_\sbb^{(\ib,r)} \varphi_r(x_\sbb)}}\\
\leq & 2 \Eb_{X} \sbr{\Eb_{\sigma} \sbr{\nbr{\sum_{\sbb \in d(\ib)} 				\sigma_\sbb \eta_\sbb^{(\ib,r)} \varphi_r(x_\sbb)}^2}}^{\frac{1}{2}} \leq  2\|\eta^{(\ib,r)}\|_2 ,
\end{align*}

Here the first inequality follows from convexity of the argument. The subsequent equality is a consequence of the fact that $X$ and $X'$ are drawn from the same distribution, hence a swapping permutation with the ghost-sample leaves terms unchanged. The following inequality is an application of the triangle inequality. Next we use the Cauchy-Schwartz inequality, convexity and last the fact that $\nbr{\varphi_r(x)} \leq 1$. Combining both bounds yields $\epsilon_\ib \geq  \|\eta^{(\ib,r)}\|_2 \rbr{2+\sqrt{\ln (1/\delta)}}$.
For the definition of efficient number, since $\nbr{\eta}_1 = 1$, we have $n_{\ib,r} = 1/ \|\eta^{(\ib,r)}\|_2^2 $. Thus we obtained the theorem.
\end{proof}

%%%%%%%%%%%%%%%%%%%%%%%%%%%%%%%%%%%%%%%%%%%%%
\subsection{ Proof of Theorem \ref{theorem:tensorbounds}}
\label{proof:tensorbounds}
\begin{proof}
By theorem \ref{theorem:momentbounds} and the definition of $S_2^{\ib}$ and $S_3^{\ib}$, the bounds for tensors can be easily obtained. For $S_2^{\ib}$, we have:
\begin{equation}
\begin{aligned}
	\nbr{\hat{S}_2^{\ib}- S_2^{\ib}} =&  \nbr{\hat{M}_2^{\ib} - C_2 \cdot \hat{S}_1^{\ib}\otimes \hat{S}_1^{\ib} - M_2^{\ib} + C_2 \cdot S_1^{\ib}\otimes S_1^{\ib}}\\
           \leq& \nbr{\hat{M}_2^{\ib}-M_2^{\ib}} + C_2 \nbr{\hat{S}_1^{\ib} \otimes \hat{S}_1^{\ib} - S_1^{\ib} \otimes S_1^{\ib}}\\
        \leq& \nbr{\hat{M}_2^{\ib}-M_2^{\ib}} + \nbr{\hat{S}_1^{\ib} -S_1^{\ib}}\nbr{\hat{S}_1^{\ib} -S_1^{\ib}}+ 2\nbr{S_1} \nbr{\hat{S_1} -S_1 }\\
       \leq&  \sbr{  \|\eta^{(\ib,2)}\|_2 + 2 \|\eta^{(\ib,1)}\|_2}\rbr{2+\sqrt{\ln (3/\delta)}}+  \sbr{  \|\eta^{(\ib,1)}\|_2 \rbr{2+\sqrt{\ln (3/\delta)}}}^2 
\end{aligned}
\end{equation}
For $S_3,$ expanding the 
\begin{equation}
\begin{aligned}
	\nbr{\hat{S}_3^{\ib} - S_3^{\ib}} 
& \leq \lVert M_3^{\ib} - C_4 \cdot S_1^{\ib} \otimes S_1^{\ib} \otimes S_1^{\ib} - C_5 \cdot \symm_3 \sbr{S_2^{\ib} \otimes S_1^{\ib}} - \hat{M}_3^{\ib} + C_4 \cdot \hat{S}_1 \otimes \hat{S}_1 \otimes \hat{S}_1\\
& \,\,\,\,\,\,\,\, + C_5 \cdot \symm_3 \sbr{\hat{S}_2 \otimes \hat{S}_1} \rVert \\ 
 \leq &\nbr{M_3^{\ib} - \hat{M}_3^{\ib}} + C_4 \rbr{ \nbr{S_1^{\ib}-\hat{S}_1^{\ib}}^3+ 3S_1^{\ib}\nbr{S_1^{\ib}-\hat{S}_1^{\ib}}^2}+ 3C_4 S_1^{\ib^2}\nbr{S_1^{\ib}-\hat{S}_1^{\ib}} \\
 &+  3 C_5\nbr{S_2^{\ib} - \hat{S}_2^{\ib}} \nbr{S_1^{\ib} - \hat{S}_1^{\ib}} + 3 C_5 \rbr{ S_1^{\ib} \nbr{S_2^{\ib} - \hat{S}_2^{\ib}} +\cdot S_2^{\ib} \nbr{S_1- \hat{S_1}} }\\
\leq& \sbr{  \|\eta^{(\ib,3)}\|_2 + 3C_5 \|\eta^{(\ib,2)}\|_2}\rbr{2+\sqrt{\ln (3/\delta)}} +(C_4+3C_5) \|\eta^{(\ib,1)}\|_2 \rbr{2+\sqrt{\ln (3/\delta)}}\\
&+ 3(C_4+C_5) \|\eta^{(\ib,1)}\|_2^2 \rbr{2+\sqrt{\ln (3/\delta)}}^2+ C_4 \|\eta^{(\ib,1)}\|_2^3 \rbr{2+\sqrt{\ln (3/\delta)}}^3
\end{aligned}
\end{equation}
\end{proof}
%%%%%%%%%%%%%%%%%%%%%%%%%%%%%%%%%%%%%%%%%%%%%%%%%%

\subsection{ Proof of Theorem \ref{theorem:reconstruction}}
\label{proof:reconstruction}

\begin{proof}
\label{proof:hdp3moments}
We follow the similar steps for complexity analysis in \cite{AnaFosHsuKakLiu12}.
Using the definition of tensor structure we stated in Lemma \ref{lem:hdp3moments_general}, we define:
\begin{equation}
	\tilde{\Phi} := T(\sqrt{C_3^{0}} \cdot \sqrt{\pi_0}, \Phi),
\end{equation}
where $H = [H_1, H_2, \cdots, H_k]$ is a normalized vector. So we have:
\begin{equation}
\begin{aligned}
	\sqrt{C_3^{0} \min_j \pi_{0j}} \sigma_k \rbr{\Phi} \leq \sigma_k \rbr{\tilde{\Phi}} \leq 1,\\
        \sigma_1 \rbr{\tilde{\Phi}} \leq \sigma_1 \rbr{\Phi}\sqrt{C_3^{0} \gamma_0}.
\end{aligned}
\end{equation}
Thus, $S_2^0$ and $S_3^0$ can be transformed to:
\begin{equation}
\begin{aligned}
	S_2^0 &= T( C_3^{0} \cdot \mathrm{diag}\rbr{\pi_0},\Phi,\Phi) = \tilde{\Phi}\tilde{\Phi}^T\\
	S_3^0 &= T( \frac{C_6^{0}}{C_3^{0}\sqrt{C_3^{0}}}  \mathrm{diag}\rbr{\frac{1}{\sqrt{\pi_0}}},\tilde{\Phi},\tilde{\Phi},\tilde{\Phi}) 
\end{aligned}
\end{equation}
Let $\lambda_i$ be the singular values of $S_3,$ we have:
\begin{equation}
	\lambda_i =  \frac{C_6^0}{C_3^0\sqrt{C_3^0}}\frac{1}{\sqrt{\pi_{0i}}}
\end{equation}
such that, for $i \in [K],$
 \begin{equation}
	 \frac{C_6^0}{C_3^0\sqrt{C_3^0}}\frac{1}{\sqrt{\gamma_0}} \leq \lambda_i \leq  \frac{C_6^0}{C_3^0\sqrt{C_3^0}}\frac{1}{\sqrt{\min_j \pi_{0j} }}.
\end{equation}

Next, as in Algorithm \ref{alg:hdpeca}, $\hat{W}$ whitens a rank $k$ approximation to $S_2^0,$ U. Here we define $\hat{S}_{2,k}^0$ to be the best rank $k$ approximation of $\hat{S}_2^0$. Besides, define:
\begin{equation}
	M := W^T \tilde{\Phi}, \,\,\,\, \hat{M} = \hat{W}^T \tilde{\Phi}.
\end{equation}
\begin{lemma}{(Lemma C.1 in \cite{AnaFosHsuKakLiu12})}
  \label{lemma:W_bounds}
Let $\Pi_W$ be the orthogonal projection onto the range of $W$ and $\Pi$ be the orthogonal projection onto the range of $0$. Suppose $\nbr{\hat{S}_2^{0} - S_2^0} \leq \sigma_k \rbr{S_2^0}/2$,
\begin{equation}
\begin{aligned}
&\nbr{M^0} = 1,\,\,\,\, \nbr{\hat{M}^0} \leq 2, \,\,\,\,  \nbr{\hat{W}} \leq \frac{2}{\sigma_k\rbr{\tilde{\Phi}}},\\
  &\nbr{\hat{W}^+} \leq 2 \sigma_1\rbr{\tilde{\Phi}},\,\,\,\,  \nbr{W^+} \leq 3 \sigma_1\rbr{\tilde{\Phi}} \\
&\nbr{M^0 - \hat{M}^0} \leq \frac{4}{\sigma_k\rbr{\tilde{\Phi}}^2}\nbr{\hat{S}_2^0 - S_2^0} ,\\
 &\nbr{\hat{W}^+ - W^+} \leq \frac{6\sigma_1\rbr{\tilde{\Phi}}}{\sigma_k \rbr{\tilde{\Phi}}^2} \nbr{\hat{S}_2^0 - S_2^0},\\
 &\nbr{\Pi - \Pi_W} \leq \frac{4}{\sigma_k\rbr{\tilde{\Phi}}}\nbr{\hat{S}_2^0 - S_2^0}.
\end{aligned}
\end{equation}
\end{lemma}
By using the upper bound of $\lambda_i$ and {\bf Lemma C.1} and {\bf Lemma C.2} in \cite{AnaFosHsuKakLiu12}, we get:
\begin{lemma}{}
  \label{lemma:Ws2_bounds}
Suppose $E_{S_2^0} \leq \sigma_k \rbr{S_2^0}/2$. For $\nbr{\theta} = 1,$ we have:
\begin{equation}
\begin{aligned}
&\nbr{T\rbr{ S_3^0, W, W,W\theta} -T\rbr{ \hat{S}_3^0, \hat{W}, \hat{W},\hat{W}\theta}  }\\
&\leq c \rbr{\frac{C_6 \nbr{S_2^0-\hat{S}_2^0}}{C_3^0 \sqrt{C_3^0} \sqrt{\min_j \pi_{0j}}\sigma_k\rbr{\tilde{\Phi}}^2} + \frac{\nbr{S_3^0-\hat{S}_3^0}}{ \sigma_k\rbr{\tilde{\Phi}}^3}} 
\end{aligned} 
\end{equation}
\end{lemma}
Following the similar steps in  {\bf Lemma C.3} in \cite{AnaFosHsuKakLiu12}, we have
\begin{lemma}{(SVD Accuracy)}
  \label{lemma:svd_bound}
 Suppose $\nbr{S_2^0-\hat{S}_2^0} \leq \sigma_k \rbr{S_2^0}/2,$ with probability greater than $1-\delta'$
\begin{equation}
    \nbr{v_i - \hat{v}_i} \leq c \frac{k^3 C_3^0 \sqrt{C_3^0 \gamma_0}}{\delta'C_6^0} c_1
\end{equation}
where 
\begin{align}
c_1 = \frac{C_6^0 \nbr{S_2^0-\hat{S}_2^0}}{C_3^0 \sqrt{C_3^0} \sqrt{\min_j \pi_{0j}}\sigma_k\rbr{\tilde{\Phi}}^2} + \frac{\nbr{S_3^0-\hat{S}_3^0}}{ \sigma_k\rbr{\tilde{\Phi}}^3}
\end{align}
\end{lemma}
%%%%%%%%%%%%%%%%%%%%%%%%%%%%%%%%%%%%%%%%%%%%%%%%%
Combine everything together, we have:
\begin{lemma}{(Lemma C.6 in \cite{AnaFosHsuKakLiu12})}
  \label{lemma:svd_bound}
 Suppose $\nbr{S_2^0 -\hat{S}_2^0} \leq \sigma_k \rbr{S_2^0}/2,$ with probability greater than $1-\delta',$ we have
\begin{equation}
      \nbr{\Phi_i - \frac{1}{\hat{Z}_i}\rbr{\hat{W}^+}^T \hat{v}_i} \leq c \frac{k^3 \gamma_0}{\delta' \min_j \pi_{0j}^2    C_3} \epsilon
\end{equation}
where
\begin{align}
\epsilon = \frac{\nbr{S_2^0 -\hat{S}_2^0} }{\sigma_k\rbr{\Phi}^2 }+ \frac{C_3^0\nbr{S_3^0 -\hat{S}_3^0}  }{C_6^0 \sigma_k\rbr{\Phi}^3}
\end{align}
\end{lemma}
Using the bounds for tensor in Theorem \ref{theorem:tensorbounds}, we finish the proof.
\end{proof}

\end{document}